%% file: main.tex
\documentclass[11pt]{article}

\usepackage[nonatbib,final]{neurips_2025}

\usepackage[utf8]{inputenc} %
\usepackage[T1]{fontenc}    %
\usepackage{hyperref}       %
\usepackage{url}            %
\usepackage{booktabs}       %
\usepackage{amsfonts}       %
\usepackage{nicefrac}       %
\usepackage{microtype}      %
\usepackage{xcolor}         %

\input{preamble}

\title{Multivariate Latent Recalibration for  \\Conditional Normalizing Flows}

\author{%
  Victor Dheur \\
  Department of Computer Science \\
  University of Mons \\
  Mons, Belgium \\
  \texttt{victor.dheur@umons.ac.be} \\
  \And
  Souhaib Ben Taieb\thanks{Also affiliated with the Department of Computer Science, University of Mons.} \\
  Department of Statistics and Data Science \\
  Mohamed bin Zayed University of Artificial Intelligence \\
  Abu Dhabi, United Arab Emirates \\
  \texttt{souhaib.bentaieb@mbzuai.ac.ae} \\
}

\begin{document}

\maketitle

\begin{abstract}

A reliable estimate of the full conditional distribution of a multivariate response given a set of covariates is essential in many decision-making applications. However, misspecified or miscalibrated models can lead to poor approximations of the joint distribution, resulting in unreliable predictions and suboptimal decisions. Standard recalibration methods are largely restricted to univariate settings, and while conformal prediction techniques yield multivariate regions with coverage guarantees, they do not provide an explicit form of the underlying probability distribution. We address this gap by first introducing a novel notion of latent calibration, which assesses probabilistic calibration in the latent space of conditional invertible generative models such as normalizing flows and flow matching. Second, we propose latent recalibration (LR), a post-hoc model recalibration method that learns a transformation of the latent space with finite-sample bounds on latent calibration. Unlike existing recalibration methods, LR produces a recalibrated distribution with an explicit multivariate density function while remaining computationally efficient. Extensive experiments on both tabular and image datasets show that LR consistently improves latent calibration error and the negative log-likelihood of the recalibrated models.

\end{abstract}

\setcounter{footnote}{0}

\input{tex/sec-intro}

\input{tex/sec-background}
\input{tex/sec-latent}

\input{tex/sec-related-work}

\input{tex/sec-experiments}

\input{tex/sec-conclusion}

\printbibliography
\newpage
\appendix

\newpage

\input{tex/suppl-datasets}

\input{tex/suppl-related-work}
\input{tex/suppl-proofs}
\input{tex/suppl-smooth-approx}
\input{tex/suppl-experiments-details}
\input{tex/suppl-decision-making}
\input{tex/suppl-results}

\end{document}

%% file: preamble.tex
\usepackage{todonotes} %

\usepackage{amsmath, amsfonts, amssymb, amsthm} %
\usepackage{bbm} %
\usepackage{bm} %
\usepackage{dsfont} %
\usepackage{mathtools} %
\usepackage{braket} %
\usepackage{calc} %

\usepackage{xr} %
\usepackage[capitalize,noabbrev]{cleveref} %
\usepackage{nameref} %

\usepackage{thmtools, thm-restate} %

\usepackage{algorithm}
\usepackage{algorithmic}
\usepackage{listings} %

\usepackage[citestyle=authoryear,natbib=true,
backend=biber,doi=false,isbn=false,url=false,
maxcitenames=2,uniquelist=false]{biblatex}
\usepackage{csquotes} %

\usepackage{outlines} %

\usepackage{float} %
\usepackage{graphicx} %
\usepackage[export]{adjustbox} %
\usepackage{tikz} %
\usepackage[font=small]{caption}
\usepackage{subcaption} %
\usepackage{stfloats} %
\usepackage{placeins} %
\usepackage{wrapfig} %

\usepackage{array, tabularx, longtable} %
\usepackage{multirow, multicol} %
\usepackage{makecell} %
\usepackage{siunitx}

\usepackage{soul} %
\usepackage{indentfirst} %
\usepackage{hyphenat} %
\usepackage{pifont} %
\usepackage{bold-extra} %
\usepackage{xspace} %

\input{custom_commands}

\input{dl_book_commands}

\theoremstyle{definition}

\theoremstyle{definition}
\newtheorem{definition}{Definition}
\theoremstyle{remark}
\declaretheorem[style=definition]{proposition}
\theoremstyle{remark}

\bibliography{paperpile}
\renewbibmacro{in:}{}

\DeclareSourcemap{
  \maps[datatype=bibtex, overwrite]{
    \map{
      \step[fieldset=editor, null]
    }
  }
}

\AtEveryBibitem{\clearfield{month}}
\AtEveryBibitem{\clearfield{day}}
\AtEveryBibitem{\clearlist{language}}

\crefname{equation}{}{}

%% file: custom_commands.tex
\newcommand{\R}{\mathbb{R}}

\newcommand{\X}{\mathcal{X}}
\newcommand{\Y}{\mathcal{Y}}
\newcommand{\Z}{\mathcal{Z}}
\newcommand{\U}{\mathcal{U}}
\newcommand{\D}{\mathcal{D}}

\newcommand{\EE}[2]{\mathbb{E}_{#1\!\!}\left[#2\right]}
\def\E#1{\EE{\,}{#1}}

\DeclarePairedDelimiterX{\infdivx}[2]{(}{)}{%
	#1\;\delimsize\|\;#2%
}

\DeclareMathAlphabet{\mathmybb}{U}{bbold}{m}{n}
\newcommand{\indicator}{\mathmybb{1}}

\let\mc\mathcal                                             %
\let\mb\mathbb                                                  %

\newcommand{\norm}[1]{\lVert#1\rVert}

\newcommand{\LR}{\texttt{LR}\xspace}
\newcommand{\QR}{\texttt{QR}\xspace}
\newcommand{\HDRR}{\texttt{HDR-R}\xspace}
\newcommand{\BASE}{\texttt{BASE}\xspace}

\newcommand{\vsmash}[2]{%
  \smash[t]{#1}%
  \vphantom{\rule{0pt}{#2\ht\strutbox}}%
}

\newcommand{\Th}{\vsmash{\hat{T}}{1.1}}
\newcommand{\fh}{\vsmash{\hat{f}}{1.1}}
\newcommand{\Fh}{\vsmash{\hat{F}}{1.1}}

\newcommand{\controlledoverset}[3][\ht\strutbox]{%
  \raisebox{0pt}[#1][\depthof{\(#3\)}]{\(\smash[t]{\overset{#2}{#3}}\)}
}

\newcommand{\dapprox}{~\controlledoverset[\heightof{$\hat{T}$}]{d}{\approx}~}

%% file: dl_book_commands.tex
\def\floor#1{\lfloor #1 \rfloor}

\DeclareMathAlphabet{\mathsfit}{\encodingdefault}{\sfdefault}{m}{sl}
\SetMathAlphabet{\mathsfit}{bold}{\encodingdefault}{\sfdefault}{bx}{n}

\DeclareMathOperator*{\argmax}{arg\,max}
\DeclareMathOperator*{\argmin}{arg\,min}

%% file: tex/sec-intro.tex
\section{Introduction}
\label{sec:introduction}

Generating reliable uncertainty estimates is essential for trustworthy decision-making across a wide range of applications \citep{Gawlikowski2021-ty}. Multi-output regression problems, in particular, arise frequently in domains such as weather forecasting \citep{Setiawan2024-lz}, energy consumption prediction \citep{Makaremi2025-hf}, and healthcare resource utilization \citep{Cui2018-dm}. While flexible models like neural networks can achieve high predictive accuracy, their uncertainty estimates are often poorly \textit{calibrated}, meaning predicted probabilities or confidence regions do not align with empirical frequencies \citep{Guo2017-ow, Dheur2023-bo}. Furthermore, most recalibration methods are designed for the single-output setting \citep{Gneiting2007-la,Song2019-bk,Sahoo2021-hf,Kuleshov2021-ft,Dewolf2022-ug,Fakoor2023-mu,Marx2023-zj,Chung2023-kb,Gneiting2023-vu}.

Noting the general lack of methods for \textit{assessing} and \textit{recalibrating} multi-output models, \citet{Chung2024-zd} leveraged highest-density regions (HDRs) \citep{Hyndman1996-wx} to introduce the notion of HDR calibration and propose the sampling-based HDR recalibration (\HDRR) method. Recently, multi-output conformal prediction (CP) methods have also been developed to construct joint prediction sets \citep{Wang2023-vn, Feldman2023-cc, Fang2025-qs, Dheur2025-br}. However, both \HDRR and CP approaches fail to provide an explicit form for the underlying recalibrated probability distribution, with \HDRR further involving computationally intensive sampling and binning at test time.

To overcome these limitations, we introduce the \emph{latent recalibration (\LR)} method, based on a new notion of \emph{latent calibration}, which recalibrates invertible generative models (e.g., normalizing flows (NFs) or flow matching (FM)) by operating within their latent space. The core idea is to learn a transformation of the latent space such that the resulting model achieves latent calibration. Compared to set-based CP methods and sampling-based recalibration approaches such as \HDRR, \LR (i) yields a fully recalibrated generative model with an explicit multivariate probability density function (PDF), (ii) provides finite-sample latent calibration guarantees, and (iii) enables efficient density evaluation and sampling, which are essential for many applications and support improved decision-making \citep{Klein2024-ow}.

Our main contributions are:
\begin{itemize}
    \item We introduce latent calibration, a new notion of calibration evaluated within the latent space of an invertible generative model, based on the distribution of latent norms.

    \item We propose \textbf{\LR}, a recalibration method that yields a multivariate predictive distribution with an explicit PDF, offers finite-sample latent calibration guarantees, and remains computationally efficient.

    \item We empirically demonstrate, across 29 multi-output tabular datasets and one high-dimensional image dataset, that \LR consistently improves latent calibration and reduces negative log-likelihood (NLL). A public codebase is provided to ensure reproducibility.\footnote{\url{https://github.com/Vekteur/latent-recalibration}}

\end{itemize}

%% file: tex/sec-background.tex
\section{Background}
\label{sec:background}

We consider a multi-output distributional regression setting, where the goal is to predict the distribution of a $d$-dimensional response variable $Y \in \Y \subseteq \mathbb{R}^d$ from a $p$-dimensional input vector $X \in \X \subseteq \mathbb{R}^p$. We assume access to a dataset $\D = \{(X^{(j)}, Y^{(j)})\}_{j=1}^N$, where the samples $(X^{(j)}, Y^{(j)})$ are drawn i.i.d. from $P_{XY}$, the true joint distribution over $\X \times \Y$. The dataset is partitioned into three disjoint subsets: a training set $\D_{\text{train}}$, a test set $\D_{\text{test}}$, and a calibration set $\D_{\text{cal}} = \{(X^{(i)}, Y^{(i)})\}_{i=1}^n$. The true conditional distribution of $Y$ given $X = x$ is denoted by $P_{Y|X=x}$. Similarly the true cumulative distribution function (CDF) is denoted $F_{Y|X=x}$, and the corresponding PDF, assumed to exist for all $x \in \X$, is denoted $f_{Y|X=x}$. More generally, we denote the distribution, CDF and PDF of any random variable $A$ by $P_A$, $F_A$, and $f_A$, respectively. Additionally, estimates are denoted $\hat{P}_A$, $\Fh_A$, and $\fh_A$.

\subsection{Normalizing Flows for predictive density estimation}

Using $\D_{\text{train}}$, our goal is to estimate the conditional PDF $f_{Y|X=x}$ for inputs $x \in \X$. NFs offer a flexible framework for modeling complex distributions over continuous random variables. Specifically, an NF defines a learnable bijective (conditional) transformation $\Th : \Z \times \X \to \Y$ between a latent space $\Z \subseteq \mathbb{R}^d$ and the output space $\Y \subseteq \mathbb{R}^d$. For any $y \in \Y$ and $x \in \X$, the transformation satisfies $\Th(\Th^{-1}(y; x); x) = y$. Given an input $x \in \X$, the NF maps a latent random variable $Z \in \Z$ (typically drawn from a known base distribution, such as $\mc{N}(0, I_d)$) to a new random variable $\Th(Z; x)$. The resulting conditional PDF is computed using the change-of-variables formula:
\begin{equation}
    \fh_{Y|X=x}(y) = f_Z(\Th^{-1}(y; x)) \left| \det\left( \nabla_y \Th^{-1}(y; x) \right) \right|,
\end{equation}
where $f_Z$ denotes the density of the latent variable $Z$. NFs are typically trained by minimizing the NLL over the training dataset using mini-batch stochastic gradient descent. For details on NF architectures, we refer the reader to \cite{Papamakarios2021-rc}.

For brevity, the main text focuses on normalizing flows, but our method is also compatible with flow matching, as shown in \cref{sec:results_cfm}.

\subsection{Statistical calibration}

While training with strictly proper scoring rules such as the NLL encourages accurate predictions, it does not guarantee that the resulting predictions are reliable or calibrated, meaning they are statistically aligned with the true distribution of the observations \citep{Gneiting2007-la}. This issue is particularly relevant under limited data or model misspecification, and it has gained renewed attention with the observation that modern neural network classifiers are often miscalibrated and overconfident \citep{Guo2017-ow}.

\paragraph{Probabilistic calibration.}
For real-valued outcomes ($d = 1$), probabilistic calibration \citep{Gneiting2007-la} builds on the probability integral transform (PIT). Denote $\Fh_{Y|X}$ a predictive CDF for $Y$ whose value depends on the random variable $X$. Assuming $\Fh_{Y|X=x}$ is continuous for any $x \in \X$, probabilistic calibration requires that
\begin{equation}
	F_{Y|X}(Y) \sim \U(0, 1).
\end{equation}

\paragraph{Multivariate calibration.}
Compared to the univariate case, calibration for vector-valued outcomes has been relatively underexplored. Moreover, assessing calibration in the multivariate setting ($d \geq 2$) is inherently more challenging, as the PIT is no longer uniformly distributed \citep{Genest2001-qz}. While probabilistic calibration can be assessed separately for each dimension, this approach may miss important dependencies between outputs \citep{Chung2024-zd}.

The primary approach to assessing multivariate calibration involves reducing multivariate predictions and observations to univariate summary statistics, and then evaluating the uniformity of the PITs of these transformed values \citep{Allen2024-fm}. Let $(X, Y) \sim P_{X,Y}$ and $(X, \hat{Y}) \sim P_X \hat{P}_{Y|X}$, and define a transformation function (also known as a pre-rank function) $g: \X \times \Y \to \R$. If $G = g(X, Y)$ and $\hat{G} = g(X, \hat{Y})$, then by the probability integral transform, the random variable $\hat{U} = F_{\hat{G}}(G)$ is uniformly distributed whenever $\hat{G} \dapprox G$. In this case, we say that $\hat{P}_{Y|X}$ is probabilistically calibrated with respect to the transformation $g$. \citet{Chung2024-zd} introduced HDR calibration as a special case, where $g(x, y) = -\fh_{Y|X=x}(y)$. In that case, the corresponding PIT $\hat{U} = F_{\hat{G}}(G) = \text{HPD}_{\fh_{Y|X}}(Y)$ is the highest predictive density (HPD; \cite{Box1992-zd}).
This form of calibration ensures that HDRs derived from the predictive distribution achieve correct empirical coverage at all nominal probability levels.

\subsection{Recalibration methods}
\label{sec:recalibration_methods}

Recalibration methods adjust a potentially miscalibrated base predictor (e.g., $\Fh_{Y|X}$) to produce an updated predictor (e.g., $\Fh'_{Y|X}$) that satisfies a desired calibration property.

\paragraph{Quantile recalibration.}

For the univariate case ($d = 1$), quantile recalibration (\QR) \citep{Kuleshov2018-tb} is a recalibration method that enforces probabilistic calibration. Let $\hat{U} = \Fh_{Y|X}(Y)$, and define $F_{\hat{U}}$, the CDF of $\hat{U}$, as the \textit{calibration map}. The recalibrated CDF is given by $\Fh'_{Y|X} = F_{\hat{U}} \circ \Fh_{Y|X}$. By construction, $\Fh'_{Y|X}(Y)$ is uniformly distributed over $[0, 1]$. Specifically, for any $\alpha \in (0, 1)$:
\begin{align}
\mathbb{P}(\Fh'_{Y|X}(Y) \leq \alpha)
&= \mathbb{P}(\Fh_{Y|X}(Y) \leq F_{\hat{U}}^{-1}(\alpha)) = F_{\hat{U}}(F_{\hat{U}}^{-1}(\alpha)) = \alpha.
\end{align}

\paragraph{HDR recalibration.}
For $d \geq 1$, \citet{Chung2024-zd} proposed a sampling-based HDR recalibration method (\HDRR) that targets HDR calibration.
Let $\hat{U} = F_{\hat{G}}(G)$, and define $F_{\hat{U}}$ as the calibration map.
Given a new input $x_\text{test}$, $K$ candidates samples $\{\hat{Y}^{(k)}\}_{k=1}^K$ are drawn from $\hat{P}_{Y|X=x_\text{test}}$, with pre-rank values $\hat{G}^{(k)} = -\fh_{Y|X=x_\text{test}}(\hat{Y}^{(k)})$, $k = 1, \dots, K$.
Based on binning, \HDRR resamples from the set of candidate samples, producing a new set of recalibrated samples $\{\hat{Y}'^{(k)}\}_{k=1}^K$ with pre-rank values $\hat{G}'^{(k)} = -\fh_{Y|X=x_\text{test}}(\hat{Y}'^{(k)})$. This resampling process induces a new random variable $\hat{Y}'$ and its pre-rank $\hat{G}' = -\fh_{Y|X=x_\text{test}}(\hat{Y}')$.
The number of samples in each bin is determined such that
\begin{equation}
    F_{\hat{G}'}(-\fh_{Y|X}(Y)) \sim \U(0, 1).
\end{equation}
For completeness, we provide an exact algorithm in \cref{sec:hdr_recalibration}. While effective in calibrating HDRs, \HDRR has several limitations: (i) it does not produce an explicit recalibrated PDF $\fh'_{Y|X}$; (ii) it generates duplicate samples; (iii) it is subject to discretization errors when estimating $F_{G|x}$; and (iv) it is computationally expensive, as it requires generating $K$ initial samples for every recalibrated output.

\paragraph{Estimating the calibration map.} In practice, the ideal calibration map $F_{\hat{U}}$ is unknown and is estimated based on the calibration statistics $\{\hat{U}_{i}\}_{i=1}^n$ with $\hat{U}_{i} \sim F_{\hat{U}}$.
A standard estimator is the empirical CDF $\Fh_{\hat{U}}(\alpha) = \frac{1}{n} \sum_{i=1}^n \indicator(\hat{U}_{i} \leq \alpha)$ but differentiable estimators can also be employed \citep{Marx2022-yz,Dheur2024-zs}.

\subsection{Conformal prediction}
Conformal prediction (CP) constructs distribution-free prediction sets $\hat{R}_\alpha(X)$ with finite-sample marginal coverage guarantees, i.e., $\mathbb{P}(Y \in \hat{R}_\alpha(X)) \geq 1 - \alpha$ for any desired significance level $\alpha \in (0,1)$ \citep{Vovk2005-ib, Angelopoulos2021-rc}. A common variant, Split CP (SCP) \citep{Papadopoulos2002-eb}, partitions the data into a training set $\D_\text{train}$ and a calibration set $\D_\text{cal}$. A base predictor is trained on $\D_\text{train}$, and then a conformity score $s: \X \times \Y \to \R$, where lower values indicate better agreement between the model's predictions and the observations. SCP constructs the prediction set by computing the empirical $(1 - \alpha)$-quantile of the conformity scores evaluated on the calibration set, i.e., $\{s(X^{(1)}, Y^{(1)}), \dots, s(X^{(n)}, Y^{(n)}), +\infty\}$, denoted by $\Fh_S^{-1}(1 - \alpha)$. The resulting prediction region is $\hat{R}_\alpha(x) = \{ y \in \Y : s(x, y) \leq \Fh_S^{-1}(1 - \alpha) \}$, which satisfies the marginal coverage guarantee.

\begin{table}[t]
    \centering
    \caption{Comparison of calibration notions, the associated calibration statistic $\hat{U}$ (uniform under calibration), recalibration methods, and related conformal conformity scores.}
    \label{table:comparison_calibration}
	\vspace{0.2cm}
    \resizebox{\linewidth}{!}{ %
    \begin{tabular}{llll}
        \toprule
        \textbf{Calibration notion} & \textbf{Calibration statistic} & \textbf{Recalibration method} & \textbf{Conformal method} \\
        \midrule
        Probabilistic ($d=1$) & $\Fh_{Y|X}(Y)$ & Quantile recalibration (\QR) & DCP \\
        HDR ($d\geq1$) & $\text{HPD}_{\fh_{Y|X}}(Y)$ & HDR recalibration (\HDRR) & HPD-split \\
        \textbf{Latent ($d\geq 1$)} & $F_{\rho_{\Z}(Z)}(\ell_{\Th}(Y; X))$ & \textbf{Latent recalibration (\LR)} & CONTRA/L-CP \\
        \bottomrule
    \end{tabular}
    }
    \vspace{-1.2em}
\end{table}

Notably, specific choices of conformity scores correspond to recalibration statistics: Distributional Conformal Prediction (DCP) \citep{Chernozhukov2021-sg} uses $s_{\text{DCP}}(x, y) = \Fh_{Y|X=x}(y)$, while HPD-split \citep{Izbicki2022-ru} uses $s_{\text{HPD-split}}(x, y) = \text{HPD}_{\fh_{Y|X=x}}(y)$; these match the transformations used in \QR and \HDRR, respectively. This highlights a unified framework connecting conformal prediction and recalibration, summarised in \cref{table:comparison_calibration}.

%% file: tex/sec-latent.tex
\section{A New Latent Recalibration Method for Normalizing Flows}
\label{sec:latent_calibration_recalibration}

We propose a new recalibration method, called \textit{latent 
recalibration} (\LR), for conditional NFs. \LR operates in the latent space and is specifically designed to achieve our newly introduced notion of multivariate \textit{latent calibration}.

\subsection{A New Notion of Multivariate Latent Calibration}

Recall that, given a latent variable $Z \in \Z$ with a known distribution and an input $x \in \X$, conditional NFs estimate the conditional distribution of $Y$, $F_{Y|X=x}$, by learning a conditional bijective transformation $\Th: \Z \to \Y$ such that the PDF $\fh_{Y|X}$ of the transformed variable $\Th(Z; x)$ approximates $f_{Y|X=x}$. However, model misspecification or significant estimation errors in the learned transformation $\Th$ can lead to poor calibration of the induced distribution of $\Th(Z; x)$.

We propose to leverage the simple structure of the latent space $\Z$ and assess calibration directly in this space, a notion we refer to as \emph{latent calibration}. By definition, if the NF is well-specified for $F_{Y|X=x}$, then the inverse transformation $\Th^{-1}$ satisfies $\Th^{-1}(Y; X) \dapprox Z$, where $\dapprox$ denotes approximate equality in distribution. 
Building on this observation, we define a norm $\rho_{\Z} : \Z \to \mathbb{R}_+$ over $\Z$ (e.g., $\rho_{\Z}(z) \coloneqq \|z\|$). The goal is to test whether $\rho_{\Z}(\Th^{-1}(Y; X)) \dapprox \rho_{\Z}(Z)$.

Since the distribution of $Z$ is known and standard (e.g., standard Gaussian), the distribution of $\rho_{\Z}(Z)$ is often known in closed-form. For instance, if $Z \sim \mc{N}(0, I_d)$ and $\rho_{\Z}(z) = \|z\|$, then $\rho_{\Z}(Z)$ follows a Chi distribution with $d$ degrees of freedom ($\chi_d$), whose PDF, CDF, and quantile function can be computed efficiently. As another example, if $Z \sim \U(B_d)$ is uniformly distributed over the unit hyperball and $\rho_{\Z}(z) = \|z\|$, then $\rho_{\Z}(Z)$ follows a $\text{Beta}(d, 1)$ distribution.

\begin{definition}
	Consider a NF defined by a latent variable $Z$ and a bijective transformation $\Th$. For a pair $(X, Y)$, define the \emph{latent norm} w.r.t. $\hat{T}$ as
	\begin{equation}	    
	   \hat{L} =  \ell_{\Th}(Y; X) = \rho_{\Z}(\Th^{-1}(Y; X)). \label{eq:defL}
	\end{equation}
	The NF is said to be \emph{latent calibrated} w.r.t. $Z$ and the norm $\rho_{\Z}$ if the PIT of the latent norm follows a standard uniform distribution, i.e.,
	\begin{equation}
	    \hat{U} = F_{\rho_{\Z}(Z)}(\hat{L}) \sim \mc{U}(0, 1).
    \end{equation}
\end{definition}

To assess whether a model is latent calibrated, we define the \emph{latent expected calibration error} (L-ECE) as the $L^1$ distance between the CDF of the PIT variable $\hat{U}$ and the CDF of the uniform distribution:
\begin{equation}
    \text{L-ECE}(\Th) = \int_0^1 \left| F_{\hat{U}}(\alpha) - \alpha \right| \, d\alpha,
\end{equation}
The L-ECE is minimized at 0 when $\Th$ is perfectly latent calibrated, and has a maximum value of 0.5.

\subsection{Multivariate Latent Recalibration}
\label{sec:latent_recalibration}

We propose a multivariate latent recalibration method, called \LR, which performs a post-hoc adjustment of the latent space of a NF to ensure that the resulting model is latent calibrated. Key advantages of \LR are that it yields a recalibrated distribution with an explicit PDF, remains computationally efficient, and has finite-sample guarantees on latent calibration (see \cref{sec:lr_properties}).

\paragraph{Latent space transformation.} 

\LR uses the CDF $F_{\hat{L}}$ as its calibration map. 
We define a scalar strictly increasing transformation $r: \R_+ \to \R_+$ using the quantile function $F_{\hat{L}}^{-1}$, which maps the original latent norms $l \in \R_+$ to recalibrated norms as follows:
\begin{equation} \label{eq:r_transform_final}
    r(l) = F_{\hat{L}}^{-1}(F_{\rho_{\Z}(Z)}(l)).
\end{equation}
We also define a vector-valued transformation $R: \Z \to \Z$ based on the scalar transformation $r$, which maps latent vectors $z$ such that $\rho_\Z(R(z)) = r(\rho_\Z(z))$. When using the Euclidean norm $\rho_\Z(z) = \norm{z}$, $R$ is a radial transformation:
\begin{equation} \label{eq:R_transform_final}
    R(z) = \frac{ r(\norm{z}) }{ \norm{z} } \cdot z \quad (\text{with } R(0) = 0).
\end{equation}
The transformation $R$ rescales each vector $z$ along its original direction by replacing its norm $\|z\|$ with $r(\|z\|)$. This procedure defines a new latent variable $Z' = R(Z)$, and the associated recalibrated NF $\Th(Z'; X)$.

\begin{proposition}
\label{prop:lr_is_calibrated}
The recalibrated NF $\Th(Z'; X)$ defined with the new latent variable $Z' = R(Z)$ is \emph{latent calibrated}, i.e. $\hat{U}' = F_{\rho_{\Z}(Z')}(\hat{L}) \sim \U(0, 1)$.
\end{proposition}

\begin{proof}
Consider the inverse transformation $r^{-1}(l) = F^{-1}_{\rho_{\Z}(Z)}\left( F_{\hat{L}}(l) \right)$ for a latent norm $l \in \mathbb{R}_+$. Then, using $\rho_\Z(R(z)) = r(\rho_\Z(z))$ the following identity holds:
\begin{align}
	F_{\rho_{\Z}(Z')}(l) 
	= F_{r(\rho_{\Z}(Z))}(l) 
	= F_{\rho_{\Z}(Z)}\left( r^{-1}(l) \right) 
	= F_{\hat{L}}(l), \quad \forall l \in \mathbb{R}_+.
    \label{eq:Z'_to_F_S}
\end{align}
Then, it follows that 
\begin{align}
    \mb{P}(\hat{U}' \leq \alpha)
    = \mb{P}(F_{\rho_{\Z}(Z')}(\hat{L}) \leq \alpha)
    = \mb{P}(F_{\hat{L}}(\hat{L}) \leq \alpha)
    = \alpha.
\end{align}
\end{proof}

\paragraph{Recalibrated predictive density.}

A distinctive feature of our LR recalibration procedure is that it produces a recalibrated distribution with an explicit multivariate PDF.

Note that the recalibrated NF can be interpreted as a composite transformation $\Th' = \Th \circ R$, applied to the original latent variable $Z$ with density $f_Z$, typically a standard multivariate Gaussian. Given $x \in \X$ and $y \in \Y$, assuming the transformation $R$ is differentiable, the recalibrated predictive density $\fh'_{Y|X=x}(y)$ can be computed using the change of variables formula. Let $z' = \Th^{-1}(y; x)$ and $z = R^{-1}(z')$. Then, we have
\begin{align}
	\fh'_{Y|X=x}(y) &= f_{Z}\left( z \right) \left| \det\left( \nabla_z R(z) \right) \right|^{-1} \left| \det\left( \nabla_y \Th^{-1}(y; x) \right) \right|. \label{eq:recalibrated_pdf}
\end{align}
Let us consider the case where $\rho_{\Z}(z) = \|z\|$. The inverse transformation takes the form $R^{-1}(z') = \frac{r^{-1}(\|z'\|)}{\|z'\|} \cdot z'$ and the Jacobian determinant of \( R \) can be computed efficiently as:
\begin{equation}
    \left| \det\left( \nabla_z R(z) \right) \right| = \left( \frac{r(l)}{l} \right)^{d-1} \cdot \frac{\partial r(l)}{\partial l}, \quad \text{with } l = \|z\|.
    \label{eq:log_det_jac_R}
\end{equation}
A detailed proof is provided in \cref{sec:proof_radial_ldj}.
The term $\frac{\partial r(l)}{\partial l}$ in \eqref{eq:log_det_jac_R} is computed using the chain rule as:
\begin{equation}
    \frac{\partial r(l)}{\partial l} = \frac{\partial F_{\hat{L}}^{-1}(l')}{\partial l'} \cdot \frac{\partial F_{\rho_{\Z}(Z)}(l)}{\partial l}, \quad \text{where } l' = F_{\rho_{\Z}(Z)}(l).
\end{equation}

To compute \( \partial F_{\rho_{\Z}(Z)}(l) / \partial l \), we leverage the fact that \( \rho_{\Z}(Z) \sim \chi_d \), whose PDF is available in closed-form and can be evaluated efficiently.

\begin{figure}[t]
	\includegraphics[width=\linewidth]{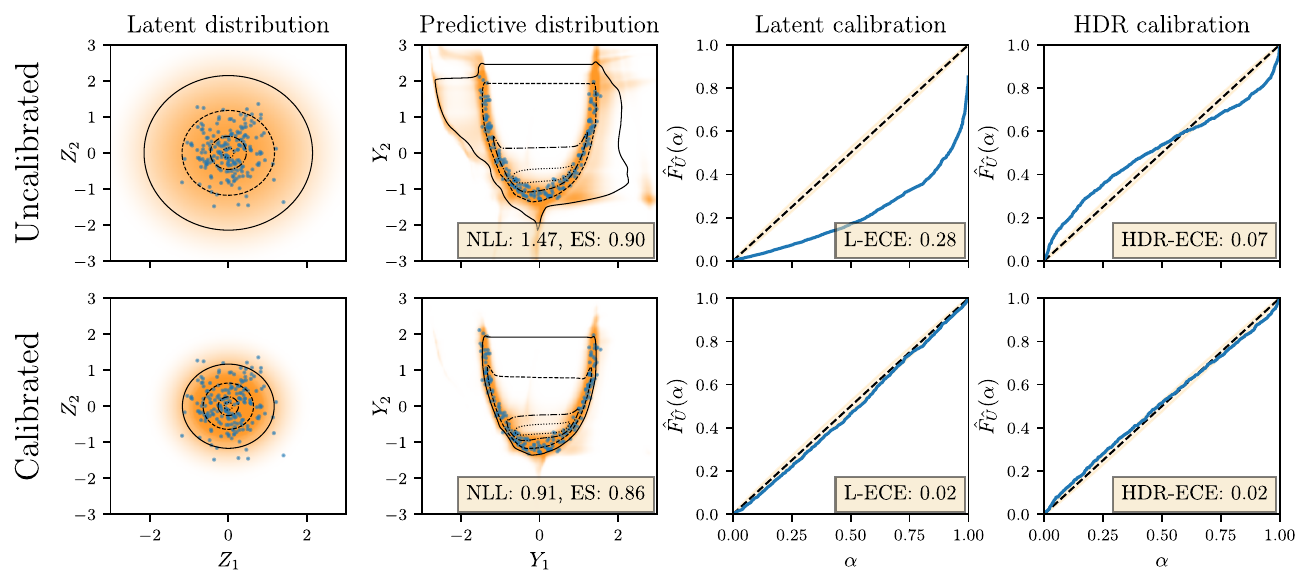}
\caption{Illustration of \LR for a bivariate output. The first column shows the latent distribution, the second column displays the predictive PDF, and the third and fourth columns show reliability diagrams for latent and HDR calibration, respectively. The first row corresponds to an uncalibrated NF, and the second is the same model after \LR. Calibration points and their projections in the latent space are shown in blue. The PDF for both the latent distribution and the predictive distribution is shown in orange. Level sets of the PIT of the latent norm at levels 0.01, 0.1, 0.5, and 0.9 are indicated with black contours in the second column, and their corresponding preimages are shown in the first column. \LR improves both latent calibration (third column) and HDR calibration (fourth column). Additional prediction examples on real-world datasets are presented in \cref{sec:examples_predictions}.}
    \label{fig:visualization}
    \vspace{-0.4cm}
\end{figure}

In practice, $F_{\hat{L}}$ is estimated by computing latent norms $\hat{L}_i = \ell_{\Th}(X^{(i)}, Y^{(i)})$ using samples $(X^{(i)}, Y^{(i)})$ from the calibration set $\D_{\text{cal}}$. \cref{sec:density_estimation} details how this can be achieved using kernel density estimation or monotonic splines, resulting in a differentiable estimate $\Fh_{\hat{L}}$ of $F_{\hat{L}}$.
All operations are carried out in log-space to ensure numerical stability. \cref{fig:visualization} illustrates \LR, with the recalibrated predictive density $\fh'_{Y|X}$ shown in the second column of the second row.

\subsection{Useful Properties of Multivariate Latent Recalibration}
\label{sec:lr_properties}

We present finite-sample coverage guarantees for \LR and highlight its connections to conformal prediction methods.
We assume that $R$ depends on an estimate $\Fh_{\hat{L}}$ of $F_{\hat{L}}$ based on latent norms $\hat{L}_1, \dots, \hat{L}_n$.

\paragraph{Finite-sample coverage guarantees for recalibrated latent norms.}

Let us assume that the estimated calibration map $\Fh_{\hat{L}}$ maps the $i$-th order statistic $\hat{L}_{(i)}$ of $\hat{L}_1, \dots, \hat{L}_n$ within a margin $\lambda / (n + 1) \geq 0$ of the target quantile $i / (n + 1)$, that is,
\begin{equation}
    \Fh_{\hat{L}}(\hat{L}_{(i)}) \in \left[\frac{i - \lambda}{n + 1}, \frac{i + \lambda}{n + 1}\right].
\end{equation}
Then, letting $\epsilon = \frac{1 + \lambda}{n + 1}$, Theorem 1 of \cite{Marx2022-yz} yields the following finite-sample coverage guarantee for the recalibrated latent norms:
\begin{equation}
    \mathbb{P}\left(F_{\rho_{\Z}(Z')}\left( \ell_{\Th}(Y; X) \right) \leq \alpha\right) = \mathbb{P}\left(\Fh_{\hat{L}}(\hat{L}) \leq \alpha\right) \in \left[\alpha - \epsilon, \alpha + \epsilon\right],
\end{equation}
where we used \cref{eq:Z'_to_F_S} for the first equality and the probabilities are taken over $X$, $Y$, and the recalibrated latent norms $\hat{L}_1, \dots, \hat{L}_n$.

\paragraph{Equivalence with conformal prediction sets.}

We observe that the prediction sets derived from the recalibrated predictive density of \LR coincide exactly with those obtained by the multivariate conformal methods CONTRA \citep{Fang2025-qs} and L-CP \citep{Dheur2025-br}. Specifically, this equivalence holds when \LR uses the empirical CDF of the calibration scores $\mc{L} = \{\hat{L}_1, \dots, \hat{L}_n, +\infty\}$ as its calibration map, i.e., $\Fh_{\hat{L}}(l) = \frac{1}{n + 1} \sum_{i=1}^n \indicator(\hat{L}_i \leq l)$.

CONTRA and L-CP are conformal methods that construct prediction sets using the conformity score $s_{\text{CONTRA}}(x, y) = s_{\text{L-CP}}(x, y) = \ell_{\Th}(y; x)$. Under this choice of calibration map, for any $x \in \X$ and $\alpha \in (0, 1)$, we have
\begin{align}
\{ y \in \Y: F_{\rho_{\Z}(Z')}(\ell_{\Th}(y; x)) \leq \alpha \}
&= \{ y \in \Y: \Fh_{\hat{L}}(\ell_{\Th}(y; x)) \leq \alpha \} \label{eq:latent_sublevel_set} \\
&= \{ y \in \Y: s_\text{CONTRA}(x, y) \leq \Fh_{\hat{L}}^{-1}(\alpha) \}, \label{eq:latent_conformal_region}
\end{align}
where $\Fh_{\hat{L}}^{-1}(\alpha) = \hat{L}_{(\lceil \alpha(n + 1) \rceil)}$ denotes the $(1 - \alpha)$ right empirical quantile of the calibration scores $\mc{L}$. This shows that the $\alpha$-sublevel sets of the PIT of the latent norm of \LR \cref{eq:latent_sublevel_set} correspond exactly to the conformal prediction sets produced by CONTRA and L-CP at coverage $\alpha$ \cref{eq:latent_conformal_region}. While this equivalence is notable, it is important to point out that the chosen calibration map $\Fh_{\hat{L}}$, being non-differentiable, does not yield a well-defined recalibrated predictive density function $\fh'_{Y|X}$.
This equivalence is summarized in the last row of \cref{table:comparison_calibration}.

\paragraph{Equivalence of \texorpdfstring{LR}{\LR} and \texorpdfstring{QR}{\QR} in the single-output setting.}

\QR is a special case of \LR where $d = 1$, $\Z=[0,1]$, $Z \sim \U(0, 1)$, $\Th^{-1}(y; x) = \Fh_{Y|X=x}(y)$ and $\rho_\Z(z) = z$. In this case, $R = \Fh_{\hat{L}}^{-1}$ and thus $\Th'^{-1}(\cdot; x) = R^{-1} \circ \Th^{-1}(\cdot; x) = \Fh_{\hat{L}} \circ \Fh_{Y|X=x} = \Fh'_{Y|X=x}$, showing that both methods perform exactly the same transformation.

%% file: tex/sec-related-work.tex
\section{Related work}
\label{sec:related_work}

Our work builds upon and contributes to generative modeling, calibration, conformal prediction, and methods that combine these concepts in the context of multi-output regression. An extended description of related works is available in \cref{sec:suppl-related_work}.

Various notions of calibration have been studied, including probabilistic \citep{Gneiting2007-la}, marginal \citep{Gneiting2007-la} and HDR \citep{Chung2024-zd} calibration. \citet{Ziegel2014-lh,Allen2024-fm} also proposed multivariate notions of calibration but, to our knowledge, no calibration methods for these notions have been proposed.

While traditional CP focuses on univariate intervals \citep{Romano2019-kp, Sesia2021-tn}, recent multivariate CP methods create flexible regions. HPD-split \citep{Izbicki2022-ru} uses HPD values as scores. PCP \citep{Wang2023-vn} uses balls around samples. ST-DQR \citep{Feldman2023-cc} selects samples based on a region in a latent space and creates balls around these samples. CONTRA \citep{Fang2025-qs} and L-CP \citep{Dheur2025-br} operate in the latent space of NFs.

Some methods explicitly merge CP and recalibration. \citet{Vovk2020-pg, Vovk2019-sn} developed conformal predictive systems for calibrated univariate distributions. MCC \citep{Marx2022-yz} unified univariate recalibration methods under a CP lens. Our work extends this direction to multivariate outputs via a transformation in the latent space.

%% file: tex/sec-experiments.tex
\section{Experiments}
\label{sec:experiments}

We present an extensive experimental study using 29 tabular datasets widely used in prior research \citep{Tsoumakas2011-wf, Cevid2022-ev,Chung2024-zd, Feldman2023-cc, Wang2023-vn, del-Barrio2024-la, Camehl2024-vw}. Furthermore, while recent work on model recalibration \citep{Chung2024-zd, Fang2025-qs} has primarily focused on data modalities with relatively low output dimensionality, we also include a high-dimensional output setting with an image dataset with a larger output dimension \citep{Choi2020-xa}.

\subsection{Datasets}

\paragraph{Tabular datasets.} The tabular datasets range in size from 103 to 50{,}000 data points, with the number of input features ($p$) varying from 1 to 368, and the number of output variables ($d$) ranging from 2 to 16. A detailed summary of these datasets is provided in Table~\ref{table:datasets} in Appendix~\ref{sec:datasets}. Following the protocol of \citet{Chung2024-zd}, we use a 65/20/15 split for training, validation, and testing. All input features and output targets are normalized to have zero mean and unit variance on the training set. Experiments are repeated 10 times with a different random splitting. For each run, we compare the same base model with or without recalibration.

\paragraph{Image dataset.} We use the AFHQ dataset \citep{Choi2020-xa}, which consists of high-resolution animal face images. The input $x \in \mathcal{X} = \{0, 1, 2\}$ indicates one of three classes (cat, dog, or wild animal), and the output is a $256 \times 256$ RGB image $y \in \mathcal{Y} = [-1, 1]^{3 \times 256 \times 256}$, resulting in an output dimension of $d = 196{,}608$. We follow the standard split with 14{,}630 training instances and 1{,}500 test instances. To improve sample quality, \citet{Zhai2024-ex} add Gaussian noise $\epsilon \sim \mathcal{N}(0, 0.07^2)$ to each image $y$ during training. 

\subsection{Experimental setup}

\paragraph{(Non-recalibrated) base model.} For the tabular datasets, we consider convex potential flows \citep{Huang2020-md}, masked autoregressive flows \citep[MAFs,][]{Papamakarios2017-uh} and flow matching \citep[FM,][]{Lipman2022-tm}. A notable difference from the setup of \citet{Chung2024-zd} is that their predictive distributions are restricted to multivariate Gaussians with diagonal covariance, whereas NFs can model dependencies between output dimensions. For the image dataset, we use the TarFlow model \citep{Zhai2024-ex}, a transformer-based conditional NF pre-trained on AFHQ, which achieves state-of-the-art likelihood performance. As is standard, all aforementioned NFs use a latent variable $Z \sim \mathcal{N}(0, I)$. Details on these base models are provided in \cref{sec:base_predictors}, with hyperparameter tuning details for convex potential flows in \cref{sec:additional_details}. Results for MAFs and FM are deferred to \cref{sec:results_autoregressive_flows,sec:results_cfm}. In the following, we denote the non-recalibrated base model as \BASE.

\paragraph{Compared methods.} For our latent recalibration method, \LR, we use the Euclidean norm $\rho_\mathcal{Z}(z) = \|z\|$ and estimate $F_L$ using kernel density estimation with a Gamma kernel; details are provided in \cref{sec:density_estimation}. For both tabular and image datasets, we compare \LR with the base model \BASE. Additionally, we include \HDRR for tabular datasets only, as it becomes computationally prohibitive for TarFlow. For tabular datasets, following \citet{Chung2024-zd}, the recalibration map is learned on the validation set. This avoids using additional data for calibration and ensures a fair comparison with \BASE, but sacrifices finite-sample guarantees. For the image dataset, since no separate calibration set is available, calibration is performed on the training data. This also sacrifices finite-sample guarantees, but we will show below that it still leads to substantial improvements in calibration.

\paragraph{Error metrics.} We consider several error metrics to compare the different methods. For the tabular datasets, we evaluate model calibration using the latent expected calibration error (L-ECE) and the HDR expected calibration error (HDR-ECE). Both metrics range from 0 (best) to 0.5 (worst). Predictive accuracy is assessed using two strictly proper scoring rules: negative log-likelihood (NLL) and the energy score (ES). Notably, \LR yields a recalibrated density with a closed-form PDF, enabling direct computation of the NLL, which is not possible with \HDRR. Since the scales of NLL and ES vary across datasets, we report relative values, defined as the difference with the score achieved by \BASE. All metrics are negatively oriented. Exact definitions are provided in \cref{sec:evaluation_metrics}. For the image dataset, we report L-ECE and the bits per dimension (BPD), following \citet{Zhai2024-ex}. BPD corresponds to a rescaled version of the NLL (details in \cref{sec:evaluation_metrics}).

\subsection{Results}

\paragraph{Tabular datasets.}

\begin{wrapfigure}[14]{r}{0.6\textwidth}
	\centering
    \vspace{-1.6em}
	\includegraphics[width=\linewidth]{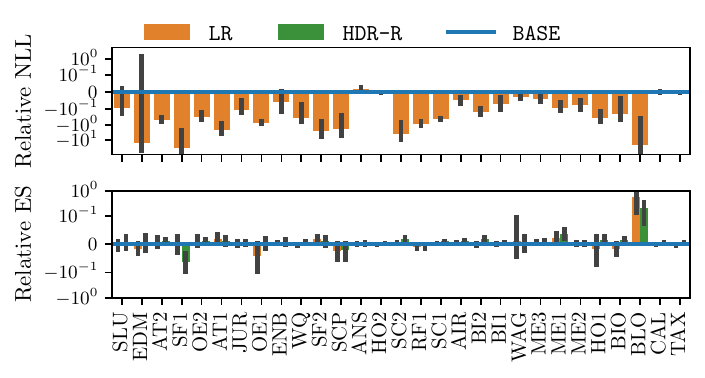}
	\vspace{-0.6cm}
	\caption{Relative NLL and ES on datasets sorted by size, using a convex potential flow model.}
	\label{fig:barplot/scoring_rules}
\end{wrapfigure}

\cref{fig:barplot/scoring_rules} presents the normalized difference relative to \BASE for NLL and ES. We observe that \LR reduces the NLL on the majority of datasets. Since the NLL is a strictly proper scoring rule, this indicates that the recalibrated density $\fh'_{Y|X}$ produced by \LR generally provides a better fit to the true data distribution than the original model $\fh_{Y|X}$. In contrast to the NLL, the ES of \LR and \HDRR is largely unchanged. In \cref{sec:discriminative_ability}, we attribute this phenomenon to the metric's weaker discriminative ability relative to misspecifications in variance, correlation, and overall dependency structure. A notable exception is the increased ES on the \texttt{blog\_data} (BLO) dataset, where one output is discrete and a single value is repeated across 64\% of the instances. Based on this observation, we suggest not using \LR on datasets with discrete outputs.

\cref{fig:barplot/calibration} also shows the L-ECE (as a measure of latent calibration) and HDR-ECE (as a measure of HDR calibration), respectively. We see that \BASE exhibits significant latent miscalibration across many datasets, with L-ECE values reaching up to 0.25 out of a maximum of 0.5. In contrast, \LR consistently and substantially reduces L-ECE, demonstrating its effectiveness in achieving the desired latent calibration. Moreover, L-ECE tends to decrease as dataset size increases, which aligns with the finite-sample guarantees discussed in \cref{sec:lr_properties}. Reliability diagrams in \cref{sec:reliability_diagrams} further confirm this improvement. Additional experiments on a misspecified model are provided in \cref{sec:results_misspecified}, with significant improvements given by \LR across all metrics including ES.

\begin{figure}[ht]
	\includegraphics[width=\linewidth]{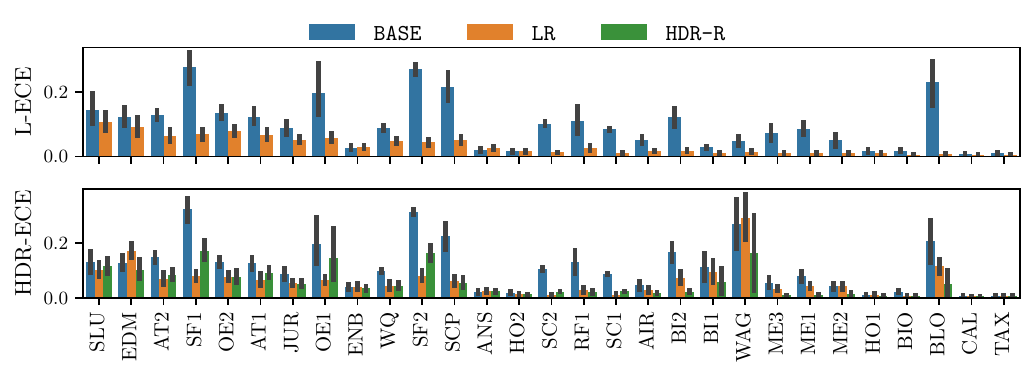}
	\vspace{-0.6cm}
	\caption{L-ECE and HDR-ECE on datasets sorted by size, using a convex potential flow model.}
	\vspace{-0.4cm}
	\label{fig:barplot/calibration}
\end{figure}

In \cref{fig:barplot/calibration}, while \LR does not explicitly target HDR calibration, we observe that it significantly improves HDR-ECE compared to \BASE on most datasets, often performing on par with \HDRR. As expected, \HDRR achieves low HDR-ECE values by design. These results suggest that improving latent calibration also enhances the calibration of HDRs.

Access to a full, calibrated PDF is crucial for any task requiring estimation of the probability mass within an \textit{arbitrary, non-standard region} of the output space, a capability that set-based methods (like CP) or pure sampling-based methods (like \HDRR) do not provide. A direct application is anomaly detection, where low-density points are classified as anomaly \citep{Rozner2024-cj,Perini2024-tn}. Other examples include risk assessment in engineering, targeted material design, or optimal control. To make the benefits of a full PDF concrete, we provide an experiment on a decision-making task in \cref{sec:decision_making_experiment}.

Further detailed results with the same convex potential flow base predictor are provided in \cref{sec:results_cpf}. \cref{sec:NLL_performance_gain} shows that the primary NLL gain obtained by \LR is due to finding more “plausible” latent codes for the observed data under the base latent distribution. \cref{sec:computational_efficiency} shows that \LR significantly improves the time to compute the calibration map
compared to \HDRR. \cref{sec:discriminative_ability} hypothesizes through theoretical and empirical considerations that the reason the ES remains largely unchanged after \LR is due to its relative insensitivity to misspecifications in variance and correlation.

\paragraph{Image dataset.} The goal of our image data experiment is to understand the behaviour of latent calibration and recalibration with high-dimensional outputs.
\cref{table:lr_tarflow} shows that \BASE suffers from severe latent miscalibration, with L-ECE values approaching the maximum of 0.5. \LR dramatically improves latent calibration, reducing L-ECE to below 0.01. We also report the bits-per-dimension (BPD), a scaled version of the NLL. Notably, \LR does not degrade the original NLL; in fact, it slightly reduces it. \LR preserves the visual quality of the samples from the base model, with no perceptually visible changes, which aligns with the very small change we observed in NLL.

\begin{table}[t]
	\small
	\caption{Performance of \LR compared to \BASE on the AFHQ dataset with TarFlow (standard errors across 20 evaluations).}
	\label{table:lr_tarflow}
	\vspace{0.2cm}
	\centering
	\begin{tabular}{llll}
		\toprule
		\multicolumn{2}{c}{L-ECE} & \multicolumn{2}{c}{BPD} \\
		\BASE & \LR & \BASE & \LR \\
		\midrule
		$\text{0.474}_{\text{0.000625}}$ & $\text{0.00895}_{\text{0.00160	}}$ & $\text{5.477}_{\text{3.523e-05}}$ & $\text{5.465}_{\text{1.772e-05}}$ \\
		\bottomrule
	\end{tabular}
\end{table}

%% file: tex/sec-conclusion.tex
\section{Conclusion and limitations}
\label{sec:conclusion}

We introduced latent recalibration (\LR), a novel post-hoc method for calibrating conditional normalizing flows in multi-output regression. By transforming the latent space based on calibration scores derived from latent distances, \LR achieves latent calibration, ensuring that prediction regions defined in the latent space have correct coverage. Unlike many conformal prediction methods that only output sets, and unlike sampling-based recalibration methods, \LR yields a fully specified, recalibrated PDF. This offers significant advantages in terms of computational efficiency and applicability to tasks requiring density estimates. Our extensive experiments on tabular and high-dimensional image data demonstrate that \LR consistently improves NLL, latent calibration, and HDR calibration.

We identify the main limitations of \LR as follows.
Firstly, \LR intentionally adjusts only the magnitude of latent vectors, not their direction, and thus cannot fix miscalibration arising from errors in the orientation of the learned latent manifold. While \LR can only perform simple adjustments, this allows simplifying the difficult multivariate calibration problem into a tractable univariate one (calibrating norms). This enables connections with conformal prediction and recalibration methods, and has good empirical performance.
Secondly, \LR requires the norm of the latent distribution to follow a simple distribution. This is usually the case, as normalizing flows predominantly use a standard Gaussian latent variable.
Thirdly, \LR requires an invertible transformation between the response and latent spaces, and a latent random variable with a known, tractable density, which makes it incompatible with models such as variational auto-encoders \citep{Kingma2014-wu} or denoising diffusion probabilistic models \citep{Ho2020-rc}. Instead, NF and FM models are natural fits for \LR.
Despite these considerations, \LR provides a practical and effective tool for obtaining reliable, calibrated multivariate predictive distributions from generative models.

%% file: tex/suppl-datasets.tex
\section{Datasets}
\label{sec:datasets}

\cref{table:datasets} shows the datasets considered in our experiments, with associated reference papers. The datasets are characterized by their total number of instances, number of features $p$ and number of outcomes $d$. The preprocessing follows the setup described in \cite{Grinsztajn2022-nu}.

\begin{table}[H]
	\caption{Lists of evaluated tabular datasets.}
	\vspace{0.2cm}
	\label{table:datasets}
    \centering
	\input{tables/datasets.tex}
\end{table}

%% file: tables/datasets.tex
\begin{tabular}{lllrrr}
\toprule
 &  &  & Total size & $p$ & $d$ \\
Paper & Dataset & Abbreviation &  &  &  \\
\midrule
\multirow[t]{13}{*}{\cite{Tsoumakas2011-wf}} & slump & SLU & 103 & 7 & 3 \\
 & edm & EDM & 154 & 16 & 2 \\
 & atp7d & AT2 & 296 & 355 & 6 \\
 & sf1 & SF1 & 323 & 31 & 3 \\
 & oes97 & OE2 & 334 & 263 & 16 \\
 & atp1d & AT1 & 337 & 354 & 6 \\
 & jura & JUR & 359 & 15 & 3 \\
 & oes10 & OE1 & 403 & 298 & 16 \\
 & enb & ENB & 768 & 3 & 2 \\
 & wq & WQ & 1060 & 16 & 14 \\
 & sf2 & SF2 & 1066 & 31 & 3 \\
 & scpf & SCP & 1137 & 8 & 3 \\
\cite{del-Barrio2024-la} & ansur2 & ANS & 1986 & 1 & 2 \\
\cite{Camehl2024-vw} & households & HO2 & 7207 & 14 & 4 \\
\multirow[t]{4}{*}{\cite{Tsoumakas2011-wf}} & scm20d & SC2 & 8966 & 60 & 16 \\
 & rf1 & RF1 & 9005 & 64 & 8 \\
 & scm1d & SC1 & 9803 & 279 & 16 \\
\multirow[t]{4}{*}{\cite{Cevid2022-ev}} & births2 & BI2 & 10000 & 24 & 4 \\
 & air & AIR & 10000 & 15 & 6 \\
 & births1 & BI1 & 10000 & 23 & 2 \\
 & wage & WAG & 10000 & 78 & 2 \\
\multirow[t]{6}{*}{\cite{Feldman2023-cc}} & meps\_21 & ME3 & 15656 & 138 & 2 \\
 & meps\_19 & ME1 & 15785 & 138 & 2 \\
 & meps\_20 & ME2 & 17541 & 138 & 2 \\
 & house & HO1 & 21613 & 17 & 2 \\
 & bio & BIO & 45730 & 8 & 2 \\
 & blog\_data & BLO & 50000 & 269 & 2 \\
 \cite{del-Barrio2024-la} & calcofi & CAL & 50000 & 1 & 2 \\
\cite{Wang2023-vn} & taxi & TAX & 50000 & 4 & 2 \\
\bottomrule
\end{tabular}

%% file: tex/suppl-related-work.tex
\section{Extended related work}
\label{sec:suppl-related_work}

\paragraph{Normalizing flows.}
Various NF architectures exist, including RealNVP \citep{Dinh2016-ux}, MAF \citep{Papamakarios2017-uh}, Glow \citep{Kingma2018-hp}, spline flows \citep{Durkan2019-sc}, convex potential flows \citep{Huang2020-md} and transformer flows \citep{Zhai2024-ex}. NFs are well-suited for \LR due to their invertible mapping and explicit density.
Radial flows \citep{Rezende2015-hx} are a special case of \cref{eq:R_transform_final} with $r(t) = t \cdot (\alpha + t + \beta) / (\alpha + t)$ where $\alpha \in \R_+$ and $\beta \in \R$ are learned parameters.

\paragraph{Conditional notions of calibration.}
Multiple extensions of notions of calibration have been proposed by imposing different conditions, including group calibration \citep{Pleiss2017-zu}, distribution calibration (or auto-calibration) \citep{Song2019-bk, Tsyplakov2013-mm}, individual calibration \citep{Zhao2020-ze}, and threshold calibration \citep{Sahoo2021-hf}. Latent calibration could be extended by imposing similar conditions.

\paragraph{Calibration methods.}
Improving calibration often involves post-hoc recalibration or regularization during training.
\textit{Recalibration} methods adjust pre-trained models. Notable methods include \cite{Kuleshov2018-tb,Kuleshov2022-pv,Chung2024-zd}.
\textit{Regularization} methods incorporate calibration objectives into training \citep{Utpala2020-nw, Marx2023-zj, Dheur2023-bo}.
However, these methods target univariate settings and can trade off predictive accuracy (NLL, CRPS) for calibration \citep{Yoon2023-pa,Dheur2023-bo}. \citet{Dheur2024-zs} integrated \QR end-to-end into training, followed by post-hoc recalibration.

%% file: tex/suppl-proofs.tex
\section{Proofs}

\subsection{Jacobian determinant of the radial transform}
\label{sec:proof_radial_ldj}

Recall that the transformation $R$ is defined as
\begin{equation}
	R(z) = \frac{r(l)}{l} z
\end{equation}
where $l = \norm{z}$, for $z \neq 0$. It maps $z$ to a new vector $R(z)$ such that its norm becomes $r(l)$ while its direction $z/l$ is preserved (for $z \neq 0$).
We analyze this transformation using hyperspherical coordinates $(l, \omega_1, \dots, \omega_{d-1})$, where $l=\norm{z}$ is the radial distance and $(\omega_1, \dots, \omega_{d-1})$ are the angular coordinates. The transformation $R$ maps these coordinates from $(l, \omega_1, \dots, \omega_{d-1})$ to $(r(l), \omega_1, \dots, \omega_{d-1})$, as only the radial distance is altered.

The Cartesian volume element $\mathrm{d}^d z$ is related to the hyperspherical volume element by $\mathrm{d}^d z = l^{d-1} \mathrm{d}l \, \mathrm{d}\Omega_{d-1}$, where $\mathrm{d}\Omega_{d-1}$ is the surface element on the unit $(d-1)$-sphere.
Under the transformation $R$, the new radial coordinate is $l' = r(l)$, so its differential is $\mathrm{d}l' = \frac{\partial r(l)}{\partial l} \mathrm{d}l$. The angular part $\mathrm{d}\Omega_{d-1}$ remains unchanged.
The transformed volume element $\mathrm{d}^d R(z)$ is thus given by:
\begin{equation}
\mathrm{d}^d R(z) = (r(l))^{d-1} \left(\frac{\partial r(l)}{\partial l} \mathrm{d}l\right) \mathrm{d}\Omega_{d-1}.
\end{equation}
The Jacobian determinant $\left| \det\left( \nabla_z R(z) \right) \right|$ is the ratio of the transformed volume element $\mathrm{d}^d R(z)$ to the original volume element $\mathrm{d}^d z$:
\begin{equation}
\left| \det\left( \nabla_z R(z) \right) \right| = \frac{(r(l))^{d-1} \frac{\partial r(l)}{\partial l} \mathrm{d}l \, \mathrm{d}\Omega_{d-1}}{l^{d-1} \mathrm{d}l \, \mathrm{d}\Omega_{d-1}} = \left( \frac{r(l)}{l} \right)^{d-1} \frac{\partial r(l)}{\partial l},
\end{equation}
which corresponds to \cref{eq:log_det_jac_R}.
This holds for $l = \norm{z} > 0$. Since $r: \R_+ \to \R_+$, $r(l) \ge 0$. Furthermore, $r(l)$, as defined by \cref{eq:r_transform_final}, is a composition of non-decreasing functions (a CDF and an inverse CDF), making it non-decreasing, so $\frac{\partial r(l)}{\partial l} \ge 0$. Thus, the expression is inherently non-negative.

%% file: tex/suppl-smooth-approx.tex
\section{Differentiable calibration maps using density estimation}
\label{sec:density_estimation}

To obtain a differentiable calibration map $\Fh_{\hat{L}}$, we estimate the PDF $\fh_{\hat{L}}$ of the calibration data using density estimation. We identified two approaches that performed well in our experiments.

As an implementation detail for both approaches, density estimation was generally improved by first applying the transformation $g(t) = t^{1/3}$ to the calibration scores $\{\hat{L}_i\}_{i=1}^n$. After density estimation in the transformed space, the data is rescaled using the inverse transformation $g^{-1}(t) = t^3$.

\subsection{Kernel density estimation with Gamma kernels}

We found that kernel density estimation (KDE) with Gamma kernels is effective because the Gamma distribution has positive support, which is appropriate for the calibration scores $\hat{L}_i \ge 0$. Let $\Gamma(\zeta, \lambda)$ denote a Gamma distribution with shape $\zeta > 0$ and rate $\lambda > 0$. A Gamma distribution $\Gamma(\mu \lambda, \lambda)$ has a mean of $\mu$. We center a Gamma kernel at each calibration score $\hat{L}_i$, using the distribution $\Gamma(\hat{L}_i\lambda, \lambda)$, which has a mean of $\hat{L}_i$.

The resulting estimated CDF $\Fh_{\hat{L}}(t)$ is given by the average of the individual kernel CDFs:
\begin{align}
    \Fh_{\hat{L}}(t) = \frac{1}{n} \sum_{i = 1}^n F_{\Gamma(\hat{L}_i \lambda, \lambda)}(t).
\end{align}
The rate parameter $\lambda$ is chosen by minimizing the NLL of the calibration dataset under the KDE model. This is done using 10-fold cross-validation over the grid $\left\{ 10^{-5 + 10 \cdot \frac{i}{99}}\right\}_{i = 0}^{99}$. This hyperparameter selection process is efficient and performed once per run.

\cref{density_estimation/kde} shows an example fit on all datasets, illustrating the empirical and estimated smooth CDFs (left $y$ axis) and the estimated log PDF (right $y$ axis).

\begin{figure}[H]
	\includegraphics[width=\linewidth]{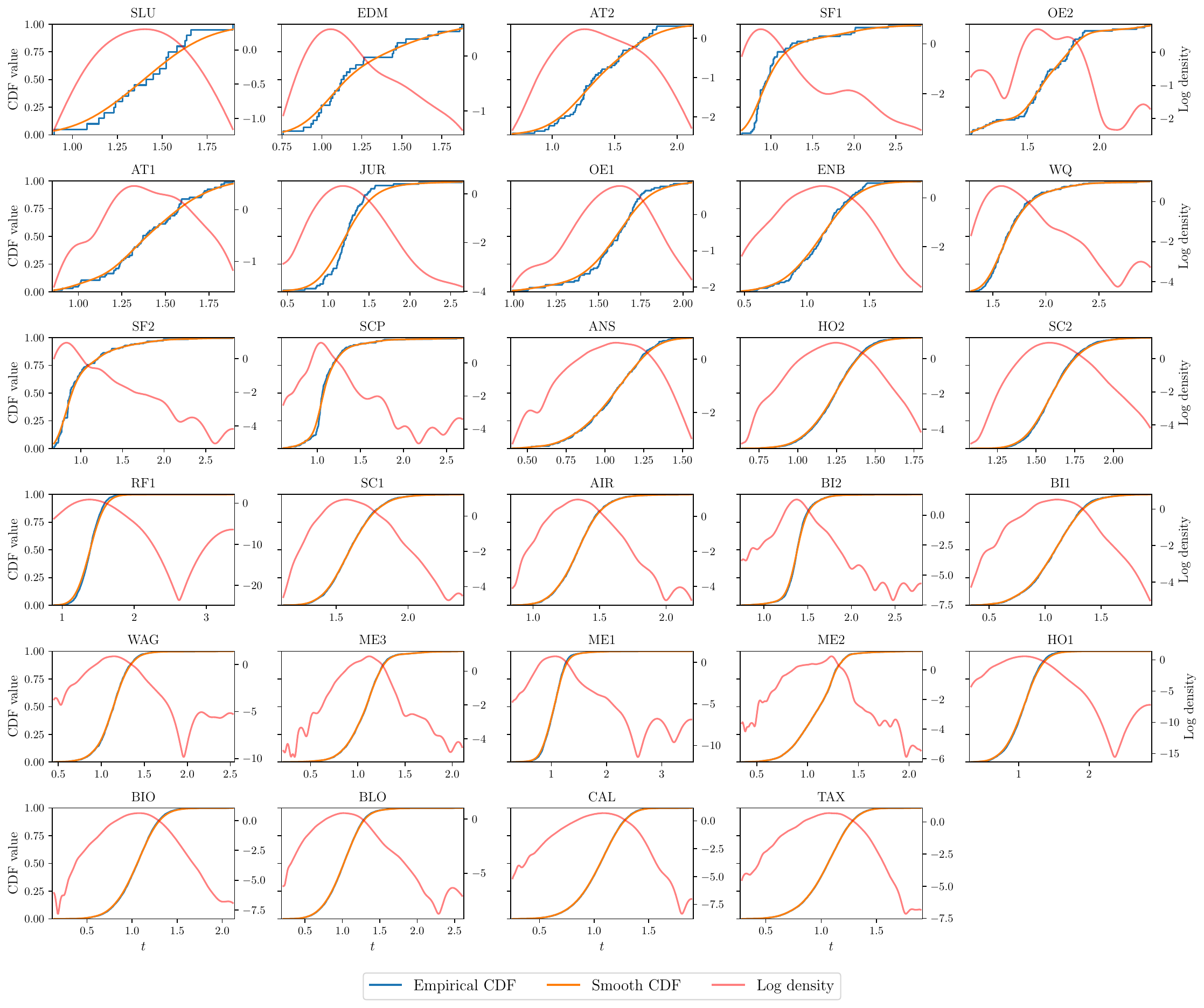}
	\caption{Density estimation using KDE with a Gamma kernel.}
	\label{density_estimation/kde}
\end{figure}

\subsection{Rational Quadratic Splines}

Rational Quadratic Splines \citep{Durkan2019-sc} provide a flexible framework for defining invertible and differentiable transformations. A base spline $\Phi$ maps $[-1, 1]$ to $[-1, 1]$. To handle the unbounded domain of the latent norms, we use the transformation $\Psi = \text{tanh}^{-1} \circ \Phi \circ \text{tanh}$, which maps $\R$ to $\R$ and retains invertibility and differentiability. This transformation is used to model the distribution of the latent norms by learning a mapping from a standard Gaussian distribution to the data distribution.

For training, the data is normalized to have zero mean and unit variance. The parameters of the spline are optimized to minimize the NLL of the calibration dataset under the defined model. Optimization is performed using Adam \citep{Kingma2014-pd}. To maximize data information, we perform early stopping on the training dataset itself and stop if the loss did not improve by 1e-4 for 50 epochs. Overfitting is prevented by limiting the number of bins and thus the flexibility of the spline. Specifically, we use 4 bins if $n \leq 30$, 5 bins if $n \leq 50$, 6 bins if $n \leq 70$, 7 bins if $n \leq 80$, 8 bins if $n \leq 90$, and 9 bins if $n \leq 100$.

Similar to \cref{density_estimation/kde}, \cref{density_estimation/spline} shows an example fit for all datasets.

\begin{figure}[H]
	\includegraphics[width=\linewidth]{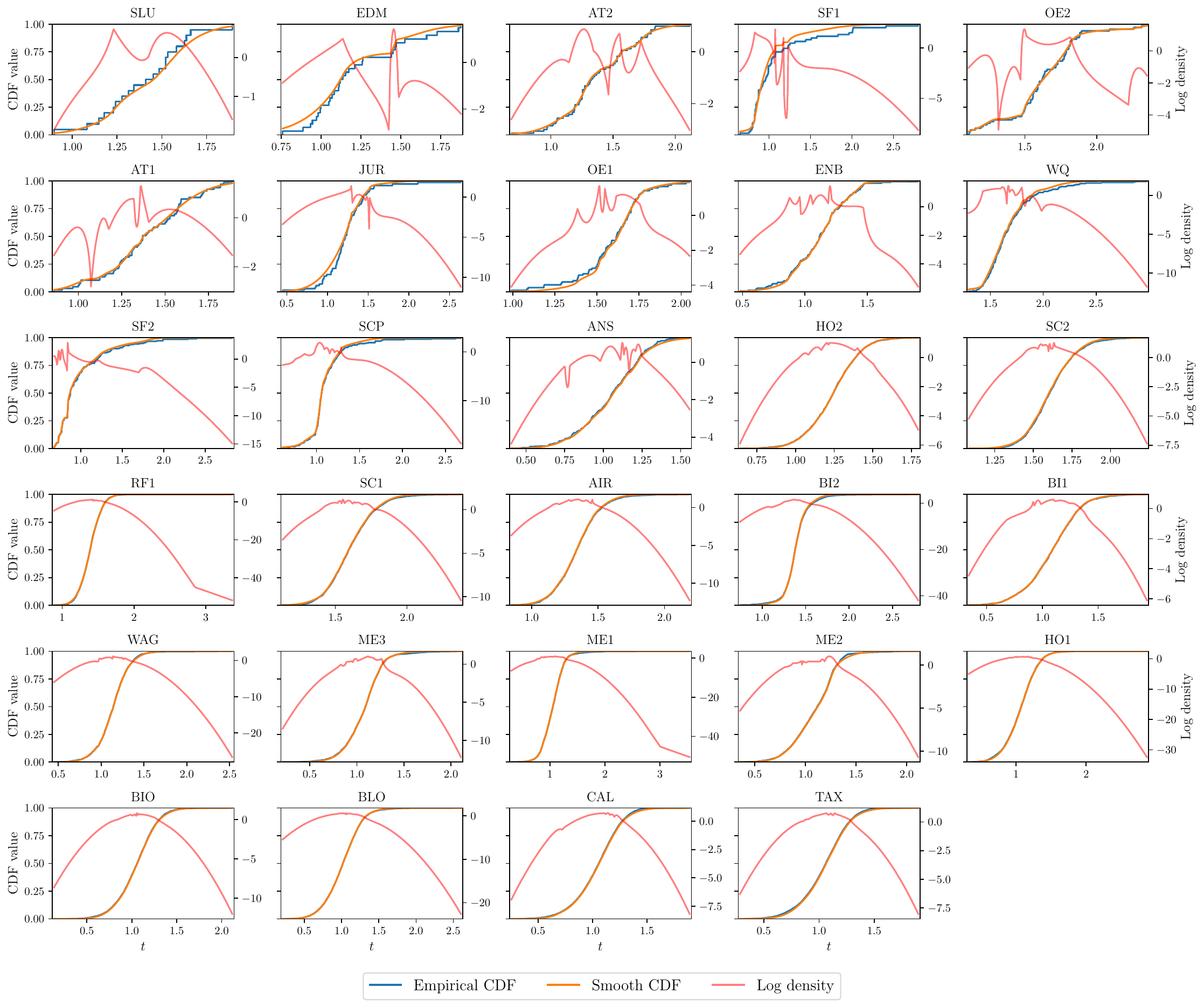}
	\caption{Density estimation using a rational quadratic spline.}
	\label{density_estimation/spline}
\end{figure}

\subsection{Challenges with numerical precision}

In this section, we address a potential alternative approach and explain why it is impractical due to numerical precision constraints. Theoretically, one could attempt to estimate the calibration map $\Fh_{\hat{U}}$ using density estimation on the quantity $\hat{U} = F_{\rho_{\Z}(Z)}(\ell_{\Th}(Y; X))$. Since $\hat{U}$ is expected to follow a standard uniform distribution under ideal conditions, estimating its density might appear practical.

However, this approach faces significant numerical precision issues, particularly when the latent space dimensionality $d$ is large. When $d$ is large, the CDF $F_{\rho_{\Z}(Z)}$ for $\rho_{\Z}(Z) \sim \chi_d$ becomes extremely steep around its mode $\sqrt{d - 1}$. For example, in our image application with $d = 196,608$, a proportion $99\%$ of the probability mass is concentrated in the narrow interval $[440.7, 446.2]$. In single-precision floating-point arithmetic, the CDF saturates quickly: $F_{\chi_d}(t)$ is numerically $0.0$ for $t < 433.4$ and $1.0$ for $t > 447.2$.

If the latent model is miscalibrated, the values of $\ell_{\Th}(Y; X)$ for the calibration data can fall outside this narrow range where the CDF has fine-grained variation. This results in many $\hat{U}$ values being numerically $0.0$ or $1.0$. For smaller dimensions, similar issues can occur, although less frequently. For instance, with $d = 1$, $F_{\chi_1}(t)$ is numerically $1.0$ for $t > 5.54$ in single precision. When a significant portion of the calibration data for $\hat{U}$ consists of values numerically identical to $0.0$ or $1.0$, accurate density estimation becomes impossible.

For this reason, in \cref{sec:latent_recalibration}, we based our calibration map on density estimation of $\hat{L} = \ell_{\Th}(Y; X)$ directly, using $\Fh_{\hat{L}}$.

%% file: tex/suppl-experiments-details.tex
\section{Additional details on experimental setup}
\label{sec:additional_details}

Computing the main tabular data results requires approximately 24 hours on an RTX A6000 GPU, and reproducing the image results requires approximately 6 hours on an RTX 6000 GPU. Experiments can require up to 48 GB of VRAM, primarily due to large batch sizes during sampling for evaluation metrics. Decreasing the batch size reduces VRAM requirements but increases computation time.

For tabular datasets, we tune hyperparameters using grid search, selecting those that yield the lowest NLL on the validation set. For the convex potential flow, the number of units in the input convex neural network is chosen from $[10, 20, 40]$, the number of layers from $[2, 3, 5]$, and the learning rate from $[5\times10^{-3}, 10^{-3}, 2\times10^{-4}]$. All models are trained by minimizing the NLL with the Adam optimizer \citep{Kingma2014-pd} using a batch size of 1024.

\subsection{Evaluation metrics}
\label{sec:evaluation_metrics}

\paragraph{NLL.}

We compute the average NLL over the test set as $\D_{\text{test}}$:
\begin{align}
    \widehat{\text{NLL}} = \frac{1}{|\D_{\text{test}}|} \sum_{(X, Y) \in \D_{\text{test}}} -\log \hat{f}_{Y|X=X}(Y).
\end{align}

\paragraph{L-ECE.}
For each test point $(X^{(i)}, Y^{(i)}) \in \D_{\text{test}}$, we compute the PIT of the latent norm $\hat{U}_i = F_{\rho_{\Z}(Z)}\left( \ell_{\Th}(Y^{(i)}; X^{(i)}) \right)$. The Latent Expected Calibration Error (L-ECE) is then estimated as the $L_1$ distance between the empirical CDF of $\{\hat{U}_i\}_{i=1}^{|\D_{\text{test}}|}$ and the uniform CDF:
\begin{align}
    \widehat{\text{L-ECE}} = \frac{1}{|\D_{\text{test}}|} \sum_{j = 1}^{|\D_{\text{test}}|} \left|\hat{U}_{(j)} - \frac{j}{|\D_{\text{test}}| + 1}\right|,
\end{align}
where $\hat{U}_{(j)}$ denotes the $j$-th order statistic of the computed PIT values.

\paragraph{Energy Score.}
For each test point $(X, Y) \in \D_{\text{test}}$, we generate two independent sets of $K$ samples, $\mc{S}_x$ and $\mc{S}'_x$, from the predictive distribution $\hat{f}_{Y|X=x}(\cdot)$. The Energy Score (ES) is estimated as:
\begin{align}
    \widehat{\text{ES}} = \frac{1}{|\D_{\text{test}}|} \sum_{(X, Y) \in \D_{\text{test}}} \left( \frac{1}{K} \sum_{\hat{y} \in \mc{S}_x} \|\hat{y} - Y\| - \frac{1}{2K^2} \sum_{\hat{y} \in \mc{S}_x, \hat{y}' \in \mc{S}'_x} \|\hat{y} - \hat{y}'\| \right).
\end{align}
In our experiments, we use $K=100$.

\paragraph{HDR-ECE.}
For each test point $(X^{(i)}, Y^{(i)}) \in \D_{\text{test}}$, we compute $G_i = \text{HPD}_{\hat{f}_{Y|X=X^{(i)}}}(Y^{(i)})$, as defined in \cref{table:comparison_calibration}. The HDR Expected Calibration Error (HDR-ECE) is estimated similarly to L-ECE:
\begin{align}
    \widehat{\text{HDR-ECE}} = \frac{1}{|\D_{\text{test}}|} \sum_{j = 1}^{|\D_{\text{test}}|} \left|G_{(j)} - \frac{j}{|\D_{\text{test}}| + 1}\right|,
\end{align}
where $G_{(j)}$ is the $j$-th order statistic of the computed HDR pre-ranks. Note that computing the HDR-ECE for \HDRR exactly is not possible as \HDRR does not yield an explicit recalibrated density $\hat{f}'_{Y|X}$. Following \citet{Chung2024-zd}, we use the density $\hat{f}_{Y|X}$ of the original (non-recalibrated) model for \HDRR when evaluating its HDR-ECE.

\paragraph{BPD.}
For image datasets, we report the Bits Per Dimension (BPD), calculated as in \citet{Zhai2024-ex}.
where $d$ is the output dimensionality (e.g., $d = 3 \times 256 \times 256$ for AFHQ). The BPD is then:
\begin{align}
    \widehat{\text{BPD}} = (\widehat{\text{NLL}} / d + \log 128) / \log 2.
\end{align}
Here, the $\log 128$ term accounts for the scaling of pixel values from $[0, 255]$ to $[-1, 1]$, and division by $\log 2$ converts the NLL from nats to bits.

\paragraph{Relative NLL or ES.}
To better visualize improvements in NLL or ES relative to the baseline model \BASE, we report the difference in these scores, normalized by the absolute value of the score of \BASE. For example, the relative NLL for \LR is computed as $(\widehat{\text{NLL}}_{\LR} - \widehat{\text{NLL}}_{\BASE}) / |\widehat{\text{NLL}}_{\BASE}|$. A negative value indicates improvement by \LR.

\subsection{HDR recalibration}
\label{sec:hdr_recalibration}

For completeness, \cref{alg:pre-rank-recal} provides the exact recalibration procedure of the \HDRR baseline \citep{Chung2024-zd} introduced in \cref{sec:recalibration_methods}. Instead of the HDR recalibration algorithm in \cite{Chung2024-zd}, we present a direct generalization to any pre-rank $g: \X \times \Y \to \R$, which we call pre-rank recalibration. \HDRR is a special case when $g(x, y) = -\hat{f}_{Y|X=x}(y)$. 

\begin{algorithm}[t]
	\caption{Pre-rank recalibration.}
	\label{alg:pre-rank-recal}
	\begin{algorithmic}[1]
		\STATE \textbf{Input:} Calibration dataset $\D_\text{cal}$, pre-rank $g: \X \times \Y \to \R$, number of samples $K$, number of bins $B$, base predictor with predictive distribution $\hat{P}_{Y|X}$, test input $x_\text{test}$.
		\STATE \textbf{Calibration:}
		\STATE \textbf{for} $(X^{(i)}, Y^{(i)}) \in \D_\text{cal}$
		\STATE \quad \textbf{for} $k = 1$ to $K$
		\STATE \qquad $\hat{Y}^{(k, i)} \sim \hat{P}_{Y|X=X^{(i)}}$
		\STATE \qquad $\hat{G}^{(k,i)} \gets g\left(X^{(i)}, \hat{Y}^{(k, i)}\right)$
		\STATE \quad Define $\hat{F}_{\hat{G}\mid X=X^{(i)}}(c) = \frac{1}{K} \sum_{k=1}^K \indicator\left( \hat{G}^{(k,i)} \leq c \right)$
		\STATE \quad $\hat{U}_i \gets \hat{F}_{\hat{G}\mid X=X^{(i)}}\left(g\left(X^{(i)}, Y^{(i)}\right)\right)$
		\STATE Define $\hat{F}_{\hat{U}}(u) = \frac{1}{|\D_\text{cal}|} \sum_{i=1}^{|\D_\text{cal}|} \indicator\left(\hat{U}_i \leq u\right)$ \COMMENT{Calibration map}
		\STATE \textbf{Prediction:}
		\STATE \textbf{for} $k = 1$ to $K$
		\STATE \quad $\hat{Y}^{(k)} \sim \hat{P}_{Y\mid X=x_\text{test}}$
		\STATE \quad $\hat{G}^{(k)} \gets g\left(x_\text{test}, \hat{Y}^{(k)}\right)$
		\STATE Define a permutation $\pi$ such that $\hat{G}^{(\pi(1))} \leq \dots \leq \hat{G}^{(\pi(K))}$
		\STATE $\mc{S}' \gets \emptyset$ \COMMENT{Initial set of samples}
		\STATE \textbf{for} $b=1$ to $B$
		\STATE \quad $n_b \gets \floor{K \hat{F}_{\hat{U}}(\frac{b}{B})} - \floor{K \hat{F}_{\hat{U}}(\frac{b - 1}{B})}$ \COMMENT{Number of resamples}
		\STATE \quad \textbf{if} $n_b > 0$
		\STATE \qquad $\mc{B}_b \gets \left\{ \left\lfloor \frac{K (b - 1)}{B} \right\rfloor + 1, \dots, \left\lfloor \frac{K b}{B} \right\rfloor \right\}$
		\STATE \qquad $\mc{S}_b \gets \{ \hat{Y}^{(\pi(k))} \}_{k \in \mc{B}_b}$ \COMMENT{Samples pool}
		\STATE \qquad $\{ \tilde{Y}^{(k)} \}_{k=1}^{n_b} \sim P_{\mc{S}_b}$ \COMMENT{Resampling with replacement}
		\STATE \qquad $\mc{S}' \gets \mc{S}' \cup \{ \tilde{Y}^{(k)} \}_{k=1}^{n_b}$
		\RETURN recalibrated predictive samples $\mc{S}'$
	\end{algorithmic}
\end{algorithm}

%% file: tex/suppl-decision-making.tex
\section{Decision-making experiment}
\label{sec:decision_making_experiment}

To make the benefits of a full PDF concrete, we have conducted an experiment on a decision-making task.

\paragraph{Experiment setup.}

We use the SLUMP dataset, where inputs are ingredients for producing concrete and outputs $Y = (S, F, C) \in \mathbb{R}^3$ are three concrete properties. A manufacturer must decide among 3 actions $\mc{A} = \{A, B, D\}$ whether a given batch of ingredients is suitable for one of two projects ($A$ or $B$), each with specific requirement regions, or if it should be discarded ($D$). The decision has different financial utilities and risks:
\begin{itemize}
    \item Requirements for Project A: $7 \leq S \leq 20, 55 \leq F \leq 65, 25 \leq C \leq 40$.
    \item Requirements for Project B: $20 \leq S \leq 29, 70 \leq F \leq 100, 15 \leq C \leq 30$.
\end{itemize}

The expected utility for an agent with policy $a: \X \to \mc{A}$ is given by
\begin{equation}
\E{u(Y, a(X))} \text{ with } u(y, a) = \begin{cases}
    2000, & \text{if $a = A$ and $y \in \text{Region}_A$} \\
    -30, & \text{if $a = A$ and $y \not\in \text{Region}_A$} \\
    1500, & \text{if $a = B$ and $y \in \text{Region}_B$} \\
    -15, & \text{if $a = B$ and $y \not\in \text{Region}_B$} \\
    -10, & \text{if $a = C$.}
\end{cases}
\end{equation}

The optimal action is chosen by maximizing the estimated expected utility. This requires estimating the probabilities $\hat{P}(Y \in \text{Region}_A \mid X)$ and $\hat{P}(Y \in \text{Region}_B \mid X)$, which are computed using two approaches: (1) Monte Carlo estimation with 125 samples, or (2) numerical integration of the PDF over a 5x5x5 grid via the trapezoidal rule.

The agent acts according to the policy $a^*(X) = \argmax_{a \in \{A, B, C\}} u_a(X)$ with
\begin{align*}
    u_A(X) &= 2000 \hat{P}(Y \in \text{Region}_A \mid X) - 30 \hat{P}(Y \not\in \text{Region}_A \mid X) \\
    u_B(X) &= 1500 \hat{P}(Y \in \text{Region}_B \mid X) - 15 \hat{P}(Y \not\in \text{Region}_B \mid X) \\
    u_C(X) &= -10.
\end{align*}

\paragraph{Results.}

\cref{table:decision_making_results} show two key observations:
\begin{enumerate}
    \item Using the PDF via numerical integration leads to better decisions (higher utility) than relying on a finite number of samples.
    \item The improved calibration from \LR provides a more accurate PDF, leading to a significant further increase in utility. \HDRR, which relies on resampling from the original uncalibrated density, actually harmed decision quality in this task.
\end{enumerate}

\begin{table}[t]
  \centering
  \caption{Comparison of estimation strategies and methods. The best performing combination is highlighted in bold.}
  \label{table:decision_making_results}
  \begin{tabular}{llc}
    \toprule
    \textbf{Method} & \textbf{Estimation Strategy} & \textbf{Average Utility} \\
    \midrule
    \BASE   & Sampling                     & 62.53 $\pm$ 11.33 \\
    \HDRR  & Sampling                     & 32.38 $\pm$ 10.81 \\
    \BASE   & PDF (Numerical Integration)  & 76.23 $\pm$ 11.99 \\
    \LR     & \textbf{PDF (Numerical Integration)} & \textbf{113.31 $\pm$ 12.91} \\
    \bottomrule
  \end{tabular}
\end{table}

This demonstrates a concrete scenario where an explicit, calibrated PDF is not just a theoretical advantage but a practical necessity for optimal decision-making.

%% file: tex/suppl-results.tex
\section{Examples of predictive distributions on real-world tabular datasets}
\label{sec:examples_predictions}

\cref{fig:example_predictions/densities_1,fig:example_predictions/densities_2} display examples of predictive PDFs on real-world tabular datasets with two-dimensional outputs ($d=2$). Each row corresponds to a different dataset. For each dataset, two random test instances, $(x^{(1)}, y^{(1)})$ and $(x^{(2)}, y^{(2)})$ from $\D_{\text{test}}$, are shown.

Columns 1 and 3 show the predictive densities from the uncalibrated base predictor \BASE (i.e., $\hat{f}_{Y|X=x^{(1)}}(\cdot)$ and $\hat{f}_{Y|X=x^{(2)}}(\cdot)$). Columns 2 and 4 show the corresponding predictive densities from the \LR-recalibrated model (i.e., $\hat{f}'_{Y|X=x^{(1)}}(\cdot)$ and $\hat{f}'_{Y|X=x^{(2)}}(\cdot)$). All densities are visualized in orange. The true target observations ($y^{(1)}$ and $y^{(2)}$) are marked with a blue dot. The negative log-likelihood of the true target under the respective predictive density is provided in the bottom right corner of each plot. Black contour lines indicate level sets of the PIT of the latent norm ($F_{\rho_{\Z}(Z')}(\ell_{\Th}(y; x))$ for the \LR model, and $F_{\rho_{\Z}(Z)}(\ell_{\Th}(y; x))$ for the \BASE model) at probability levels 0.01, 0.1, 0.5, and 0.9.

In many cases, when the \BASE model is already reasonably well-calibrated, \LR applies a subtle adjustment that is difficult to perceive visually. In other instances, the recalibration effect is more pronounced, visibly altering the shape and spread of the predictive distribution to better align with latent calibration. Note that two-dimensional datasets often benefit from smaller NLL improvements according to \cref{table:lr/MQF2_scoring_rules_lr}, suggesting that stronger adjustments should be perceived in higher dimensions.

\begin{figure}[H]
	\centering
	\includegraphics[width=\linewidth]{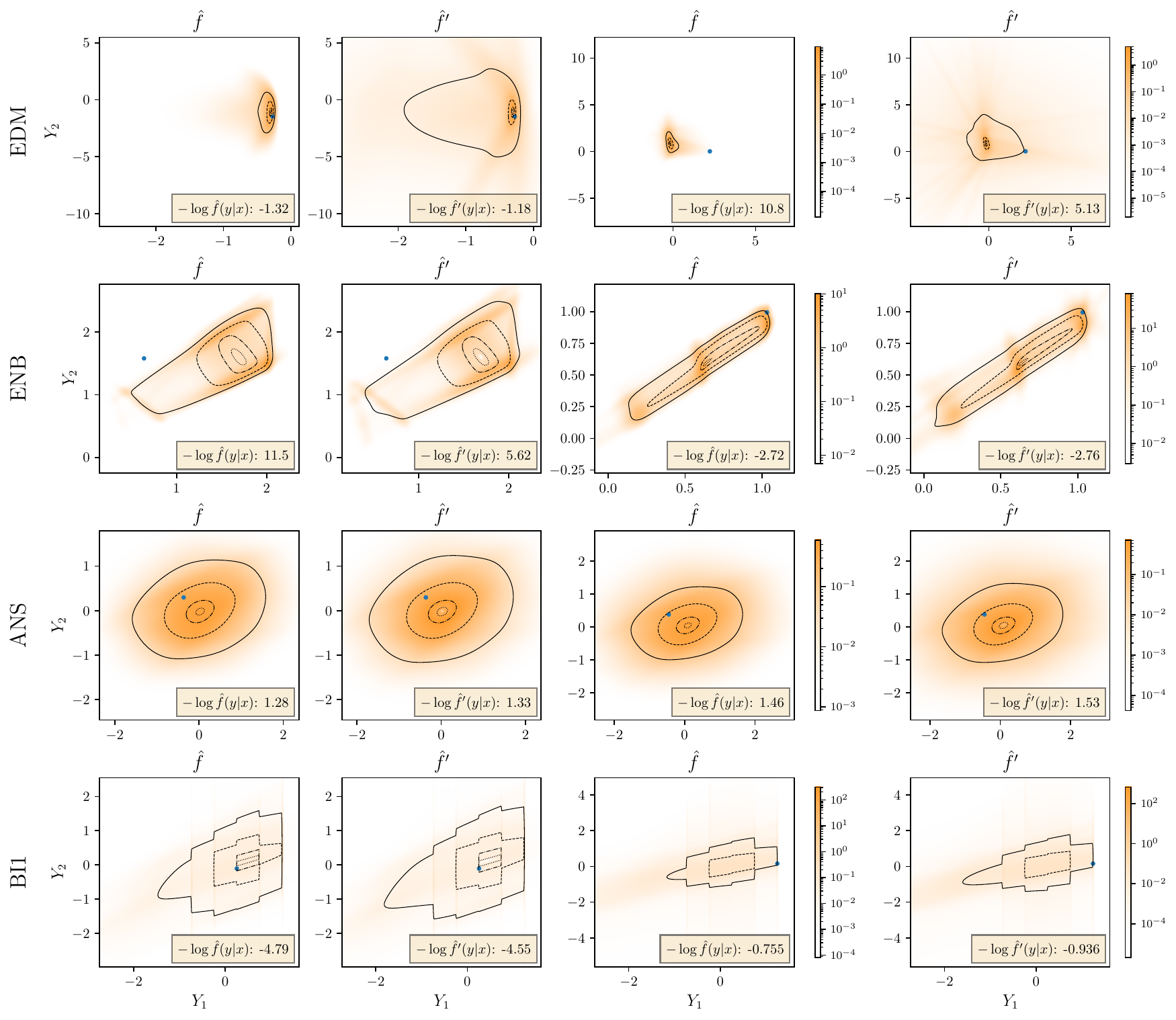}
	\vspace{-0.2cm}
	\caption{Examples of 2D predictive densities on real-world datasets for random test points $(x, y) \in \D_{\text{test}}$.}
	\label{fig:example_predictions/densities_1}
\end{figure}

\begin{figure}[H]
	\centering
	\includegraphics[width=\linewidth]{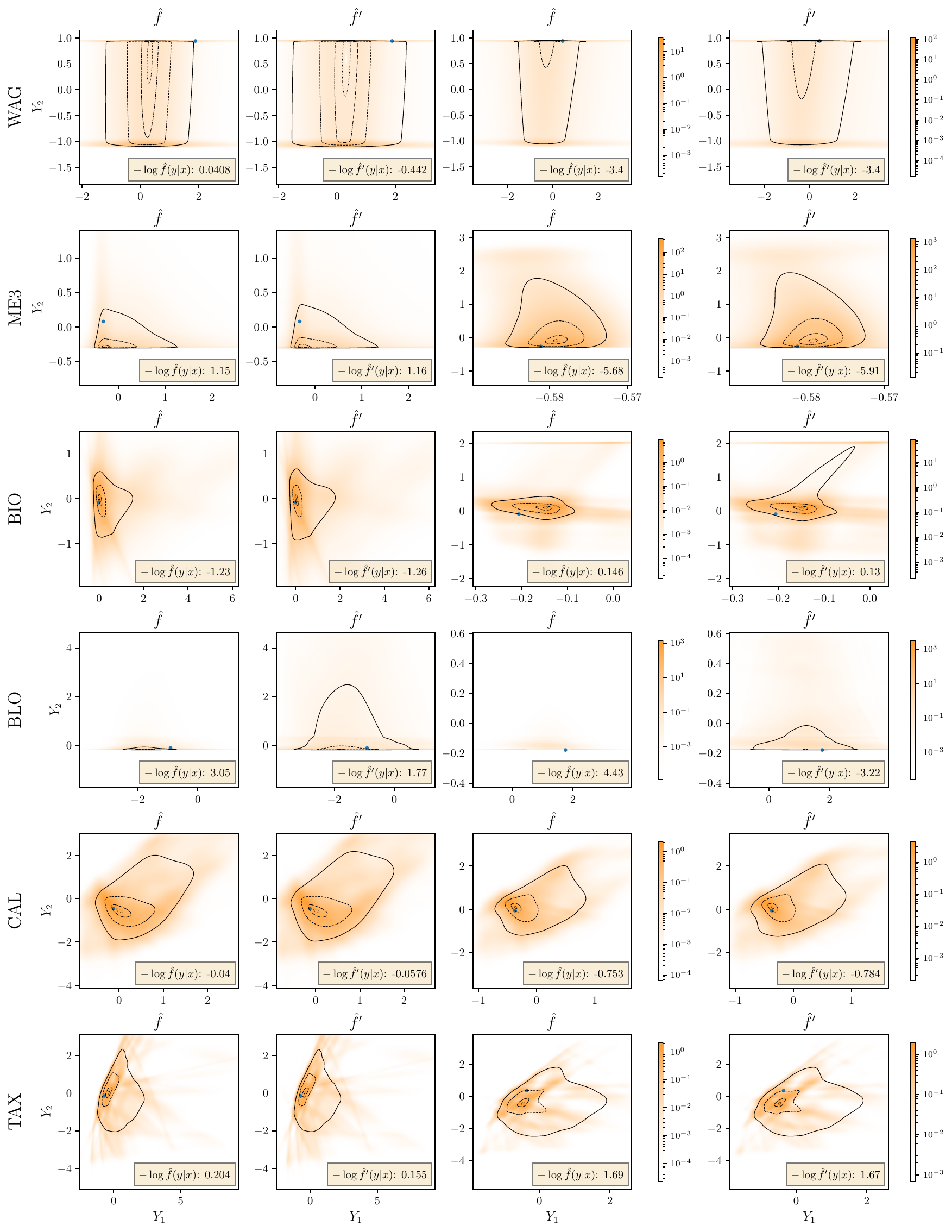}
	\vspace{-0.2cm}
	\caption{Examples of 2D predictive densities on real-world datasets for random test points $(x, y) \in \D_{\text{test}}$.}
	\label{fig:example_predictions/densities_2}
\end{figure}

\section{Reliability diagrams}
\label{sec:reliability_diagrams}

\cref{fig:reliability_diagrams/latent_distance} shows reliability diagrams for latent calibration. These diagrams plot the nominal probability levels $\alpha \in [0,1]$ against the empirical probabilities $\hat{F}_{\hat{U}}(\alpha)$, where $\hat{U} = F_{\rho_{\Z}(Z)}\left( \ell_{\Th}(Y; X) \right)$ are the PIT values computed on the test set. We also report 90\% consistency bands, represented by the shaded area around the diagonal, as described by \citet{Gneiting2023-at}. The \BASE model often exhibits miscalibration (deviations from the diagonal), whereas \LR consistently aligns closely with the diagonal, demonstrating significantly improved latent calibration.

\begin{figure}[H]
	\centering
	\includegraphics[width=\linewidth]{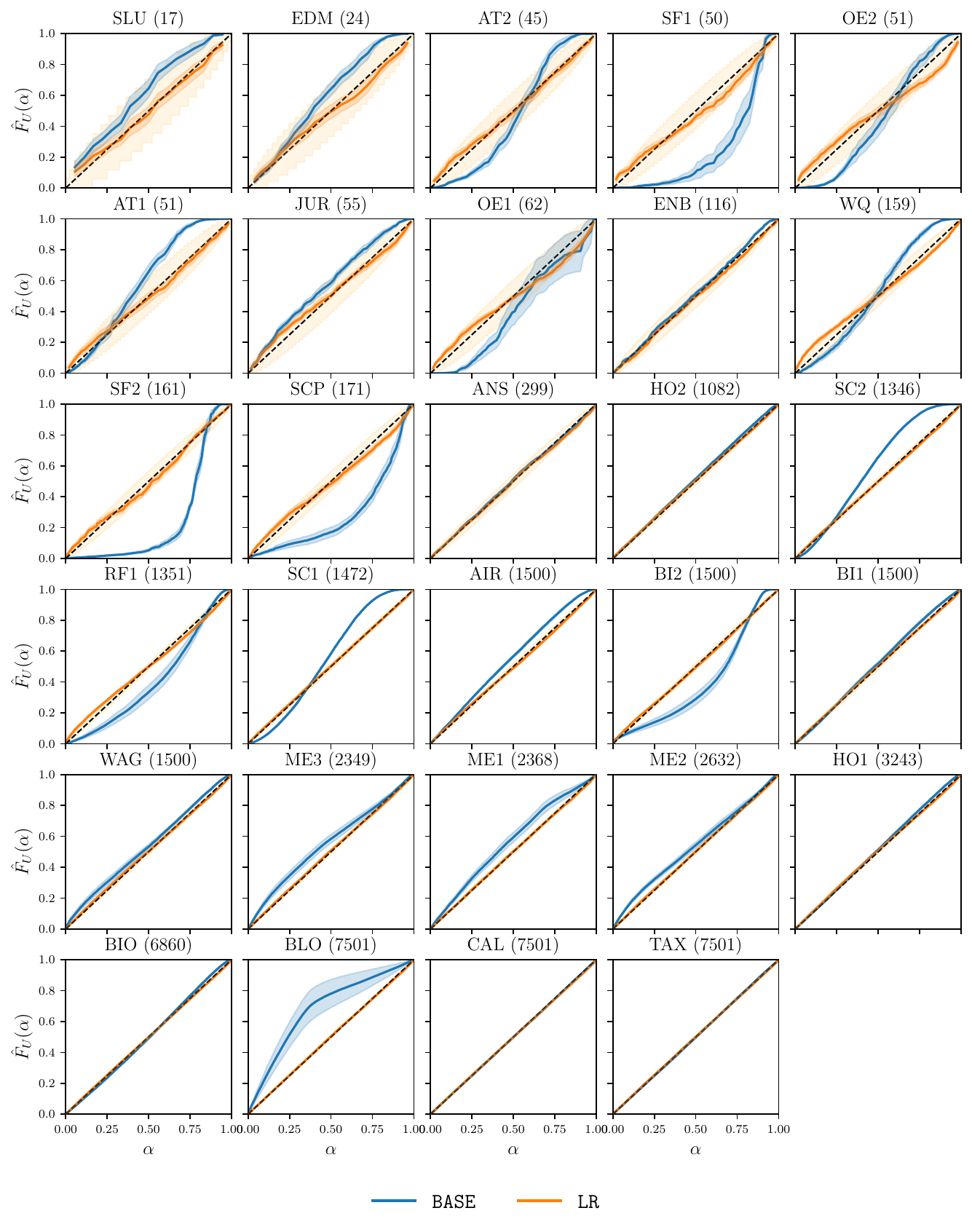}
	\vspace{-0cm}
	\caption{Latent calibration diagrams}
	\label{fig:reliability_diagrams/latent_distance}
\end{figure}

\cref{fig:reliability_diagrams/hpd} shows reliability diagrams for HDR calibration. These diagrams plot nominal probability levels $\alpha$ against empirical probabilities $\hat{F}_{\hat{U}}(\alpha)$, where $\hat{U} = \text{HPD}_{\hat{f}_{Y|X}}(Y)$ are the HDR pre-rank values from the test set. Again, the \BASE model frequently shows miscalibration. Both \LR and \HDRR improve HDR calibration, though for \LR this improvement is a beneficial side effect rather than a direct optimization target, unlike its consistent improvement of latent calibration shown in \cref{fig:reliability_diagrams/latent_distance}.

\begin{figure}[H]
	\centering
	\includegraphics[width=\linewidth]{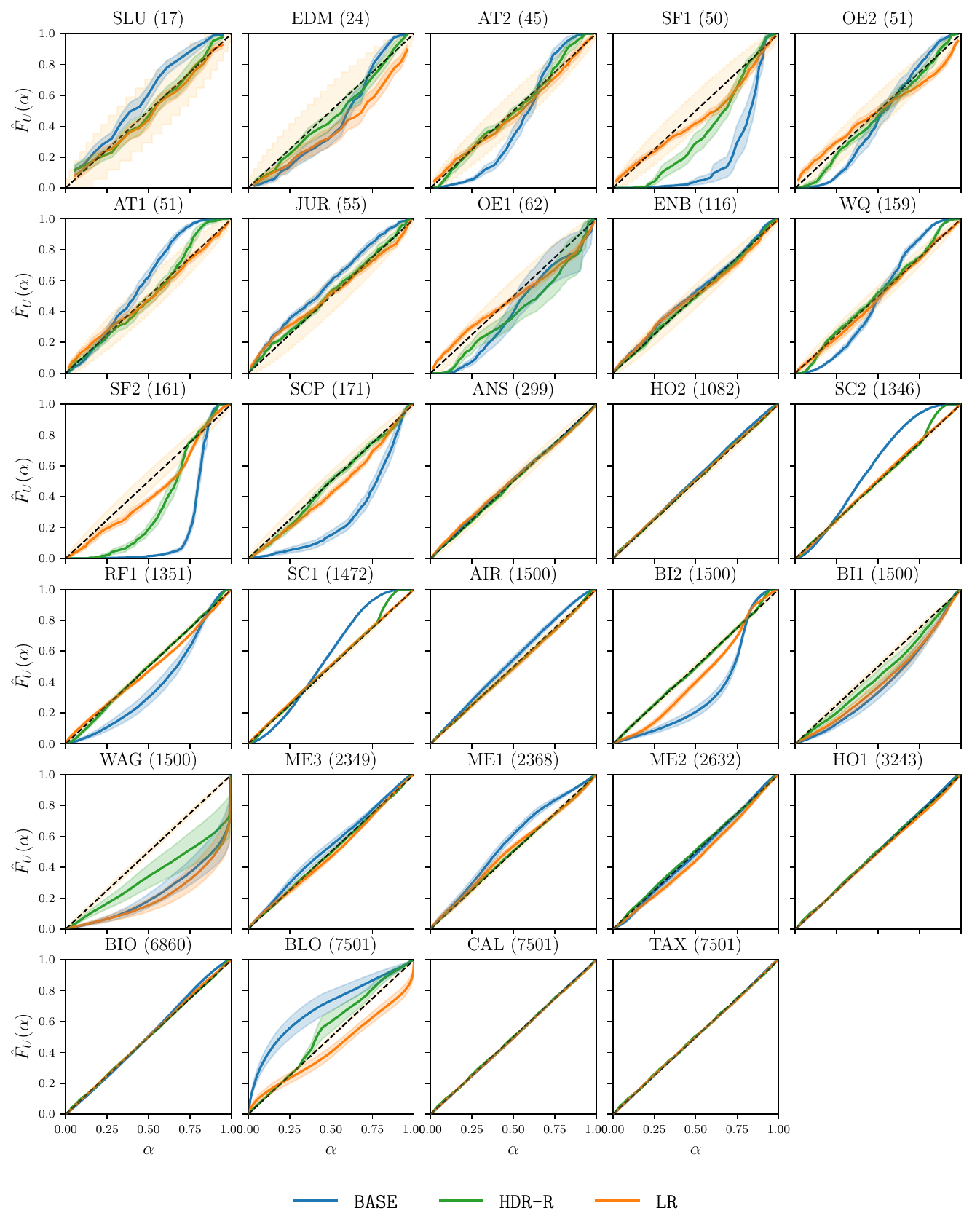}
	\vspace{-0cm}
	\caption{HDR-calibration diagrams}
	\label{fig:reliability_diagrams/hpd}
\end{figure}

\section{Base predictors}
\label{sec:base_predictors}

This section provides details on the base predictors considered in this paper.

\subsection{Convex potential flows}
\label{sec:convex_potential_flows}

Convex potential flows \citep{Huang2020-md} parameterize the bijective transformation $\hat{T}$ via a strongly convex potential whose gradient yields the inverse map. Given $x \in \X$, let the model define a scalar potential $\hat{V}: \Y \times \X \to \R$ where $\hat{V}(\cdot; x)$ is strongly convex for every fixed $x$. The associated transformation is the gradient
\[
	\hat{T}^{-1}(y; x) = \nabla_y \hat{V}(y; x) \in \R^d .
\]
To ensure convexity in $y$, $\hat{V}(\cdot; x)$ is parameterized by an \textit{input-convex neural network} \citep[ICNN,][]{Amos2017-zw}. A small quadratic term $\tfrac{\alpha}{2}\lVert y\rVert^2$ enforces strong convexity of $\hat{V}(\cdot; x)$.

Since $\hat{V}(\cdot; x)$ is strongly convex, $\hat{T}^{-1}(\cdot; x)$ is a bijection $\Y \to \Z$ with inverse $\hat{T}(\cdot; x)$. Generating new conditional samples requires inverting $\hat{T}^{-1}(\cdot; x)$. Given $z \in \Z$, one recovers $\hat{T}(z; x)$ as the unique minimizer of the convex objective
\begin{equation}
	\label{eq:cpflow_inversion}
	\hat{T}(z; x) \in \argmin_{y \in \Y} \left\{ \hat{V}(y; x) - z^\top y \right\}.
\end{equation}
Indeed, since $\hat{V}(\cdot; x)$ is differentiable and strongly convex, the minimum is attained when
\begin{align}
	\nabla_y \left(\hat{V}(y; x) - z^\top y\right) = 0
	\iff \nabla_y \hat{V}(y; x) = z
	\iff \hat{T}^{-1}(y; x) = z
	\iff y = \hat{T}(z; x).
\end{align}
In practice, \eqref{eq:cpflow_inversion} can be solved efficiently with gradient-based convex optimization (e.g., L-BFGS), and strong convexity ensures convergence to a unique solution.

The Hessian $\nabla_y^2 \hat{V}(y; x)$ of  is positive definite, with
\[
	\left| \det \nabla_y \hat{T}^{-1}(y; x) \right| = \det\left( \nabla_y^2 \hat{V}(y; x) \right).
\]
One approach to compute this determinant is by explicitly forming the Hessian $H = \nabla_y^2 \hat{V}(y; x)$. Concretely, $\hat{T}^{-1}(y; x) = \nabla_y \hat{V}(y; x)$ is first evaluated with a single forward pass and backpropagation, and then the Hessian is formed as
\[
	H = \left( \frac{\partial \hat{T}^{-1}(y; x)}{\partial y_1}, \dots, \frac{\partial \hat{T}^{-1}(y; x) }{\partial y_d} \right) \in \R^{d \times d},
\]
requiring $d$ additional backpropagations. Finally, the determinant can be computed explicitly in $O(d^3)$ (e.g., via Cholesky). Hence, the overall computational complexity is $O\left( d\, C_{\text{backprop}} + d^3 \right)$
where $C_{\text{backprop}}$ denotes the cost of a single backpropagation. This brute-force approach is practical for small $d$ but becomes prohibitive as $d$ grows; more efficient strategies for large $d$ are discussed in \cite{Huang2020-md}.

\subsection{Masked autoregressive flows (MAFs)}
\label{sec:masked_autoregressive_flows}

Masked autoregressive flows \citep{Papamakarios2017-uh} implement $\Th^{-1}$ as an autoregressive affine transformation.
For each coordinate $i \in [d]$, the model uses two autoregressive NNs (e.g. MADE, \citep{Germain2015-gs}) to parameterize
\[
	\hat{\mu}_i : \R^{i-1} \times \X \to \R,
	\qquad
	\hat{\rho}_i : \R^{i-1} \times \X \to \R,
\]
and $\hat{\sigma}_i(y_{<i}, x) = \exp(\hat{\rho}_i(y_{<i}, x))$ ensures positive outputs.
The masking mechanism ensures that $\hat{\mu}_i(y_{<i}, x)$ and $\hat{\sigma}_i(y_{<i}, x)$ depend only on the preceding coordinates $y_{<i} = (y_1, \dots, y_{i-1})$ and the conditioning variable $x$.

The inverse transformation for each coordinate is then
\[
	z_i = \frac{y_i - \hat{\mu}_i(y_{<i}; x)}{\hat{\sigma}_i(y_{<i}; x)},
	\qquad i = 1, \dots, d.
\]

Since this mapping is triangular in $y$, the Jacobian of $\Th^{-1}$ is also triangular, which makes the determinant computation efficient:
\[
	\left| \det \nabla_y \Th^{-1}(y; x) \right|
	= \prod_{i=1}^d \frac{1}{\hat{\sigma}_i(y_{<i}; x)}.
\]

\subsection{Flow matching (FM)}
\label{sec:flow_matching}

Flow matching \citep{Lipman2022-tm} is a recent generative modeling paradigm which has rapidly been gaining popularity. The key motivation behind flow matching is to combine the strengths of normalizing flows (NFs) and diffusion models while alleviating their main limitations. NFs enable exact likelihood estimation and efficient sampling but often suffer from limited expressiveness due to architectural constraints. Diffusion models, on the other hand, offer remarkable expressiveness and stability but typically require slow iterative sampling and do not provide tractable likelihoods. Flow matching addresses these issues by framing generative modeling as the learning of a continuous-time flow that transports noise to data, enabling both efficient training and fast sampling while retaining theoretical connections to likelihood-based methods.

Given $x \in \X$, we model the conditional predictive PDF $\hat{f}$ using a transformation defined by the ordinary differential equation (ODE)
$\frac{d\tilde{y}}{dt} = \hat{v}\left(t, \tilde{y}, x\right)$,
with a NN-parameterized vector field
$\hat{v} : [0,1] \times \Y \times \X \to \Y$.
Training uses the straight-line interpolant between latent $z \sim \mathcal{N}\left(0,I\right)$ and data $y \sim \hat{f}(\cdot \mid x)$,
\[
	\tilde{y}(t) = \left(1-t\right) z + t y,
\]
whose target velocity is constant with respect to $t$:
\[
	\frac{d}{dt}\tilde{y}(t) = y - z .
\]
The Conditional Flow Matching objective is
\begin{equation}
	\min_{\hat{v}}\ \mathbb{E}
	\left\| \hat{v}\left(t, \left(1-t\right) Z + t Y, X\right) - \left(Y - Z\right) \right\|^2
\end{equation}
where the expectation is over $t \sim \mathcal{U}\left(0,1\right)$, $(X, Y) \sim P_{X,Y}$, and $Z \sim \mathcal{N}\left(0,I\right)$.
After training, forward numerical integration generates samples from $\hat{f}_{Y|X=x}$ \citep{Chen2018-fi}:
\begin{equation}
	\hat{T}\left(z; x\right) = \tilde{y}\left(1\right)
	= z + \int_{0}^{1} \hat{v}\left(t, \tilde{y}\left(t\right), x\right) dt ,
	\qquad \tilde{y}\left(0\right)=z.
\end{equation}
Reverse-time integration encodes $y$ into its latent $z$:
\begin{equation}
	\hat{T}^{-1}\left(y; x\right) = \tilde{y}\left(0\right)
	= y + \int_{1}^{0} \hat{v}\left(t, \tilde{y}\left(t\right), x\right) dt ,
	\qquad \tilde{y}\left(1\right)=y .
\end{equation}
For $y \in \Y$, set $z = \hat{T}^{-1}\left(y; x\right)$. The log-likelihood follows from the instantaneous change-of-variables formula along the unique ODE path $\tilde{y}\left(t\right)$ with $\tilde{y}\left(0\right)=z$ and $\tilde{y}\left(1\right)=y$:
\begin{equation}
	\log \hat{f}_{Y|X=x}(y) = \log f_Z\left(z\right) - \int_{0}^{1} \text{Tr}\left(\nabla_{\tilde{y}} \hat{v}\left(t, \tilde{y}\left(t\right), x\right)\right) dt .
\end{equation}
The trace of the Jacobian can be computed using $d$ backpropagations, which can be prohibitive if $d$ is large.
It can also be efficiently approximated using Hutchinson's estimator:
\begin{equation}
	\text{Tr}\left(\nabla_{\tilde{y}} \hat{v}\left(t, \tilde{y}\left(t\right), x\right)\right)
	= \mathbb{E}_{\epsilon \sim \mathcal{N}\left(0, I\right)}\left[\epsilon^\top \left(\nabla_{\tilde{y}} \hat{v}\left(t, \tilde{y}\left(t\right), x\right)\right) \epsilon\right]
	\approx \frac{1}{K} \sum_{k=1}^{K} \epsilon_k^\top \left(\nabla_{\tilde{y}} \hat{v}\left(t, \tilde{y}\left(t\right), x\right)\right) \epsilon_k,
\end{equation}
with independent $\epsilon_k \sim \mathcal{N}\left(0, I\right)$, which enables practical likelihood computation.

\section{Additional results with convex potential flows}
\label{sec:results_cpf}

\subsection{Investigating the NLL performance gain}
\label{sec:NLL_performance_gain}

To investigate the source of the NLL improvement, we consider the decomposition of the NLL of the recalibrated model:
\begin{equation}
	\label{nll_decomposition}
	-\log \hat{f}'_{Y|X=x}(y) = -\log f_{Z}\left( z \right) - \log \left| \det\left( \nabla_z R(z) \right) \right|^{-1} -\log \left| \det\left( \nabla_y \hat{T}^{-1}(y; x) \right) \right|
\end{equation}
with $z' = \hat{T}^{-1}(y; x)$ and $z = R^{-1}(z')$. The third term is identical for both \BASE and \LR.
We analyze the contributions of the first two terms across the largest tabular datasets to avoid the table being too large.
All reported terms are averaged over the test set and over 10 runs.

\begin{table}[H]
	\footnotesize
	\centering
	\caption{Analysis of the NLL of \LR compared to \BASE with the terms described in \cref{nll_decomposition}.}
    \label{table:lr/MQF2_nll_decomposition}
	\vspace{0.2cm}
	\input{tables/lr/MQF2/nll_decomposition}
\end{table}

The table reveals a clear pattern. The NLL improvement from \LR is primarily driven by the first term. By radially transforming the latent codes $z'$ to new points $z$ that are more consistent with the base density $f_Z$, the latent density term $-\log f_Z(z)$ is significantly reduced. The recalibration Jacobian (second term) typically adds a small penalty (increases NLL), but this is almost always outweighed by the large gains from the first term. This confirms that \LR works by finding more \enquote{plausible} latent codes for the observed data under the base latent distribution.

\subsection{Computational efficiency}
\label{sec:computational_efficiency}

The difference in computation time can be measured in two aspects:
\begin{itemize}
	\item For calibration, the computational complexity of \HDRR is $O(M F n)$ and \LR is $O(R n)$ where $M = 100$ corresponds to the number of samples of \HDRR per instance, $F$ the time for the forward mapping $\hat{T}$ and $R$ the time for the reverse mapping $\hat{T}^{-1}$.
	\item For inference, it is a bit more subtle. Given a test insance $x$, \HDRR requires to sample at least $M$ times ($O(M F)$) to obtain a recalibrated sample, which can be a weakness, e.g., if only one conditional sample is needed. \LR only incurs a low fixed cost $C$ for evaluating the recalibration map ($O(C + F)$). Thus, the inference time is not directly comparable.
\end{itemize}

We report the calibration time of \HDRR and \LR in seconds on the largest datasets using the convex potential flow model and averaged over 10 runs.

\begin{table}[H]
\centering
\caption{Calibration times (part 1)}
\begin{tabular}{lrrrrrrrr}
\toprule
Method & HO2 & SC2 & RF1 & SC1 & AIR & BI2 & BI1 & WAG \\
\midrule
\HDRR & 1.56 & 2.78 & 4.90 & 2.57 & 4.86 & 18.80 & 8.80 & 15.00 \\
\LR    & 0.232 & 0.185 & 0.156 & 0.187 & 0.149 & 0.142 & 0.133 & 0.133 \\
\bottomrule
\end{tabular}
\end{table}

\begin{table}[H]
\centering
\caption{Calibration times (part 2)}
\begin{tabular}{lrrrrrrrr}
\toprule
Method & ME3 & ME1 & ME2 & HO1 & BIO & BLO & CAL & TAX \\
\midrule
\HDRR & 9.98 & 12.70 & 19.30 & 7.19 & 37.70 & 168.00 & 8.53 & 10.50 \\
\LR    & 0.152 & 0.153 & 0.163 & 0.192 & 0.290 & 0.328 & 0.482 & 0.324 \\
\bottomrule
\end{tabular}
\end{table}

On CIFAR-10 with TarFlow, the time difference is larger and can be prohibitive for \HDRR:

\begin{table}[H]
\centering
\begin{tabular}{ll}
	\toprule
	Method & CIFAR-10 \\
	\midrule
	\HDRR & 183182 \\
	\LR & 1259 \\
	\bottomrule
\end{tabular}
\end{table}

\subsection{Discriminative ability of the energy score and NLL}
\label{sec:discriminative_ability}

For a comprehensive evaluation of \LR, we report the ES in addition to the NLL. While \LR often leads to improved NLL, the ES remains largely unchanged. We hypothesize that this stems from the score's fundamental limitations in discriminative ability. 

\paragraph{Theoretical considerations.}
As established in \cite{Pinson2013-fa} and corroborated by \cite{Alexander2022-es}, the ES is sensitive to shifts in the mean but notoriously insensitive to misspecifications in variance, correlation, and overall dependency structure. \LR is a post-hoc procedure that primarily corrects the shape and spread of the predictive distribution. Therefore, the ES is fundamentally ill-suited to capture the specific improvements \LR provides.

In contrast, the NLL is uniquely suited for this evaluation. As the only local strictly proper scoring rule, its value depends only on the probability density at the precise location of the observed outcome \citep{Du2021-zg}. This locality makes it highly discerning of the very improvements \LR makes to the distributional shape, which is why we observe significant and consistent NLL reductions.

\paragraph{Empirical illustration.}
To provide a clear, empirical illustration, we designed a controlled synthetic experiment based on the dataset in \cref{fig:visualization}. The goal here is to isolate this specific property of the scoring rules in a setting free from the confounding variables of complex, real-world data.

We use an oracle predictor that knows the true data-generating distribution from \cref{fig:visualization} for everything except the spread around the arc, which is controlled by a standard deviation parameter $\sigma$. In \cref{table:sensitivity_10}, we then evaluate the predictor's NLL and ES (based on 100 samples) as we vary its estimate of $\sigma$. The metrics are averaged over 10 runs, and the true value is $\sigma=0.05$.

\begin{table}[htbp]
\centering
\caption{Metrics averaged over 10 runs, with standard error}
\label{table:sensitivity_10}
\vspace{0.1cm}
\begin{tabular}{lll}
\toprule
\multicolumn{1}{c}{$\sigma$} & \multicolumn{1}{c}{NLL} & \multicolumn{1}{c}{ES} \\
\midrule
0.01        & $12.01_{0.195}$            & $0.8733_{0.00298}$ \\
0.03        & $1.274_{0.0222}$           & $0.8724_{0.00267}$ \\
0.04        & $0.9232_{0.0129}$          & $\mathbf{0.8723_{0.00295}}$ \\
0.05 (True) & $\mathbf{0.8557_{0.00870}}$ & $0.8740_{0.00171}$ \\
0.06        & $0.8781_{0.00629}$         & $0.8741_{0.00182}$ \\
0.07        & $0.9365_{0.00483}$         & $0.8753_{0.00214}$ \\
0.10        & $1.161_{0.00272}$          & $0.8741_{0.00163}$ \\
0.20        & $1.774_{0.00106}$          & $0.8782_{0.00176}$ \\
\bottomrule
\end{tabular}
\end{table}

This experiment clearly illustrates the issue:
\begin{itemize}
	\item The NLL shows a sharp, clear minimum at the true value of $\sigma=0.05$, correctly identifying the best model.
	\item The ES remains almost completely flat for a wide range of $\sigma$ values (from 0.01 to 0.10). It fails to reliably distinguish a model with the correct variance from one that is substantially over- or under-confident.
\end{itemize}

This insensitivity is so profound that detecting a statistically significant signal with the ES requires an impractically large number of samples. The table below shows that only with 5000 runs does the ES minimum align with the true $\sigma$, and even then the differences are minuscule:

\begin{table}[htbp]
\centering
\caption{Metrics averaged over 5000 runs, with standard error}
\label{table:sensitivity_5000}
\begin{tabular}{lll}
\toprule
\multicolumn{1}{c}{$\sigma$} & \multicolumn{1}{c}{NLL} & \multicolumn{1}{c}{ES} \\
\midrule
0.01        & $11.91_{0.00773}$              & $0.8754_{0.000139}$ \\
0.03        & $1.263_{0.000868}$             & $0.8751_{0.000139}$ \\
0.04        & $0.9177_{0.000495}$            & $0.8751_{0.000139}$ \\
0.05 (True) & $\mathbf{0.8487_{0.000323}}$   & $\mathbf{0.8750_{0.000139}}$ \\
0.06        & $0.8732_{0.000231}$            & $0.8751_{0.000138}$ \\
0.07        & $0.9329_{0.000176}$            & $0.8751_{0.000138}$ \\
0.10        & $1.159_{0.000104}$             & $0.8756_{0.000137}$ \\
0.20        & $1.774_{6.91\text{e-}05}$      & $0.8797_{0.000133}$ \\
\bottomrule
\end{tabular}
\end{table}

This controlled experiment, therefore, proposes an explanation for the insentivity of the ES to \LR.

\subsection{Additional tables}

For reference, \cref{table:lr/MQF2_scoring_rules_lr,table:lr/MQF2_calibration_lr} provide the precise mean values and standard errors for NLL, Energy Score, L-ECE, and HDR-ECE across all tabular datasets when using convex potential flows as the base predictor. For each metric and dataset, all values that are statistically indistinguishable to the best value according to a Z-test at significance level 0.1 are highlighted in bold.

\begin{table}[H]
	\scriptsize
	\centering
	\caption{Full comparative table, using a convex potential flow model.}
    \label{table:lr/MQF2_scoring_rules_lr}
	\vspace{0.2cm}
	\input{tables/lr/MQF2/scoring_rules.tex}
\end{table}

\begin{table}[H]
	\scriptsize
	\centering
	\caption{Full comparative table, using a convex potential flow model.}
    \label{table:lr/MQF2_calibration_lr}
	\vspace{0.2cm}
	\input{tables/lr/MQF2/calibration.tex}
\end{table}

\section{Results with MAFs}
\label{sec:results_autoregressive_flows}

For completeness, this section reports results using a MAF \citep{Papamakarios2017-uh} as the base predictor. The architecture consists of stacked flow layers, where each layer’s conditioner is a masked autoencoder \citep{Germain2015-gs} parameterizing rational quadratic spline transformations \citep{Durkan2019-sc}. We tune hyperparameters using grid search. The number of stacked flows is chosen from $[3, 5, 8]$, the number of hidden units per flow from $[32, 64]$, and the number of hidden layers per flow from $[2, 3]$. The learning rate is selected from $[5\times10^{-3}, 10^{-3}]$. Each flow learns a rational quadratic spline transformation.

The findings, illustrated in \cref{fig:barplot/lr/ARFlow/latent_based} and \cref{fig:barplot/lr/ARFlow/sample_based} (and detailed in \cref{table:lr/ARFlow_scoring_rules_lr,table:lr/ARFlow_calibration_lr}), are consistent with the main tabular results reported in \cref{sec:experiments}. Specifically, \LR provides notable improvements in L-ECE, NLL, and HDR-ECE, while achieving an energy score comparable to that of the \BASE model.

\begin{figure}[H]
	\centering
	\includegraphics[width=0.9\linewidth]{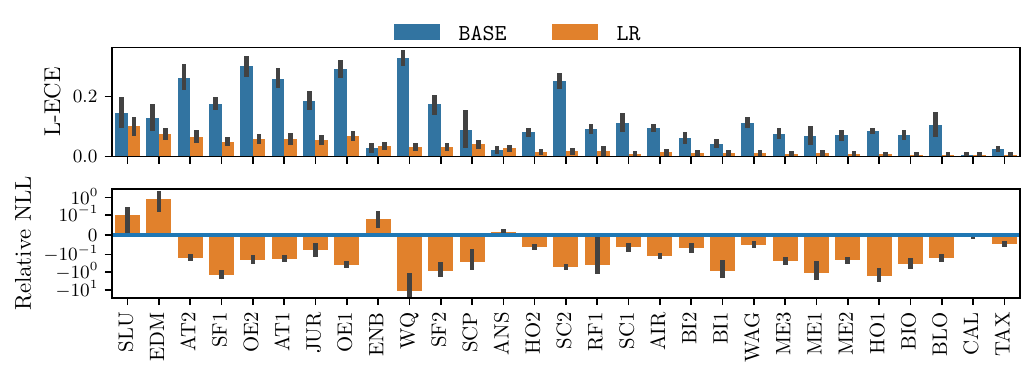}
	\vspace{-0.2cm}
	\caption{Latent calibration and NLL on datasets sorted by size, using a MAF model.}
	\label{fig:barplot/lr/ARFlow/latent_based}
\end{figure}

\begin{figure}[H]
	\centering
	\includegraphics[width=0.9\linewidth]{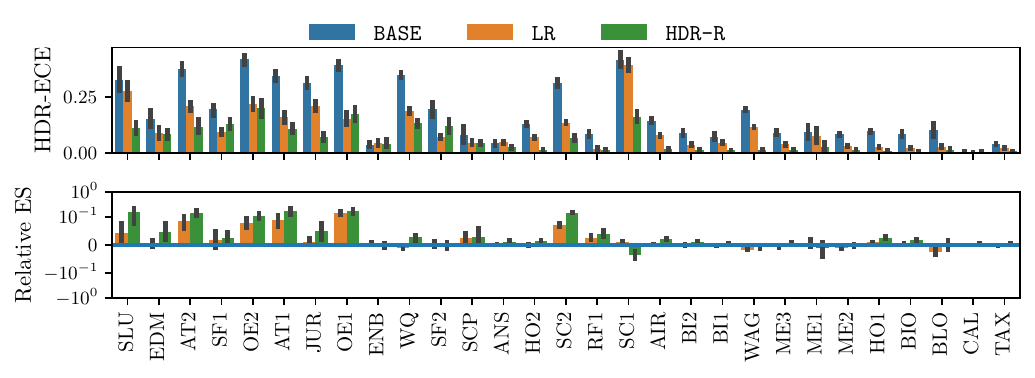}
	\vspace{-0.2cm}
	\caption{Latent calibration and NLL on datasets sorted by size, using a MAF model.}
	\label{fig:barplot/lr/ARFlow/sample_based}
\end{figure}

\begin{table}[H]
	\scriptsize
	\centering
	\caption{Full comparative table, using a MAF model.}
	\label{table:lr/ARFlow_scoring_rules_lr}
	\vspace{0.2cm}
	\input{tables/lr/ARFlow/scoring_rules.tex}
\end{table}

\begin{table}[H]
	\scriptsize
	\centering
	\caption{Full comparative table, using a MAF model.}
	\label{table:lr/ARFlow_calibration_lr}
	\vspace{0.2cm}
	\input{tables/lr/ARFlow/calibration.tex}
\end{table}

\section{Results with Flow Matching}
\label{sec:results_cfm}

While our paper focuses on normalizing flows, \LR is fully compatible with flow matching (FM) models (\cref{sec:flow_matching}), which also learn invertible mappings and assume a known latent distribution. For these models, we tune hyperparameters using grid search. The number of hidden units is chosen from $[32, 64]$, the number of hidden layers from $[2, 3, 5]$, and the learning rate from $[5\times10^{-3}, 10^{-3}, 2\times10^{-4}]$.

The FM results are aligned with the NFs results, with \LR standing out particularly on the L-ECE and NLL metrics.

\begin{figure}[H]
	\centering
	\includegraphics[width=0.9\linewidth]{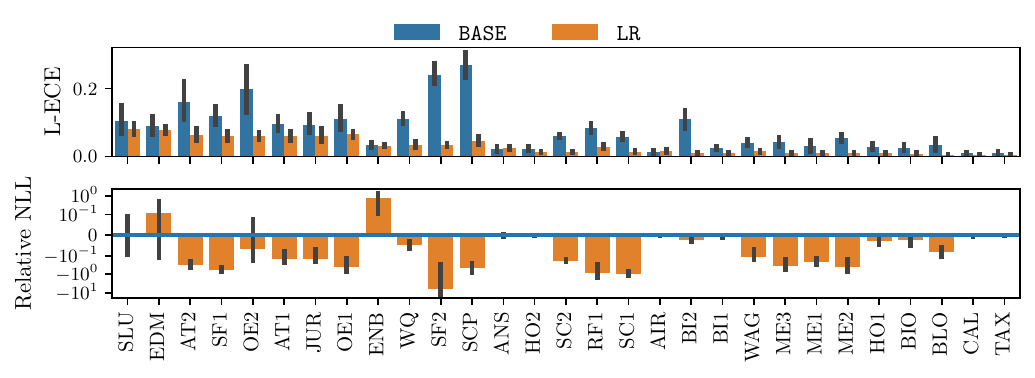}
	\vspace{-0.2cm}
	\caption{Latent calibration and NLL on datasets sorted by size, using a FM model.}
	\label{fig:barplot/lr/CFM/latent_based}
\end{figure}

\begin{figure}[H]
	\centering
	\includegraphics[width=0.9\linewidth]{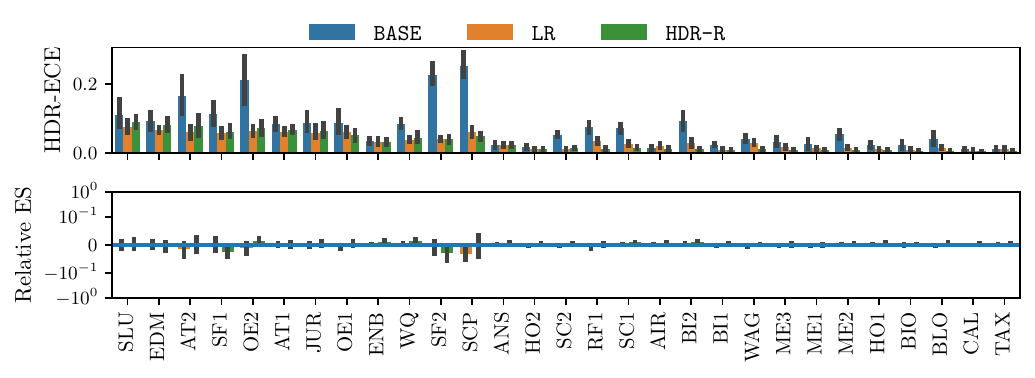}
	\vspace{-0.2cm}
	\caption{Latent calibration and NLL on datasets sorted by size, using a FM model.}
	\label{fig:barplot/lr/CFM/sample_based}
\end{figure}

\begin{table}[H]
	\scriptsize
	\centering
	\caption{Full comparative table, using a FM model.}
	\label{table:lr/CFM/scoring_rules}
	\vspace{0.2cm}
	\input{tables/lr/CFM/scoring_rules.tex}
\end{table}

\begin{table}[H]
	\scriptsize
	\centering
	\caption{Full comparative table, using a FM model.}
	\label{table:lr/CFM/calibration}
	\vspace{0.2cm}
	\input{tables/lr/CFM/calibration.tex}
\end{table}

\section{Results with a misspecified convex potential flow}
\label{sec:results_misspecified}

\cref{table:lr_misspecified/MQF2_scoring_rules_lr,table:lr_misspecified/MQF2_calibration_lr} along with \cref{fig:barplot/lr_misspecified/MQF2/latent_based} and \cref{fig:barplot/lr_misspecified/MQF2/sample_based} present results for a deliberately misspecified convex potential flow. This misspecification was induced by training the base predictor for only two epochs, ensuring it has low predictive accuracy and is likely poorly calibrated.

We observe that, in this additional scenario, \LR also leads to improved L-ECE and NLL on most datasets, indicating enhanced predictive accuracy compared to the \BASE misspecified model.
\LR achieves similar or improved HDR-ECE and ES compared to \HDRR. These results highlight \LR's ability to improve misspecified base predictors.

\begin{figure}[H]
	\centering
	\includegraphics[width=0.9\linewidth]{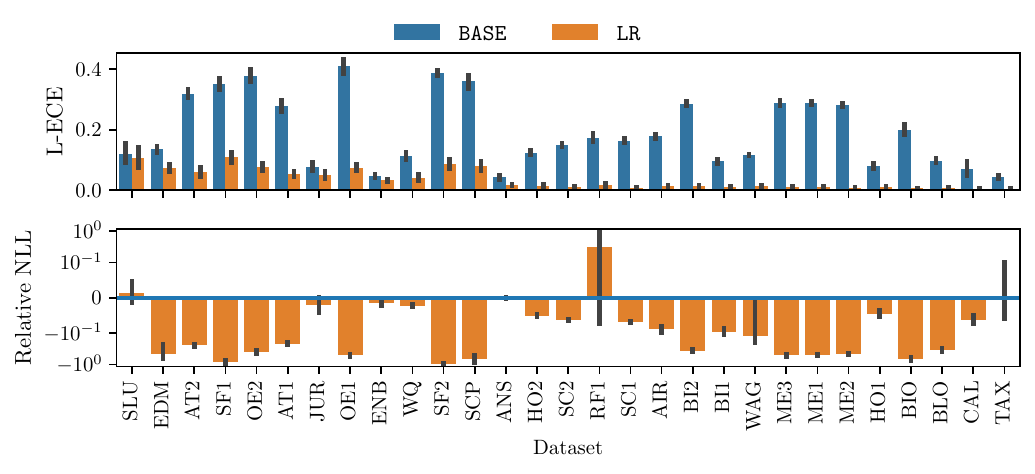}
	\vspace{-0.2cm}
	\caption{Latent calibration and NLL on datasets sorted by size, using a misspecified convex potential flow model.}
	\label{fig:barplot/lr_misspecified/MQF2/latent_based}
\end{figure}

\begin{figure}[H]
	\centering
	\includegraphics[width=0.9\linewidth]{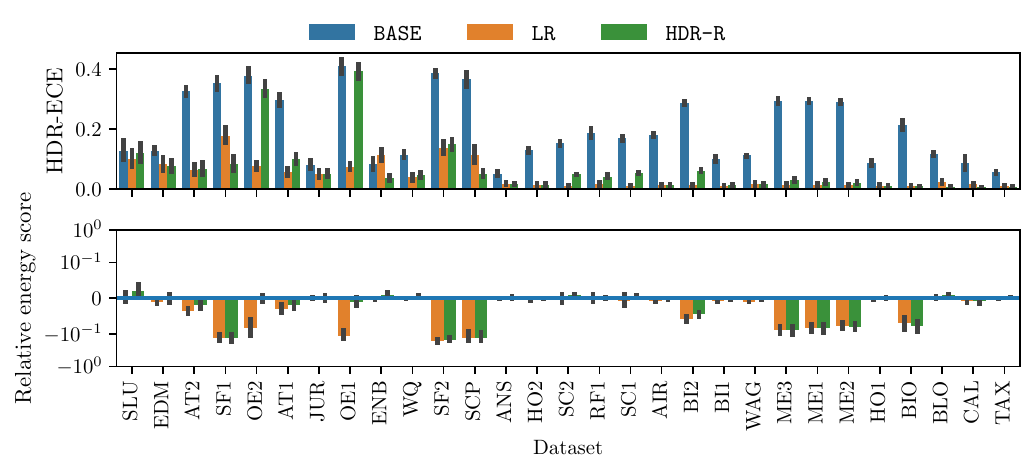}
	\vspace{-0.2cm}
	\caption{Latent calibration and NLL on datasets sorted by size, using a misspecified convex potential flow model.}
	\label{fig:barplot/lr_misspecified/MQF2/sample_based}
\end{figure}

\begin{table}[H]
	\scriptsize
	\centering
	\caption{Full comparative table, using a misspecified convex potential flow model.}
    \label{table:lr_misspecified/MQF2_scoring_rules_lr}
	\vspace{0.2cm}
	\input{tables/lr_misspecified/MQF2/scoring_rules.tex}
\end{table}

\begin{table}[H]
	\scriptsize
	\centering
	\caption{Full comparative table, using a misspecified convex potential flow model.}
    \label{table:lr_misspecified/MQF2_calibration_lr}
	\vspace{0.2cm}
	\input{tables/lr_misspecified/MQF2/calibration.tex}
\end{table}

%% file: tables/lr/MQF2/nll_decomposition.tex
\begin{tabular}{llllll}
\toprule
 & \BASE & \LR & \LR & \BASE & \LR \\
 & $-\log f_Z(z)$ & $-\log f_Z(z')$ & $- \log \left| \det\left( \nabla_z R(z) \right) \right|^{-1}$ & $-\log \hat{f}_{Y|X=x}(y)$ & $-\log \hat{f}'_{Y|X=x}(y)$ \\
\midrule
SLU & $\text{5.00}$ & $\text{4.16}$ & $\text{1.04}$ & $\text{3.09}$ & $\text{3.63}$ \\
EDM & $\text{3.91}$ & $\text{2.72}$ & $\text{1.26}$ & $\text{0.788}$ & $\text{0.855}$ \\
AT2 & $\text{11.7}$ & $\text{8.47}$ & $\text{1.67}$ & $\text{5.42}$ & $\text{4.31}$ \\
SF1 & $\text{8.11}$ & $\text{4.29}$ & $\text{0.900}$ & $\text{1.77}$ & $\text{-1.58}$ \\
OE2 & $\text{38.3}$ & $\text{22.8}$ & $\text{3.70}$ & $\text{23.2}$ & $\text{12.2}$ \\
AT1 & $\text{10.6}$ & $\text{8.44}$ & $\text{1.47}$ & $\text{2.58}$ & $\text{2.22}$ \\
JUR & $\text{5.39}$ & $\text{4.29}$ & $\text{0.817}$ & $\text{2.87}$ & $\text{2.80}$ \\
OE1 & $\text{52.7}$ & $\text{22.8}$ & $\text{2.22}$ & $\text{33.6}$ & $\text{6.13}$ \\
ENB & $\text{2.90}$ & $\text{2.78}$ & $\text{0.119}$ & $\text{-0.626}$ & $\text{-0.674}$ \\
WQ & $\text{41.1}$ & $\text{19.8}$ & $\text{3.19}$ & $\text{25.6}$ & $\text{4.94}$ \\
SF2 & $\text{8.02}$ & $\text{4.30}$ & $\text{-0.703}$ & $\text{-2.43}$ & $\text{-7.39}$ \\
SCP & $\text{14.3}$ & $\text{4.21}$ & $\text{-0.548}$ & $\text{4.29}$ & $\text{-6.75}$ \\
ANS & $\text{2.87}$ & $\text{2.78}$ & $\text{0.102}$ & $\text{1.81}$ & $\text{1.85}$ \\
HO2 & $\text{6.02}$ & $\text{5.68}$ & $\text{0.356}$ & $\text{2.66}$ & $\text{2.67}$ \\
SC2 & $\text{25.2}$ & $\text{22.7}$ & $\text{2.26}$ & $\text{1.63}$ & $\text{0.970}$ \\
RF1 & $\text{3.33e+02}$ & $\text{11.3}$ & $\text{0.345}$ & $\text{3.13e+02}$ & $\text{-10.0}$ \\
SC1 & $\text{23.9}$ & $\text{22.7}$ & $\text{0.333}$ & $\text{-1.94}$ & $\text{-3.14}$ \\
AIR & $\text{9.25}$ & $\text{8.50}$ & $\text{0.627}$ & $\text{3.84}$ & $\text{3.65}$ \\
BI2 & $\text{6.44}$ & $\text{5.71}$ & $\text{0.0988}$ & $\text{-9.09}$ & $\text{-10.5}$ \\
BI1 & $\text{3.11}$ & $\text{2.84}$ & $\text{0.190}$ & $\text{0.153}$ & $\text{0.145}$ \\
WAG & $\text{3.19}$ & $\text{2.81}$ & $\text{0.341}$ & $\text{-1.68}$ & $\text{-1.87}$ \\
ME3 & $\text{3.28}$ & $\text{2.85}$ & $\text{0.288}$ & $\text{-1.63}$ & $\text{-1.97}$ \\
ME1 & $\text{3.26}$ & $\text{2.83}$ & $\text{0.222}$ & $\text{-1.32}$ & $\text{-1.65}$ \\
ME2 & $\text{3.24}$ & $\text{2.83}$ & $\text{0.241}$ & $\text{-1.76}$ & $\text{-2.08}$ \\
HO1 & $\text{3.15}$ & $\text{2.83}$ & $\text{0.282}$ & $\text{-0.357}$ & $\text{-0.411}$ \\
BIO & $\text{3.14}$ & $\text{2.82}$ & $\text{0.209}$ & $\text{-0.928}$ & $\text{-1.16}$ \\
BLO & $\text{5.16}$ & $\text{2.84}$ & $\text{0.275}$ & $\text{-0.439}$ & $\text{-2.74}$ \\
CAL & $\text{2.85}$ & $\text{2.81}$ & $\text{0.0361}$ & $\text{0.600}$ & $\text{0.593}$ \\
TAX & $\text{2.93}$ & $\text{2.84}$ & $\text{0.0979}$ & $\text{1.60}$ & $\text{1.57}$ \\
\bottomrule
\end{tabular}

%% file: tables/lr/MQF2/scoring_rules.tex
\begin{tabular}{llllll}
\toprule
& \multicolumn{2}{c}{NLL} & \multicolumn{3}{c}{Energy score} \\
& \BASE & \LR & \BASE & \HDRR & \LR \\
\midrule
SLU & \bfseries $\text{2.61}_{\text{0.19}}$ & \bfseries $\text{2.29}_{\text{0.14}}$ & \bfseries $\text{0.791}_{\text{0.038}}$ & \bfseries $\text{0.795}_{\text{0.033}}$ & \bfseries $\text{0.785}_{\text{0.033}}$ \\
EDM & \bfseries $\text{-0.0350}_{\text{0.46}}$ & \bfseries $\text{-0.123}_{\text{1.3}}$ & \bfseries $\text{0.647}_{\text{0.049}}$ & \bfseries $\text{0.648}_{\text{0.050}}$ & \bfseries $\text{0.635}_{\text{0.044}}$ \\
AT2 & $\text{4.05}_{\text{0.88}}$ & \bfseries $\text{1.86}_{\text{0.42}}$ & \bfseries $\text{0.861}_{\text{0.044}}$ & \bfseries $\text{0.870}_{\text{0.044}}$ & \bfseries $\text{0.862}_{\text{0.042}}$ \\
SF1 & \bfseries $\text{4.41}_{\text{3.2}}$ & \bfseries $\text{-0.381}_{\text{3.6}}$ & \bfseries $\text{0.673}_{\text{0.086}}$ & \bfseries $\text{0.639}_{\text{0.093}}$ & \bfseries $\text{0.670}_{\text{0.085}}$ \\
OE2 & \bfseries $\text{37.6}_{\text{3.0e+01}}$ & \bfseries $\text{6.46}_{\text{1.3}}$ & \bfseries $\text{1.25}_{\text{0.083}}$ & \bfseries $\text{1.26}_{\text{0.083}}$ & \bfseries $\text{1.26}_{\text{0.085}}$ \\
AT1 & $\text{1.63}_{\text{0.46}}$ & \bfseries $\text{0.0783}_{\text{0.29}}$ & \bfseries $\text{0.582}_{\text{0.032}}$ & \bfseries $\text{0.587}_{\text{0.031}}$ & \bfseries $\text{0.591}_{\text{0.030}}$ \\
JUR & \bfseries $\text{2.16}_{\text{0.24}}$ & \bfseries $\text{1.87}_{\text{0.15}}$ & \bfseries $\text{0.617}_{\text{0.034}}$ & \bfseries $\text{0.618}_{\text{0.033}}$ & \bfseries $\text{0.618}_{\text{0.034}}$ \\
OE1 & \bfseries $\text{80.0}_{\text{6.9e+01}}$ & \bfseries $\text{2.93}_{\text{0.71}}$ & \bfseries $\text{1.23}_{\text{0.15}}$ & \bfseries $\text{1.23}_{\text{0.15}}$ & \bfseries $\text{1.17}_{\text{0.15}}$ \\
ENB & \bfseries $\text{-1.08}_{\text{0.11}}$ & \bfseries $\text{-1.12}_{\text{0.10}}$ & \bfseries $\text{0.249}_{\text{0.010}}$ & \bfseries $\text{0.250}_{\text{0.010}}$ & \bfseries $\text{0.249}_{\text{0.010}}$ \\
WQ & \bfseries $\text{60.8}_{\text{3.7e+01}}$ & \bfseries $\text{7.95}_{\text{3.6}}$ & \bfseries $\text{2.47}_{\text{0.025}}$ & \bfseries $\text{2.49}_{\text{0.024}}$ & \bfseries $\text{2.47}_{\text{0.024}}$ \\
SF2 & $\text{-3.37}_{\text{3.0}}$ & \bfseries $\text{-11.1}_{\text{0.63}}$ & \bfseries $\text{0.587}_{\text{0.044}}$ & \bfseries $\text{0.593}_{\text{0.046}}$ & \bfseries $\text{0.598}_{\text{0.045}}$ \\
SCP & \bfseries $\text{20.1}_{\text{2.6e+01}}$ & \bfseries $\text{-8.55}_{\text{0.49}}$ & \bfseries $\text{0.389}_{\text{0.094}}$ & \bfseries $\text{0.383}_{\text{0.095}}$ & \bfseries $\text{0.382}_{\text{0.095}}$ \\
ANS & \bfseries $\text{1.76}_{\text{0.022}}$ & \bfseries $\text{1.79}_{\text{0.020}}$ & \bfseries $\text{0.529}_{\text{0.0052}}$ & \bfseries $\text{0.531}_{\text{0.0047}}$ & \bfseries $\text{0.529}_{\text{0.0053}}$ \\
HO2 & \bfseries $\text{2.39}_{\text{0.034}}$ & \bfseries $\text{2.38}_{\text{0.035}}$ & \bfseries $\text{0.862}_{\text{0.0076}}$ & \bfseries $\text{0.866}_{\text{0.0075}}$ & \bfseries $\text{0.862}_{\text{0.0076}}$ \\
SC2 & $\text{0.795}_{\text{0.17}}$ & \bfseries $\text{0.189}_{\text{0.15}}$ & \bfseries $\text{1.25}_{\text{0.011}}$ & \bfseries $\text{1.28}_{\text{0.012}}$ & \bfseries $\text{1.26}_{\text{0.011}}$ \\
RF1 & \bfseries $\text{9.54e+02}_{\text{6.8e+02}}$ & \bfseries $\text{-4.50}_{\text{1.5}}$ & \bfseries $\text{0.534}_{\text{0.073}}$ & \bfseries $\text{0.528}_{\text{0.071}}$ & \bfseries $\text{0.529}_{\text{0.072}}$ \\
SC1 & $\text{-1.86}_{\text{0.080}}$ & \bfseries $\text{-2.63}_{\text{0.075}}$ & \bfseries $\text{0.824}_{\text{0.0047}}$ & \bfseries $\text{0.833}_{\text{0.0045}}$ & \bfseries $\text{0.825}_{\text{0.0045}}$ \\
AIR & \bfseries $\text{3.03}_{\text{0.30}}$ & \bfseries $\text{2.91}_{\text{0.30}}$ & \bfseries $\text{1.17}_{\text{0.0086}}$ & \bfseries $\text{1.19}_{\text{0.0090}}$ & \bfseries $\text{1.18}_{\text{0.0084}}$ \\
BI2 & \bfseries $\text{-11.5}_{\text{0.71}}$ & \bfseries $\text{-13.0}_{\text{0.61}}$ & \bfseries $\text{0.833}_{\text{0.013}}$ & \bfseries $\text{0.848}_{\text{0.015}}$ & \bfseries $\text{0.834}_{\text{0.014}}$ \\
BI1 & \bfseries $\text{-2.27}_{\text{0.26}}$ & \bfseries $\text{-2.40}_{\text{0.26}}$ & \bfseries $\text{0.708}_{\text{0.0056}}$ & \bfseries $\text{0.711}_{\text{0.0053}}$ & \bfseries $\text{0.708}_{\text{0.0057}}$ \\
WAG & \bfseries $\text{-3.25}_{\text{0.31}}$ & \bfseries $\text{-3.34}_{\text{0.32}}$ & \bfseries $\text{0.802}_{\text{0.048}}$ & \bfseries $\text{0.803}_{\text{0.043}}$ & \bfseries $\text{0.805}_{\text{0.046}}$ \\
ME3 & \bfseries $\text{-2.60}_{\text{0.13}}$ & \bfseries $\text{-2.69}_{\text{0.12}}$ & \bfseries $\text{0.358}_{\text{0.0082}}$ & \bfseries $\text{0.362}_{\text{0.0082}}$ & \bfseries $\text{0.360}_{\text{0.0083}}$ \\
ME1 & \bfseries $\text{-2.00}_{\text{0.13}}$ & \bfseries $\text{-2.17}_{\text{0.13}}$ & \bfseries $\text{0.357}_{\text{0.0075}}$ & \bfseries $\text{0.370}_{\text{0.0095}}$ & \bfseries $\text{0.365}_{\text{0.0087}}$ \\
ME2 & \bfseries $\text{-3.00}_{\text{0.12}}$ & \bfseries $\text{-3.22}_{\text{0.13}}$ & \bfseries $\text{0.361}_{\text{0.0041}}$ & \bfseries $\text{0.362}_{\text{0.0039}}$ & \bfseries $\text{0.362}_{\text{0.0040}}$ \\
HO1 & \bfseries $\text{-0.299}_{\text{0.038}}$ & \bfseries $\text{-0.364}_{\text{0.029}}$ & \bfseries $\text{0.346}_{\text{0.0074}}$ & \bfseries $\text{0.351}_{\text{0.0079}}$ & \bfseries $\text{0.340}_{\text{0.012}}$ \\
BIO & $\text{-1.12}_{\text{0.072}}$ & \bfseries $\text{-1.25}_{\text{0.019}}$ & \bfseries $\text{0.207}_{\text{0.0038}}$ & \bfseries $\text{0.210}_{\text{0.0039}}$ & \bfseries $\text{0.204}_{\text{0.0051}}$ \\
BLO & $\text{3.11}_{\text{2.8}}$ & \bfseries $\text{-2.11}_{\text{0.25}}$ & \bfseries $\text{0.305}_{\text{0.029}}$ & \bfseries $\text{0.365}_{\text{0.037}}$ & $\text{0.473}_{\text{0.081}}$ \\
CAL & \bfseries $\text{0.575}_{\text{0.0097}}$ & \bfseries $\text{0.575}_{\text{0.0088}}$ & \bfseries $\text{0.419}_{\text{0.0014}}$ & \bfseries $\text{0.421}_{\text{0.0015}}$ & \bfseries $\text{0.419}_{\text{0.0014}}$ \\
TAX & \bfseries $\text{1.53}_{\text{0.0069}}$ & \bfseries $\text{1.53}_{\text{0.0068}}$ & \bfseries $\text{0.692}_{\text{0.0019}}$ & $\text{0.696}_{\text{0.0021}}$ & \bfseries $\text{0.690}_{\text{0.0027}}$ \\
\bottomrule
\end{tabular}

%% file: tables/lr/MQF2/calibration.tex
\begin{tabular}{llllll}
\toprule
& \multicolumn{2}{c}{L-ECE} & \multicolumn{3}{c}{HDR-ECE} \\
& \BASE & \LR & \BASE & \HDRR & \LR \\
\midrule
SLU & \bfseries $\text{0.146}_{\text{0.026}}$ & \bfseries $\text{0.106}_{\text{0.016}}$ & \bfseries $\text{0.129}_{\text{0.022}}$ & \bfseries $\text{0.116}_{\text{0.014}}$ & \bfseries $\text{0.102}_{\text{0.013}}$ \\
EDM & \bfseries $\text{0.122}_{\text{0.016}}$ & \bfseries $\text{0.0905}_{\text{0.016}}$ & \bfseries $\text{0.128}_{\text{0.014}}$ & \bfseries $\text{0.101}_{\text{0.019}}$ & $\text{0.169}_{\text{0.014}}$ \\
AT2 & $\text{0.129}_{\text{0.0076}}$ & \bfseries $\text{0.0637}_{\text{0.010}}$ & $\text{0.149}_{\text{0.0094}}$ & \bfseries $\text{0.0817}_{\text{0.0095}}$ & \bfseries $\text{0.0688}_{\text{0.012}}$ \\
SF1 & $\text{0.279}_{\text{0.026}}$ & \bfseries $\text{0.0701}_{\text{0.0082}}$ & $\text{0.321}_{\text{0.022}}$ & $\text{0.171}_{\text{0.019}}$ & \bfseries $\text{0.0786}_{\text{0.0096}}$ \\
OE2 & $\text{0.136}_{\text{0.011}}$ & \bfseries $\text{0.0785}_{\text{0.0076}}$ & $\text{0.129}_{\text{0.0098}}$ & \bfseries $\text{0.0768}_{\text{0.012}}$ & \bfseries $\text{0.0766}_{\text{0.0079}}$ \\
AT1 & $\text{0.124}_{\text{0.013}}$ & \bfseries $\text{0.0668}_{\text{0.0098}}$ & $\text{0.125}_{\text{0.012}}$ & $\text{0.0906}_{\text{0.0095}}$ & \bfseries $\text{0.0643}_{\text{0.012}}$ \\
JUR & $\text{0.0883}_{\text{0.011}}$ & \bfseries $\text{0.0520}_{\text{0.0052}}$ & $\text{0.0866}_{\text{0.011}}$ & \bfseries $\text{0.0515}_{\text{0.0063}}$ & \bfseries $\text{0.0537}_{\text{0.0049}}$ \\
OE1 & $\text{0.197}_{\text{0.044}}$ & \bfseries $\text{0.0589}_{\text{0.0068}}$ & $\text{0.195}_{\text{0.044}}$ & \bfseries $\text{0.144}_{\text{0.049}}$ & \bfseries $\text{0.0636}_{\text{0.0074}}$ \\
ENB & \bfseries $\text{0.0272}_{\text{0.0034}}$ & \bfseries $\text{0.0292}_{\text{0.0034}}$ & \bfseries $\text{0.0375}_{\text{0.0049}}$ & \bfseries $\text{0.0347}_{\text{0.0039}}$ & \bfseries $\text{0.0397}_{\text{0.0061}}$ \\
WQ & $\text{0.0883}_{\text{0.0048}}$ & \bfseries $\text{0.0479}_{\text{0.0044}}$ & $\text{0.0975}_{\text{0.0032}}$ & \bfseries $\text{0.0443}_{\text{0.0060}}$ & \bfseries $\text{0.0432}_{\text{0.0086}}$ \\
SF2 & $\text{0.272}_{\text{0.0084}}$ & \bfseries $\text{0.0434}_{\text{0.0050}}$ & $\text{0.312}_{\text{0.0061}}$ & $\text{0.164}_{\text{0.015}}$ & \bfseries $\text{0.0797}_{\text{0.010}}$ \\
SCP & $\text{0.218}_{\text{0.025}}$ & \bfseries $\text{0.0502}_{\text{0.0060}}$ & $\text{0.223}_{\text{0.025}}$ & \bfseries $\text{0.0543}_{\text{0.010}}$ & \bfseries $\text{0.0614}_{\text{0.0093}}$ \\
ANS & \bfseries $\text{0.0190}_{\text{0.0029}}$ & \bfseries $\text{0.0254}_{\text{0.0036}}$ & \bfseries $\text{0.0199}_{\text{0.0030}}$ & \bfseries $\text{0.0234}_{\text{0.0026}}$ & \bfseries $\text{0.0249}_{\text{0.0037}}$ \\
HO2 & \bfseries $\text{0.0162}_{\text{0.0024}}$ & \bfseries $\text{0.0158}_{\text{0.0018}}$ & $\text{0.0179}_{\text{0.0023}}$ & \bfseries $\text{0.0122}_{\text{0.00080}}$ & \bfseries $\text{0.0121}_{\text{0.0017}}$ \\
SC2 & $\text{0.101}_{\text{0.0034}}$ & \bfseries $\text{0.0122}_{\text{0.00087}}$ & $\text{0.105}_{\text{0.0030}}$ & $\text{0.0223}_{\text{0.00081}}$ & \bfseries $\text{0.0117}_{\text{0.0011}}$ \\
RF1 & $\text{0.112}_{\text{0.021}}$ & \bfseries $\text{0.0260}_{\text{0.0049}}$ & $\text{0.129}_{\text{0.024}}$ & \bfseries $\text{0.0193}_{\text{0.0034}}$ & $\text{0.0288}_{\text{0.0045}}$ \\
SC1 & $\text{0.0845}_{\text{0.0015}}$ & \bfseries $\text{0.0106}_{\text{0.0013}}$ & $\text{0.0857}_{\text{0.0017}}$ & $\text{0.0229}_{\text{0.0011}}$ & \bfseries $\text{0.0116}_{\text{0.0016}}$ \\
AIR & $\text{0.0527}_{\text{0.0059}}$ & \bfseries $\text{0.0162}_{\text{0.0010}}$ & $\text{0.0449}_{\text{0.0083}}$ & \bfseries $\text{0.0160}_{\text{0.0012}}$ & $\text{0.0286}_{\text{0.0052}}$ \\
BI2 & $\text{0.122}_{\text{0.016}}$ & \bfseries $\text{0.0170}_{\text{0.0022}}$ & $\text{0.167}_{\text{0.019}}$ & \bfseries $\text{0.0219}_{\text{0.0023}}$ & $\text{0.0731}_{\text{0.011}}$ \\
BI1 & $\text{0.0279}_{\text{0.0023}}$ & \bfseries $\text{0.0102}_{\text{0.00084}}$ & \bfseries $\text{0.111}_{\text{0.026}}$ & \bfseries $\text{0.0580}_{\text{0.024}}$ & \bfseries $\text{0.0923}_{\text{0.023}}$ \\
WAG & $\text{0.0478}_{\text{0.0082}}$ & \bfseries $\text{0.0137}_{\text{0.0018}}$ & \bfseries $\text{0.267}_{\text{0.050}}$ & \bfseries $\text{0.161}_{\text{0.071}}$ & \bfseries $\text{0.290}_{\text{0.046}}$ \\
ME3 & $\text{0.0732}_{\text{0.012}}$ & \bfseries $\text{0.00957}_{\text{0.00084}}$ & $\text{0.0520}_{\text{0.010}}$ & \bfseries $\text{0.00899}_{\text{0.00064}}$ & $\text{0.0323}_{\text{0.0042}}$ \\
ME1 & $\text{0.0853}_{\text{0.010}}$ & \bfseries $\text{0.0103}_{\text{0.00078}}$ & $\text{0.0779}_{\text{0.010}}$ & \bfseries $\text{0.00976}_{\text{0.00092}}$ & $\text{0.0432}_{\text{0.0050}}$ \\
ME2 & $\text{0.0497}_{\text{0.0099}}$ & \bfseries $\text{0.0101}_{\text{0.00066}}$ & $\text{0.0415}_{\text{0.0061}}$ & \bfseries $\text{0.0136}_{\text{0.0023}}$ & $\text{0.0417}_{\text{0.0061}}$ \\
HO1 & $\text{0.0166}_{\text{0.0023}}$ & \bfseries $\text{0.00950}_{\text{0.0013}}$ & \bfseries $\text{0.0103}_{\text{0.0014}}$ & \bfseries $\text{0.00782}_{\text{0.00071}}$ & \bfseries $\text{0.0109}_{\text{0.0018}}$ \\
BIO & $\text{0.0178}_{\text{0.0019}}$ & \bfseries $\text{0.00561}_{\text{0.00076}}$ & $\text{0.0220}_{\text{0.0027}}$ & \bfseries $\text{0.00667}_{\text{0.00051}}$ & \bfseries $\text{0.00727}_{\text{0.0015}}$ \\
BLO & $\text{0.231}_{\text{0.036}}$ & \bfseries $\text{0.00735}_{\text{0.0010}}$ & $\text{0.207}_{\text{0.040}}$ & \bfseries $\text{0.0485}_{\text{0.024}}$ & $\text{0.114}_{\text{0.014}}$ \\
CAL & $\text{0.00749}_{\text{0.00096}}$ & \bfseries $\text{0.00502}_{\text{0.00074}}$ & $\text{0.00760}_{\text{0.00079}}$ & \bfseries $\text{0.00555}_{\text{0.00026}}$ & \bfseries $\text{0.00543}_{\text{0.00064}}$ \\
TAX & $\text{0.00949}_{\text{0.0013}}$ & \bfseries $\text{0.00522}_{\text{0.00053}}$ & \bfseries $\text{0.00584}_{\text{0.00063}}$ & \bfseries $\text{0.00684}_{\text{0.00050}}$ & \bfseries $\text{0.00561}_{\text{0.00086}}$ \\
\bottomrule
\end{tabular}

%% file: tables/lr/ARFlow/scoring_rules.tex
\begin{tabular}{llllll}
\toprule
& \multicolumn{2}{c}{NLL} & \multicolumn{3}{c}{Energy score} \\
& \BASE & \LR & \BASE & \HDRR & \LR \\
\midrule
SLU & \bfseries $\text{4.50}_{\text{0.37}}$ & \bfseries $\text{5.08}_{\text{0.52}}$ & \bfseries $\text{0.831}_{\text{0.051}}$ & $\text{0.945}_{\text{0.034}}$ & \bfseries $\text{0.858}_{\text{0.025}}$ \\
EDM & \bfseries $\text{0.202}_{\text{0.30}}$ & \bfseries $\text{0.590}_{\text{0.35}}$ & \bfseries $\text{0.577}_{\text{0.035}}$ & \bfseries $\text{0.602}_{\text{0.033}}$ & \bfseries $\text{0.578}_{\text{0.022}}$ \\
AT2 & $\text{7.45}_{\text{0.52}}$ & \bfseries $\text{6.22}_{\text{0.24}}$ & \bfseries $\text{0.970}_{\text{0.054}}$ & $\text{1.11}_{\text{0.044}}$ & \bfseries $\text{1.05}_{\text{0.029}}$ \\
SF1 & $\text{-1.47}_{\text{0.39}}$ & \bfseries $\text{-3.42}_{\text{0.26}}$ & \bfseries $\text{0.698}_{\text{0.087}}$ & \bfseries $\text{0.710}_{\text{0.085}}$ & \bfseries $\text{0.700}_{\text{0.053}}$ \\
OE2 & $\text{22.4}_{\text{2.6}}$ & \bfseries $\text{16.8}_{\text{0.90}}$ & \bfseries $\text{1.67}_{\text{0.13}}$ & \bfseries $\text{1.85}_{\text{0.13}}$ & \bfseries $\text{1.79}_{\text{0.087}}$ \\
AT1 & \bfseries $\text{4.63}_{\text{0.61}}$ & \bfseries $\text{3.80}_{\text{0.32}}$ & \bfseries $\text{0.698}_{\text{0.049}}$ & $\text{0.815}_{\text{0.043}}$ & \bfseries $\text{0.754}_{\text{0.029}}$ \\
JUR & \bfseries $\text{3.99}_{\text{0.29}}$ & \bfseries $\text{3.65}_{\text{0.15}}$ & \bfseries $\text{0.680}_{\text{0.037}}$ & \bfseries $\text{0.710}_{\text{0.031}}$ & \bfseries $\text{0.686}_{\text{0.022}}$ \\
OE1 & $\text{16.9}_{\text{2.7}}$ & \bfseries $\text{9.68}_{\text{0.75}}$ & \bfseries $\text{1.31}_{\text{0.13}}$ & \bfseries $\text{1.51}_{\text{0.12}}$ & \bfseries $\text{1.48}_{\text{0.080}}$ \\
ENB & \bfseries $\text{-0.939}_{\text{0.10}}$ & \bfseries $\text{-0.877}_{\text{0.071}}$ & \bfseries $\text{0.273}_{\text{0.0080}}$ & \bfseries $\text{0.273}_{\text{0.0077}}$ & \bfseries $\text{0.275}_{\text{0.0054}}$ \\
WQ & $\text{0.944}_{\text{0.91}}$ & \bfseries $\text{-1.23}_{\text{0.57}}$ & \bfseries $\text{2.50}_{\text{0.035}}$ & $\text{2.56}_{\text{0.030}}$ & \bfseries $\text{2.47}_{\text{0.019}}$ \\
SF2 & \bfseries $\text{-6.21}_{\text{1.3}}$ & \bfseries $\text{-8.72}_{\text{0.77}}$ & \bfseries $\text{0.640}_{\text{0.051}}$ & \bfseries $\text{0.637}_{\text{0.048}}$ & \bfseries $\text{0.641}_{\text{0.034}}$ \\
SCP & \bfseries $\text{-5.17}_{\text{2.2}}$ & \bfseries $\text{-7.58}_{\text{0.51}}$ & \bfseries $\text{0.392}_{\text{0.099}}$ & \bfseries $\text{0.398}_{\text{0.099}}$ & \bfseries $\text{0.400}_{\text{0.069}}$ \\
ANS & \bfseries $\text{1.89}_{\text{0.025}}$ & \bfseries $\text{1.91}_{\text{0.017}}$ & \bfseries $\text{0.531}_{\text{0.0053}}$ & \bfseries $\text{0.536}_{\text{0.0049}}$ & \bfseries $\text{0.532}_{\text{0.0036}}$ \\
HO2 & $\text{3.06}_{\text{0.052}}$ & \bfseries $\text{2.87}_{\text{0.029}}$ & \bfseries $\text{0.881}_{\text{0.0077}}$ & \bfseries $\text{0.894}_{\text{0.0068}}$ & \bfseries $\text{0.881}_{\text{0.0050}}$ \\
SC2 & $\text{1.88}_{\text{0.25}}$ & \bfseries $\text{0.936}_{\text{0.12}}$ & \bfseries $\text{1.01}_{\text{0.0091}}$ & $\text{1.17}_{\text{0.0096}}$ & $\text{1.09}_{\text{0.0060}}$ \\
RF1 & \bfseries $\text{-14.3}_{\text{1.3}}$ & \bfseries $\text{-15.5}_{\text{0.28}}$ & \bfseries $\text{0.203}_{\text{0.023}}$ & \bfseries $\text{0.210}_{\text{0.023}}$ & \bfseries $\text{0.208}_{\text{0.016}}$ \\
SC1 & \bfseries $\text{-4.48}_{\text{2.9}}$ & \bfseries $\text{-5.00}_{\text{2.0}}$ & \bfseries $\text{2.62}_{\text{0.19}}$ & \bfseries $\text{2.52}_{\text{0.18}}$ & \bfseries $\text{2.65}_{\text{0.14}}$ \\
AIR & $\text{4.29}_{\text{0.15}}$ & \bfseries $\text{3.74}_{\text{0.10}}$ & \bfseries $\text{1.21}_{\text{0.0097}}$ & $\text{1.23}_{\text{0.0090}}$ & \bfseries $\text{1.21}_{\text{0.0065}}$ \\
BI2 & $\text{-11.6}_{\text{0.27}}$ & \bfseries $\text{-12.3}_{\text{0.16}}$ & \bfseries $\text{0.831}_{\text{0.018}}$ & \bfseries $\text{0.839}_{\text{0.017}}$ & \bfseries $\text{0.830}_{\text{0.011}}$ \\
BI1 & \bfseries $\text{0.622}_{\text{0.12}}$ & \bfseries $\text{0.443}_{\text{0.077}}$ & \bfseries $\text{0.719}_{\text{0.0049}}$ & \bfseries $\text{0.722}_{\text{0.0049}}$ & \bfseries $\text{0.718}_{\text{0.0034}}$ \\
WAG & \bfseries $\text{-2.12}_{\text{0.089}}$ & \bfseries $\text{-2.22}_{\text{0.060}}$ & $\text{0.721}_{\text{0.0045}}$ & \bfseries $\text{0.716}_{\text{0.0037}}$ & \bfseries $\text{0.710}_{\text{0.0027}}$ \\
ME3 & $\text{-1.95}_{\text{0.097}}$ & \bfseries $\text{-2.37}_{\text{0.049}}$ & \bfseries $\text{0.401}_{\text{0.0080}}$ & \bfseries $\text{0.403}_{\text{0.0078}}$ & \bfseries $\text{0.398}_{\text{0.0054}}$ \\
ME1 & \bfseries $\text{-1.51}_{\text{0.36}}$ & \bfseries $\text{-1.96}_{\text{0.23}}$ & \bfseries $\text{0.531}_{\text{0.066}}$ & \bfseries $\text{0.517}_{\text{0.052}}$ & \bfseries $\text{0.540}_{\text{0.053}}$ \\
ME2 & $\text{-1.94}_{\text{0.076}}$ & \bfseries $\text{-2.35}_{\text{0.045}}$ & \bfseries $\text{0.412}_{\text{0.0047}}$ & \bfseries $\text{0.412}_{\text{0.0045}}$ & \bfseries $\text{0.409}_{\text{0.0033}}$ \\
HO1 & $\text{-0.153}_{\text{0.045}}$ & \bfseries $\text{-0.323}_{\text{0.023}}$ & \bfseries $\text{0.327}_{\text{0.0068}}$ & \bfseries $\text{0.335}_{\text{0.0070}}$ & \bfseries $\text{0.330}_{\text{0.0048}}$ \\
BIO & $\text{-1.10}_{\text{0.075}}$ & \bfseries $\text{-1.42}_{\text{0.016}}$ & \bfseries $\text{0.203}_{\text{0.0040}}$ & \bfseries $\text{0.207}_{\text{0.0040}}$ & \bfseries $\text{0.204}_{\text{0.0027}}$ \\
BLO & \bfseries $\text{-3.37}_{\text{0.35}}$ & \bfseries $\text{-3.85}_{\text{0.24}}$ & \bfseries $\text{0.372}_{\text{0.015}}$ & \bfseries $\text{0.372}_{\text{0.018}}$ & \bfseries $\text{0.362}_{\text{0.0093}}$ \\
CAL & \bfseries $\text{0.581}_{\text{0.0083}}$ & \bfseries $\text{0.578}_{\text{0.0055}}$ & \bfseries $\text{0.419}_{\text{0.0014}}$ & \bfseries $\text{0.421}_{\text{0.0014}}$ & \bfseries $\text{0.419}_{\text{0.00094}}$ \\
TAX & $\text{1.62}_{\text{0.0069}}$ & \bfseries $\text{1.54}_{\text{0.0042}}$ & \bfseries $\text{0.696}_{\text{0.0023}}$ & \bfseries $\text{0.698}_{\text{0.0021}}$ & \bfseries $\text{0.695}_{\text{0.0015}}$ \\
\bottomrule
\end{tabular}

%% file: tables/lr/ARFlow/calibration.tex
\begin{tabular}{llllll}
\toprule
& \multicolumn{2}{c}{L-ECE} & \multicolumn{3}{c}{HDR-ECE} \\
& \BASE & \LR & \BASE & \HDRR & \LR \\
\midrule
SLU & $\text{0.146}_{\text{0.024}}$ & \bfseries $\text{0.101}_{\text{0.013}}$ & $\text{0.328}_{\text{0.026}}$ & \bfseries $\text{0.110}_{\text{0.014}}$ & $\text{0.277}_{\text{0.020}}$ \\
EDM & $\text{0.128}_{\text{0.020}}$ & \bfseries $\text{0.0734}_{\text{0.0064}}$ & $\text{0.153}_{\text{0.021}}$ & \bfseries $\text{0.0849}_{\text{0.010}}$ & \bfseries $\text{0.0889}_{\text{0.013}}$ \\
AT2 & $\text{0.262}_{\text{0.019}}$ & \bfseries $\text{0.0651}_{\text{0.0070}}$ & $\text{0.376}_{\text{0.014}}$ & \bfseries $\text{0.118}_{\text{0.016}}$ & $\text{0.208}_{\text{0.010}}$ \\
SF1 & $\text{0.176}_{\text{0.0083}}$ & \bfseries $\text{0.0487}_{\text{0.0044}}$ & $\text{0.195}_{\text{0.013}}$ & $\text{0.129}_{\text{0.012}}$ & \bfseries $\text{0.0935}_{\text{0.0062}}$ \\
OE2 & $\text{0.300}_{\text{0.015}}$ & \bfseries $\text{0.0595}_{\text{0.0048}}$ & $\text{0.418}_{\text{0.011}}$ & \bfseries $\text{0.201}_{\text{0.021}}$ & \bfseries $\text{0.220}_{\text{0.013}}$ \\
AT1 & $\text{0.258}_{\text{0.013}}$ & \bfseries $\text{0.0566}_{\text{0.0069}}$ & $\text{0.342}_{\text{0.011}}$ & \bfseries $\text{0.106}_{\text{0.010}}$ & $\text{0.159}_{\text{0.011}}$ \\
JUR & $\text{0.185}_{\text{0.013}}$ & \bfseries $\text{0.0540}_{\text{0.0042}}$ & $\text{0.314}_{\text{0.011}}$ & \bfseries $\text{0.0717}_{\text{0.0094}}$ & $\text{0.212}_{\text{0.010}}$ \\
OE1 & $\text{0.292}_{\text{0.013}}$ & \bfseries $\text{0.0676}_{\text{0.0043}}$ & $\text{0.391}_{\text{0.0097}}$ & \bfseries $\text{0.174}_{\text{0.016}}$ & \bfseries $\text{0.153}_{\text{0.015}}$ \\
ENB & \bfseries $\text{0.0277}_{\text{0.0044}}$ & \bfseries $\text{0.0346}_{\text{0.0035}}$ & \bfseries $\text{0.0362}_{\text{0.0059}}$ & \bfseries $\text{0.0428}_{\text{0.0085}}$ & \bfseries $\text{0.0461}_{\text{0.0060}}$ \\
WQ & $\text{0.328}_{\text{0.0095}}$ & \bfseries $\text{0.0312}_{\text{0.0037}}$ & $\text{0.349}_{\text{0.0067}}$ & \bfseries $\text{0.134}_{\text{0.0083}}$ & $\text{0.188}_{\text{0.0052}}$ \\
SF2 & $\text{0.174}_{\text{0.014}}$ & \bfseries $\text{0.0310}_{\text{0.0026}}$ & $\text{0.197}_{\text{0.018}}$ & $\text{0.123}_{\text{0.016}}$ & \bfseries $\text{0.0726}_{\text{0.0042}}$ \\
SCP & $\text{0.0885}_{\text{0.029}}$ & \bfseries $\text{0.0397}_{\text{0.0038}}$ & $\text{0.0813}_{\text{0.020}}$ & \bfseries $\text{0.0445}_{\text{0.0052}}$ & \bfseries $\text{0.0464}_{\text{0.0051}}$ \\
ANS & \bfseries $\text{0.0212}_{\text{0.0032}}$ & \bfseries $\text{0.0272}_{\text{0.0020}}$ & $\text{0.0443}_{\text{0.0055}}$ & \bfseries $\text{0.0253}_{\text{0.0021}}$ & $\text{0.0478}_{\text{0.0036}}$ \\
HO2 & $\text{0.0806}_{\text{0.0038}}$ & \bfseries $\text{0.0139}_{\text{0.00094}}$ & $\text{0.131}_{\text{0.0043}}$ & \bfseries $\text{0.0157}_{\text{0.0018}}$ & $\text{0.0712}_{\text{0.0023}}$ \\
SC2 & $\text{0.250}_{\text{0.011}}$ & \bfseries $\text{0.0166}_{\text{0.0015}}$ & $\text{0.313}_{\text{0.0083}}$ & \bfseries $\text{0.0672}_{\text{0.0061}}$ & $\text{0.136}_{\text{0.0032}}$ \\
RF1 & $\text{0.0923}_{\text{0.0046}}$ & \bfseries $\text{0.0172}_{\text{0.0038}}$ & $\text{0.0860}_{\text{0.0062}}$ & \bfseries $\text{0.0139}_{\text{0.0012}}$ & \bfseries $\text{0.0189}_{\text{0.0028}}$ \\
SC1 & $\text{0.112}_{\text{0.013}}$ & \bfseries $\text{0.00913}_{\text{0.00061}}$ & $\text{0.417}_{\text{0.017}}$ & \bfseries $\text{0.163}_{\text{0.013}}$ & $\text{0.394}_{\text{0.013}}$ \\
AIR & $\text{0.0936}_{\text{0.0028}}$ & \bfseries $\text{0.0132}_{\text{0.0013}}$ & $\text{0.144}_{\text{0.0049}}$ & \bfseries $\text{0.0193}_{\text{0.0012}}$ & $\text{0.0790}_{\text{0.0037}}$ \\
BI2 & $\text{0.0607}_{\text{0.0068}}$ & \bfseries $\text{0.0122}_{\text{0.0013}}$ & $\text{0.0885}_{\text{0.0062}}$ & \bfseries $\text{0.0156}_{\text{0.0014}}$ & $\text{0.0385}_{\text{0.0024}}$ \\
BI1 & $\text{0.0408}_{\text{0.0039}}$ & \bfseries $\text{0.0105}_{\text{0.00081}}$ & $\text{0.0739}_{\text{0.0064}}$ & \bfseries $\text{0.0128}_{\text{0.00080}}$ & $\text{0.0463}_{\text{0.0027}}$ \\
WAG & $\text{0.113}_{\text{0.0050}}$ & \bfseries $\text{0.0118}_{\text{0.0010}}$ & $\text{0.194}_{\text{0.0036}}$ & \bfseries $\text{0.0134}_{\text{0.00090}}$ & $\text{0.119}_{\text{0.0015}}$ \\
ME3 & $\text{0.0744}_{\text{0.0058}}$ & \bfseries $\text{0.00796}_{\text{0.00065}}$ & $\text{0.0919}_{\text{0.0046}}$ & \bfseries $\text{0.0128}_{\text{0.0011}}$ & $\text{0.0395}_{\text{0.0022}}$ \\
ME1 & $\text{0.0673}_{\text{0.013}}$ & \bfseries $\text{0.0113}_{\text{0.00091}}$ & $\text{0.0932}_{\text{0.016}}$ & \bfseries $\text{0.0274}_{\text{0.0093}}$ & $\text{0.0766}_{\text{0.016}}$ \\
ME2 & $\text{0.0716}_{\text{0.0057}}$ & \bfseries $\text{0.00801}_{\text{0.00052}}$ & $\text{0.0848}_{\text{0.0034}}$ & \bfseries $\text{0.0135}_{\text{0.00087}}$ & $\text{0.0333}_{\text{0.0018}}$ \\
HO1 & $\text{0.0843}_{\text{0.0019}}$ & \bfseries $\text{0.00693}_{\text{0.00053}}$ & $\text{0.0986}_{\text{0.0030}}$ & \bfseries $\text{0.00896}_{\text{0.00069}}$ & $\text{0.0292}_{\text{0.0014}}$ \\
BIO & $\text{0.0713}_{\text{0.0048}}$ & \bfseries $\text{0.00615}_{\text{0.00056}}$ & $\text{0.0850}_{\text{0.0054}}$ & \bfseries $\text{0.00723}_{\text{0.00034}}$ & $\text{0.0230}_{\text{0.0016}}$ \\
BLO & $\text{0.105}_{\text{0.019}}$ & \bfseries $\text{0.00565}_{\text{0.00068}}$ & $\text{0.102}_{\text{0.016}}$ & \bfseries $\text{0.0152}_{\text{0.0026}}$ & $\text{0.0285}_{\text{0.0034}}$ \\
CAL & \bfseries $\text{0.00491}_{\text{0.00067}}$ & \bfseries $\text{0.00591}_{\text{0.00065}}$ & $\text{0.00543}_{\text{0.00065}}$ & $\text{0.00529}_{\text{0.00036}}$ & \bfseries $\text{0.00416}_{\text{0.00027}}$ \\
TAX & $\text{0.0235}_{\text{0.00098}}$ & \bfseries $\text{0.00563}_{\text{0.00050}}$ & $\text{0.0424}_{\text{0.0017}}$ & \bfseries $\text{0.00797}_{\text{0.00075}}$ & $\text{0.0235}_{\text{0.0012}}$ \\
\bottomrule
\end{tabular}

%% file: tables/lr/CFM/scoring_rules.tex
\begin{tabular}{llllll}
\toprule
& \multicolumn{2}{c}{NLL} & \multicolumn{3}{c}{Energy score} \\
& \BASE & \LR & \BASE & \HDRR & \LR \\
\midrule
SLU & \bfseries $\text{2.16}_{\text{0.20}}$ & \bfseries $\text{2.09}_{\text{0.13}}$ & \bfseries $\text{0.756}_{\text{0.033}}$ & \bfseries $\text{0.759}_{\text{0.033}}$ & \bfseries $\text{0.754}_{\text{0.029}}$ \\
EDM & \bfseries $\text{2.20}_{\text{0.18}}$ & \bfseries $\text{2.36}_{\text{0.35}}$ & \bfseries $\text{0.642}_{\text{0.036}}$ & \bfseries $\text{0.639}_{\text{0.037}}$ & \bfseries $\text{0.643}_{\text{0.033}}$ \\
AT2 & $\text{4.76}_{\text{0.52}}$ & \bfseries $\text{2.93}_{\text{0.21}}$ & \bfseries $\text{0.770}_{\text{0.036}}$ & \bfseries $\text{0.770}_{\text{0.032}}$ & \bfseries $\text{0.756}_{\text{0.032}}$ \\
SF1 & $\text{2.38}_{\text{0.59}}$ & \bfseries $\text{0.901}_{\text{0.17}}$ & \bfseries $\text{0.845}_{\text{0.091}}$ & \bfseries $\text{0.826}_{\text{0.092}}$ & \bfseries $\text{0.841}_{\text{0.084}}$ \\
OE2 & \bfseries $\text{9.53}_{\text{1.1}}$ & \bfseries $\text{8.94}_{\text{1.4}}$ & \bfseries $\text{1.20}_{\text{0.066}}$ & \bfseries $\text{1.22}_{\text{0.066}}$ & \bfseries $\text{1.18}_{\text{0.061}}$ \\
AT1 & \bfseries $\text{1.46}_{\text{0.30}}$ & \bfseries $\text{1.22}_{\text{0.24}}$ & \bfseries $\text{0.571}_{\text{0.032}}$ & \bfseries $\text{0.573}_{\text{0.032}}$ & \bfseries $\text{0.571}_{\text{0.031}}$ \\
JUR & \bfseries $\text{2.45}_{\text{0.25}}$ & \bfseries $\text{2.02}_{\text{0.11}}$ & \bfseries $\text{0.613}_{\text{0.027}}$ & \bfseries $\text{0.615}_{\text{0.027}}$ & \bfseries $\text{0.613}_{\text{0.026}}$ \\
OE1 & \bfseries $\text{3.99}_{\text{1.4}}$ & \bfseries $\text{2.24}_{\text{0.83}}$ & \bfseries $\text{0.861}_{\text{0.057}}$ & \bfseries $\text{0.866}_{\text{0.058}}$ & \bfseries $\text{0.855}_{\text{0.057}}$ \\
ENB & \bfseries $\text{0.139}_{\text{0.041}}$ & \bfseries $\text{0.172}_{\text{0.035}}$ & \bfseries $\text{0.313}_{\text{0.0042}}$ & \bfseries $\text{0.316}_{\text{0.0047}}$ & \bfseries $\text{0.313}_{\text{0.0041}}$ \\
WQ & $\text{15.0}_{\text{0.27}}$ & \bfseries $\text{14.3}_{\text{0.25}}$ & \bfseries $\text{2.40}_{\text{0.026}}$ & \bfseries $\text{2.44}_{\text{0.023}}$ & \bfseries $\text{2.41}_{\text{0.025}}$ \\
SF2 & \bfseries $\text{2.30}_{\text{3.0}}$ & \bfseries $\text{-1.06}_{\text{0.36}}$ & \bfseries $\text{0.668}_{\text{0.043}}$ & \bfseries $\text{0.648}_{\text{0.043}}$ & \bfseries $\text{0.662}_{\text{0.041}}$ \\
SCP & $\text{-2.04}_{\text{0.59}}$ & \bfseries $\text{-3.30}_{\text{0.18}}$ & \bfseries $\text{0.391}_{\text{0.092}}$ & \bfseries $\text{0.395}_{\text{0.097}}$ & \bfseries $\text{0.386}_{\text{0.096}}$ \\
ANS & \bfseries $\text{1.78}_{\text{0.023}}$ & \bfseries $\text{1.77}_{\text{0.021}}$ & \bfseries $\text{0.529}_{\text{0.0053}}$ & \bfseries $\text{0.533}_{\text{0.0048}}$ & \bfseries $\text{0.529}_{\text{0.0052}}$ \\
HO2 & \bfseries $\text{2.53}_{\text{0.029}}$ & \bfseries $\text{2.53}_{\text{0.029}}$ & \bfseries $\text{0.852}_{\text{0.0069}}$ & \bfseries $\text{0.856}_{\text{0.0066}}$ & \bfseries $\text{0.852}_{\text{0.0068}}$ \\
SC2 & \bfseries $\text{2.20}_{\text{0.17}}$ & \bfseries $\text{1.82}_{\text{0.17}}$ & \bfseries $\text{1.02}_{\text{0.0089}}$ & \bfseries $\text{1.02}_{\text{0.0090}}$ & \bfseries $\text{1.02}_{\text{0.0089}}$ \\
RF1 & $\text{0.570}_{\text{2.7}}$ & \bfseries $\text{-4.53}_{\text{0.20}}$ & \bfseries $\text{0.367}_{\text{0.018}}$ & \bfseries $\text{0.367}_{\text{0.018}}$ & \bfseries $\text{0.364}_{\text{0.018}}$ \\
SC1 & $\text{0.525}_{\text{0.14}}$ & \bfseries $\text{0.0719}_{\text{0.13}}$ & \bfseries $\text{0.818}_{\text{0.0077}}$ & \bfseries $\text{0.825}_{\text{0.0078}}$ & \bfseries $\text{0.819}_{\text{0.0078}}$ \\
AIR & \bfseries $\text{4.21}_{\text{0.039}}$ & \bfseries $\text{4.20}_{\text{0.040}}$ & \bfseries $\text{1.16}_{\text{0.0090}}$ & \bfseries $\text{1.17}_{\text{0.0089}}$ & \bfseries $\text{1.16}_{\text{0.0090}}$ \\
BI2 & \bfseries $\text{-4.20}_{\text{0.17}}$ & \bfseries $\text{-4.31}_{\text{0.18}}$ & \bfseries $\text{0.788}_{\text{0.013}}$ & \bfseries $\text{0.796}_{\text{0.013}}$ & \bfseries $\text{0.792}_{\text{0.013}}$ \\
BI1 & \bfseries $\text{2.11}_{\text{0.013}}$ & \bfseries $\text{2.10}_{\text{0.010}}$ & \bfseries $\text{0.703}_{\text{0.0057}}$ & \bfseries $\text{0.706}_{\text{0.0053}}$ & \bfseries $\text{0.703}_{\text{0.0056}}$ \\
WAG & \bfseries $\text{0.346}_{\text{0.036}}$ & \bfseries $\text{0.308}_{\text{0.035}}$ & \bfseries $\text{0.699}_{\text{0.0032}}$ & \bfseries $\text{0.700}_{\text{0.0034}}$ & \bfseries $\text{0.697}_{\text{0.0031}}$ \\
ME3 & \bfseries $\text{-0.350}_{\text{0.073}}$ & \bfseries $\text{-0.423}_{\text{0.069}}$ & \bfseries $\text{0.392}_{\text{0.0087}}$ & \bfseries $\text{0.393}_{\text{0.0090}}$ & \bfseries $\text{0.391}_{\text{0.0086}}$ \\
ME1 & \bfseries $\text{-0.444}_{\text{0.067}}$ & \bfseries $\text{-0.518}_{\text{0.070}}$ & \bfseries $\text{0.384}_{\text{0.0077}}$ & \bfseries $\text{0.384}_{\text{0.0078}}$ & \bfseries $\text{0.383}_{\text{0.0076}}$ \\
ME2 & \bfseries $\text{-0.342}_{\text{0.058}}$ & \bfseries $\text{-0.422}_{\text{0.054}}$ & \bfseries $\text{0.395}_{\text{0.0045}}$ & \bfseries $\text{0.397}_{\text{0.0048}}$ & \bfseries $\text{0.396}_{\text{0.0045}}$ \\
HO1 & \bfseries $\text{-0.619}_{\text{0.021}}$ & \bfseries $\text{-0.636}_{\text{0.018}}$ & \bfseries $\text{0.214}_{\text{0.0036}}$ & \bfseries $\text{0.216}_{\text{0.0038}}$ & \bfseries $\text{0.215}_{\text{0.0036}}$ \\
BIO & \bfseries $\text{-0.561}_{\text{0.040}}$ & \bfseries $\text{-0.570}_{\text{0.036}}$ & \bfseries $\text{0.236}_{\text{0.0050}}$ & \bfseries $\text{0.236}_{\text{0.0051}}$ & \bfseries $\text{0.236}_{\text{0.0049}}$ \\
BLO & $\text{-1.06}_{\text{0.030}}$ & \bfseries $\text{-1.14}_{\text{0.026}}$ & \bfseries $\text{0.258}_{\text{0.0031}}$ & \bfseries $\text{0.259}_{\text{0.0030}}$ & \bfseries $\text{0.257}_{\text{0.0030}}$ \\
CAL & \bfseries $\text{0.645}_{\text{0.0065}}$ & \bfseries $\text{0.643}_{\text{0.0064}}$ & \bfseries $\text{0.421}_{\text{0.0014}}$ & \bfseries $\text{0.422}_{\text{0.0014}}$ & \bfseries $\text{0.421}_{\text{0.0014}}$ \\
TAX & \bfseries $\text{1.66}_{\text{0.0079}}$ & \bfseries $\text{1.66}_{\text{0.0077}}$ & \bfseries $\text{0.693}_{\text{0.0021}}$ & \bfseries $\text{0.696}_{\text{0.0020}}$ & \bfseries $\text{0.693}_{\text{0.0021}}$ \\
\bottomrule
\end{tabular}

%% file: tables/lr/CFM/calibration.tex
\begin{tabular}{llllll}
\toprule
& \multicolumn{2}{c}{L-ECE} & \multicolumn{3}{c}{HDR-ECE} \\
& \BASE & \LR & \BASE & \HDRR & \LR \\
\midrule
SLU & \bfseries $\text{0.105}_{\text{0.022}}$ & \bfseries $\text{0.0809}_{\text{0.0088}}$ & \bfseries $\text{0.109}_{\text{0.021}}$ & \bfseries $\text{0.0888}_{\text{0.0089}}$ & \bfseries $\text{0.0761}_{\text{0.0092}}$ \\
EDM & \bfseries $\text{0.0906}_{\text{0.015}}$ & \bfseries $\text{0.0776}_{\text{0.0057}}$ & $\text{0.0926}_{\text{0.014}}$ & \bfseries $\text{0.0820}_{\text{0.0097}}$ & \bfseries $\text{0.0654}_{\text{0.0042}}$ \\
AT2 & $\text{0.160}_{\text{0.030}}$ & \bfseries $\text{0.0621}_{\text{0.010}}$ & $\text{0.165}_{\text{0.030}}$ & \bfseries $\text{0.0784}_{\text{0.016}}$ & \bfseries $\text{0.0605}_{\text{0.0097}}$ \\
SF1 & $\text{0.119}_{\text{0.015}}$ & \bfseries $\text{0.0593}_{\text{0.0072}}$ & $\text{0.112}_{\text{0.017}}$ & \bfseries $\text{0.0604}_{\text{0.0090}}$ & \bfseries $\text{0.0572}_{\text{0.0078}}$ \\
OE2 & $\text{0.199}_{\text{0.036}}$ & \bfseries $\text{0.0615}_{\text{0.0063}}$ & $\text{0.210}_{\text{0.036}}$ & \bfseries $\text{0.0722}_{\text{0.0098}}$ & \bfseries $\text{0.0626}_{\text{0.0072}}$ \\
AT1 & $\text{0.0959}_{\text{0.012}}$ & \bfseries $\text{0.0597}_{\text{0.0074}}$ & $\text{0.0831}_{\text{0.0089}}$ & \bfseries $\text{0.0679}_{\text{0.0057}}$ & \bfseries $\text{0.0615}_{\text{0.0051}}$ \\
JUR & $\text{0.0918}_{\text{0.015}}$ & \bfseries $\text{0.0596}_{\text{0.011}}$ & \bfseries $\text{0.0865}_{\text{0.014}}$ & \bfseries $\text{0.0639}_{\text{0.011}}$ & \bfseries $\text{0.0590}_{\text{0.010}}$ \\
OE1 & $\text{0.109}_{\text{0.019}}$ & \bfseries $\text{0.0647}_{\text{0.0049}}$ & $\text{0.0863}_{\text{0.017}}$ & \bfseries $\text{0.0514}_{\text{0.0078}}$ & \bfseries $\text{0.0618}_{\text{0.0077}}$ \\
ENB & \bfseries $\text{0.0330}_{\text{0.0037}}$ & \bfseries $\text{0.0314}_{\text{0.0021}}$ & \bfseries $\text{0.0340}_{\text{0.0044}}$ & \bfseries $\text{0.0312}_{\text{0.0042}}$ & \bfseries $\text{0.0318}_{\text{0.0048}}$ \\
WQ & $\text{0.110}_{\text{0.0086}}$ & \bfseries $\text{0.0349}_{\text{0.0049}}$ & $\text{0.0842}_{\text{0.0061}}$ & \bfseries $\text{0.0437}_{\text{0.0078}}$ & \bfseries $\text{0.0385}_{\text{0.0042}}$ \\
SF2 & $\text{0.239}_{\text{0.016}}$ & \bfseries $\text{0.0339}_{\text{0.0023}}$ & $\text{0.226}_{\text{0.015}}$ & \bfseries $\text{0.0397}_{\text{0.0051}}$ & \bfseries $\text{0.0400}_{\text{0.0030}}$ \\
SCP & $\text{0.268}_{\text{0.020}}$ & \bfseries $\text{0.0457}_{\text{0.0070}}$ & $\text{0.253}_{\text{0.019}}$ & \bfseries $\text{0.0483}_{\text{0.0053}}$ & \bfseries $\text{0.0607}_{\text{0.0067}}$ \\
ANS & \bfseries $\text{0.0222}_{\text{0.0042}}$ & \bfseries $\text{0.0236}_{\text{0.0027}}$ & \bfseries $\text{0.0231}_{\text{0.0040}}$ & \bfseries $\text{0.0244}_{\text{0.0025}}$ & \bfseries $\text{0.0236}_{\text{0.0033}}$ \\
HO2 & $\text{0.0232}_{\text{0.0038}}$ & \bfseries $\text{0.0138}_{\text{0.0011}}$ & $\text{0.0180}_{\text{0.0028}}$ & \bfseries $\text{0.0109}_{\text{0.0011}}$ & \bfseries $\text{0.0116}_{\text{0.0013}}$ \\
SC2 & $\text{0.0590}_{\text{0.0028}}$ & \bfseries $\text{0.0129}_{\text{0.0017}}$ & $\text{0.0532}_{\text{0.0032}}$ & \bfseries $\text{0.0142}_{\text{0.0011}}$ & \bfseries $\text{0.0127}_{\text{0.0012}}$ \\
RF1 & $\text{0.0843}_{\text{0.0078}}$ & \bfseries $\text{0.0273}_{\text{0.0037}}$ & $\text{0.0740}_{\text{0.0082}}$ & \bfseries $\text{0.0125}_{\text{0.0021}}$ & $\text{0.0353}_{\text{0.0039}}$ \\
SC1 & $\text{0.0573}_{\text{0.0056}}$ & \bfseries $\text{0.0133}_{\text{0.0020}}$ & $\text{0.0720}_{\text{0.0066}}$ & \bfseries $\text{0.0156}_{\text{0.0021}}$ & $\text{0.0267}_{\text{0.0037}}$ \\
AIR & \bfseries $\text{0.0130}_{\text{0.0027}}$ & \bfseries $\text{0.0151}_{\text{0.0023}}$ & \bfseries $\text{0.0158}_{\text{0.0027}}$ & \bfseries $\text{0.0132}_{\text{0.0021}}$ & $\text{0.0220}_{\text{0.0035}}$ \\
BI2 & $\text{0.109}_{\text{0.014}}$ & \bfseries $\text{0.00979}_{\text{0.0014}}$ & $\text{0.0916}_{\text{0.013}}$ & \bfseries $\text{0.0112}_{\text{0.0016}}$ & $\text{0.0288}_{\text{0.0059}}$ \\
BI1 & $\text{0.0243}_{\text{0.0023}}$ & \bfseries $\text{0.00924}_{\text{0.00093}}$ & $\text{0.0249}_{\text{0.0020}}$ & \bfseries $\text{0.00927}_{\text{0.00061}}$ & \bfseries $\text{0.0101}_{\text{0.0013}}$ \\
WAG & $\text{0.0407}_{\text{0.0051}}$ & \bfseries $\text{0.0144}_{\text{0.0013}}$ & $\text{0.0419}_{\text{0.0050}}$ & \bfseries $\text{0.0120}_{\text{0.0014}}$ & $\text{0.0307}_{\text{0.0031}}$ \\
ME3 & $\text{0.0412}_{\text{0.0077}}$ & \bfseries $\text{0.00965}_{\text{0.0011}}$ & $\text{0.0330}_{\text{0.0065}}$ & \bfseries $\text{0.0104}_{\text{0.00098}}$ & $\text{0.0176}_{\text{0.0029}}$ \\
ME1 & $\text{0.0299}_{\text{0.0086}}$ & \bfseries $\text{0.00915}_{\text{0.0014}}$ & $\text{0.0251}_{\text{0.0075}}$ & \bfseries $\text{0.00935}_{\text{0.00058}}$ & $\text{0.0143}_{\text{0.0020}}$ \\
ME2 & $\text{0.0537}_{\text{0.0060}}$ & \bfseries $\text{0.00883}_{\text{0.0013}}$ & $\text{0.0541}_{\text{0.0063}}$ & \bfseries $\text{0.00892}_{\text{0.00094}}$ & $\text{0.0156}_{\text{0.0014}}$ \\
HO1 & $\text{0.0285}_{\text{0.0046}}$ & \bfseries $\text{0.00990}_{\text{0.0016}}$ & $\text{0.0239}_{\text{0.0043}}$ & \bfseries $\text{0.00974}_{\text{0.0011}}$ & \bfseries $\text{0.0119}_{\text{0.0016}}$ \\
BIO & $\text{0.0253}_{\text{0.0053}}$ & \bfseries $\text{0.00766}_{\text{0.0019}}$ & $\text{0.0221}_{\text{0.0058}}$ & \bfseries $\text{0.00726}_{\text{0.00062}}$ & $\text{0.0102}_{\text{0.0016}}$ \\
BLO & $\text{0.0346}_{\text{0.010}}$ & \bfseries $\text{0.00498}_{\text{0.00047}}$ & $\text{0.0402}_{\text{0.0099}}$ & \bfseries $\text{0.00648}_{\text{0.00050}}$ & $\text{0.0153}_{\text{0.0018}}$ \\
CAL & $\text{0.00951}_{\text{0.0014}}$ & \bfseries $\text{0.00395}_{\text{0.00043}}$ & $\text{0.0112}_{\text{0.0018}}$ & \bfseries $\text{0.00547}_{\text{0.00030}}$ & $\text{0.00746}_{\text{0.0011}}$ \\
TAX & $\text{0.0110}_{\text{0.0018}}$ & \bfseries $\text{0.00505}_{\text{0.00059}}$ & $\text{0.0122}_{\text{0.0017}}$ & \bfseries $\text{0.00660}_{\text{0.00054}}$ & $\text{0.0133}_{\text{0.0022}}$ \\
\bottomrule
\end{tabular}

%% file: tables/lr_misspecified/MQF2/scoring_rules.tex
\begin{tabular}{llllll}
\toprule
& \multicolumn{2}{c}{NLL} & \multicolumn{3}{c}{Energy score} \\
& \BASE & \LR & \BASE & \HDRR & \LR \\
\midrule
SLU & \bfseries $\text{3.86}_{\text{0.13}}$ & \bfseries $\text{3.93}_{\text{0.16}}$ & \bfseries $\text{1.01}_{\text{0.043}}$ & \bfseries $\text{1.03}_{\text{0.050}}$ & \bfseries $\text{1.01}_{\text{0.047}}$ \\
EDM & $\text{2.84}_{\text{0.059}}$ & \bfseries $\text{1.53}_{\text{0.32}}$ & \bfseries $\text{0.880}_{\text{0.028}}$ & \bfseries $\text{0.879}_{\text{0.029}}$ & \bfseries $\text{0.871}_{\text{0.029}}$ \\
AT2 & $\text{7.99}_{\text{0.29}}$ & \bfseries $\text{6.03}_{\text{0.14}}$ & \bfseries $\text{1.39}_{\text{0.051}}$ & \bfseries $\text{1.36}_{\text{0.051}}$ & \bfseries $\text{1.34}_{\text{0.054}}$ \\
SF1 & $\text{4.05}_{\text{0.39}}$ & \bfseries $\text{0.736}_{\text{0.22}}$ & \bfseries $\text{0.857}_{\text{0.075}}$ & \bfseries $\text{0.754}_{\text{0.083}}$ & \bfseries $\text{0.753}_{\text{0.079}}$ \\
OE2 & $\text{23.1}_{\text{2.2}}$ & \bfseries $\text{13.5}_{\text{0.80}}$ & \bfseries $\text{2.24}_{\text{0.14}}$ & \bfseries $\text{2.24}_{\text{0.15}}$ & \bfseries $\text{2.07}_{\text{0.16}}$ \\
AT1 & $\text{7.68}_{\text{0.38}}$ & \bfseries $\text{5.99}_{\text{0.26}}$ & \bfseries $\text{1.42}_{\text{0.070}}$ & \bfseries $\text{1.40}_{\text{0.078}}$ & \bfseries $\text{1.39}_{\text{0.076}}$ \\
JUR & \bfseries $\text{4.19}_{\text{0.13}}$ & \bfseries $\text{4.09}_{\text{0.075}}$ & \bfseries $\text{1.08}_{\text{0.034}}$ & \bfseries $\text{1.08}_{\text{0.036}}$ & \bfseries $\text{1.08}_{\text{0.034}}$ \\
OE1 & $\text{23.8}_{\text{2.1}}$ & \bfseries $\text{11.1}_{\text{0.65}}$ & \bfseries $\text{2.14}_{\text{0.13}}$ & \bfseries $\text{2.12}_{\text{0.14}}$ & \bfseries $\text{1.91}_{\text{0.15}}$ \\
ENB & \bfseries $\text{2.05}_{\text{0.093}}$ & \bfseries $\text{2.02}_{\text{0.090}}$ & \bfseries $\text{0.815}_{\text{0.020}}$ & \bfseries $\text{0.823}_{\text{0.020}}$ & \bfseries $\text{0.814}_{\text{0.019}}$ \\
WQ & $\text{20.0}_{\text{0.15}}$ & \bfseries $\text{19.6}_{\text{0.15}}$ & \bfseries $\text{2.59}_{\text{0.027}}$ & \bfseries $\text{2.61}_{\text{0.028}}$ & \bfseries $\text{2.59}_{\text{0.027}}$ \\
SF2 & $\text{4.71}_{\text{0.55}}$ & \bfseries $\text{0.298}_{\text{0.25}}$ & $\text{0.821}_{\text{0.038}}$ & \bfseries $\text{0.703}_{\text{0.039}}$ & \bfseries $\text{0.688}_{\text{0.038}}$ \\
SCP & $\text{5.06}_{\text{2.1}}$ & \bfseries $\text{1.25}_{\text{0.36}}$ & \bfseries $\text{0.738}_{\text{0.10}}$ & \bfseries $\text{0.657}_{\text{0.11}}$ & \bfseries $\text{0.658}_{\text{0.11}}$ \\
ANS & \bfseries $\text{2.61}_{\text{0.021}}$ & \bfseries $\text{2.61}_{\text{0.020}}$ & \bfseries $\text{0.817}_{\text{0.0089}}$ & \bfseries $\text{0.819}_{\text{0.0087}}$ & \bfseries $\text{0.817}_{\text{0.0090}}$ \\
HO2 & $\text{5.46}_{\text{0.020}}$ & \bfseries $\text{5.18}_{\text{0.021}}$ & \bfseries $\text{1.23}_{\text{0.0063}}$ & \bfseries $\text{1.23}_{\text{0.0070}}$ & \bfseries $\text{1.22}_{\text{0.0067}}$ \\
SC2 & $\text{22.2}_{\text{0.12}}$ & \bfseries $\text{20.8}_{\text{0.12}}$ & \bfseries $\text{2.67}_{\text{0.019}}$ & \bfseries $\text{2.69}_{\text{0.020}}$ & \bfseries $\text{2.66}_{\text{0.021}}$ \\
RF1 & \bfseries $\text{10.8}_{\text{0.84}}$ & \bfseries $\text{16.8}_{\text{7.5}}$ & \bfseries $\text{1.69}_{\text{0.026}}$ & \bfseries $\text{1.70}_{\text{0.026}}$ & \bfseries $\text{1.69}_{\text{0.032}}$ \\
SC1 & $\text{21.4}_{\text{0.099}}$ & \bfseries $\text{20.0}_{\text{0.10}}$ & \bfseries $\text{2.56}_{\text{0.023}}$ & \bfseries $\text{2.58}_{\text{0.024}}$ & \bfseries $\text{2.54}_{\text{0.023}}$ \\
BI2 & $\text{5.62}_{\text{0.10}}$ & \bfseries $\text{3.51}_{\text{0.085}}$ & $\text{1.08}_{\text{0.016}}$ & \bfseries $\text{1.03}_{\text{0.016}}$ & \bfseries $\text{1.02}_{\text{0.016}}$ \\
BI1 & $\text{2.66}_{\text{0.030}}$ & \bfseries $\text{2.39}_{\text{0.020}}$ & \bfseries $\text{0.779}_{\text{0.0057}}$ & \bfseries $\text{0.776}_{\text{0.0057}}$ & \bfseries $\text{0.773}_{\text{0.0055}}$ \\
AIR & $\text{8.33}_{\text{0.068}}$ & \bfseries $\text{7.58}_{\text{0.044}}$ & \bfseries $\text{1.50}_{\text{0.0089}}$ & \bfseries $\text{1.50}_{\text{0.0089}}$ & \bfseries $\text{1.49}_{\text{0.0090}}$ \\
WAG & $\text{2.78}_{\text{0.0069}}$ & \bfseries $\text{2.43}_{\text{0.17}}$ & $\text{0.876}_{\text{0.0029}}$ & $\text{0.875}_{\text{0.0027}}$ & \bfseries $\text{0.868}_{\text{0.0029}}$ \\
ME3 & $\text{2.31}_{\text{0.045}}$ & \bfseries $\text{1.12}_{\text{0.046}}$ & $\text{0.602}_{\text{0.010}}$ & \bfseries $\text{0.548}_{\text{0.0087}}$ & \bfseries $\text{0.549}_{\text{0.0091}}$ \\
ME1 & $\text{2.20}_{\text{0.046}}$ & \bfseries $\text{1.10}_{\text{0.048}}$ & $\text{0.588}_{\text{0.0092}}$ & \bfseries $\text{0.539}_{\text{0.0073}}$ & \bfseries $\text{0.540}_{\text{0.0073}}$ \\
ME2 & $\text{2.15}_{\text{0.027}}$ & \bfseries $\text{1.13}_{\text{0.038}}$ & $\text{0.591}_{\text{0.0055}}$ & \bfseries $\text{0.545}_{\text{0.0059}}$ & \bfseries $\text{0.546}_{\text{0.0058}}$ \\
HO1 & $\text{2.46}_{\text{0.029}}$ & \bfseries $\text{2.34}_{\text{0.030}}$ & \bfseries $\text{0.743}_{\text{0.0068}}$ & \bfseries $\text{0.743}_{\text{0.0072}}$ & \bfseries $\text{0.741}_{\text{0.0071}}$ \\
BIO & $\text{0.963}_{\text{0.14}}$ & \bfseries $\text{0.302}_{\text{0.046}}$ & $\text{0.335}_{\text{0.0085}}$ & \bfseries $\text{0.308}_{\text{0.0058}}$ & \bfseries $\text{0.311}_{\text{0.0062}}$ \\
CAL & $\text{1.37}_{\text{0.043}}$ & \bfseries $\text{1.28}_{\text{0.035}}$ & \bfseries $\text{0.472}_{\text{0.0061}}$ & \bfseries $\text{0.468}_{\text{0.0050}}$ & \bfseries $\text{0.468}_{\text{0.0052}}$ \\
BLO & $\text{1.21}_{\text{0.033}}$ & \bfseries $\text{0.785}_{\text{0.036}}$ & \bfseries $\text{0.685}_{\text{0.0024}}$ & $\text{0.691}_{\text{0.0030}}$ & \bfseries $\text{0.687}_{\text{0.0029}}$ \\
TAX & \bfseries $\text{2.29}_{\text{0.019}}$ & \bfseries $\text{2.31}_{\text{0.11}}$ & \bfseries $\text{0.724}_{\text{0.0024}}$ & \bfseries $\text{0.725}_{\text{0.0024}}$ & \bfseries $\text{0.723}_{\text{0.0024}}$ \\
\bottomrule
\end{tabular}

%% file: tables/lr_misspecified/MQF2/calibration.tex
\begin{tabular}{llllll}
\toprule
& \multicolumn{2}{c}{L-ECE} & \multicolumn{3}{c}{HDR-ECE} \\
& \BASE & \LR & \BASE & \HDRR & \LR \\
\midrule
SLU & \bfseries $\text{0.120}_{\text{0.018}}$ & \bfseries $\text{0.107}_{\text{0.017}}$ & \bfseries $\text{0.127}_{\text{0.018}}$ & \bfseries $\text{0.119}_{\text{0.017}}$ & \bfseries $\text{0.101}_{\text{0.015}}$ \\
EDM & $\text{0.135}_{\text{0.0052}}$ & \bfseries $\text{0.0734}_{\text{0.0068}}$ & $\text{0.128}_{\text{0.0051}}$ & \bfseries $\text{0.0778}_{\text{0.0097}}$ & \bfseries $\text{0.0846}_{\text{0.012}}$ \\
AT2 & $\text{0.319}_{\text{0.0075}}$ & \bfseries $\text{0.0593}_{\text{0.0083}}$ & $\text{0.327}_{\text{0.0078}}$ & \bfseries $\text{0.0673}_{\text{0.011}}$ & \bfseries $\text{0.0642}_{\text{0.0095}}$ \\
SF1 & $\text{0.351}_{\text{0.011}}$ & \bfseries $\text{0.110}_{\text{0.0094}}$ & $\text{0.354}_{\text{0.011}}$ & \bfseries $\text{0.0838}_{\text{0.012}}$ & $\text{0.177}_{\text{0.014}}$ \\
OE2 & $\text{0.379}_{\text{0.012}}$ & \bfseries $\text{0.0773}_{\text{0.0067}}$ & $\text{0.379}_{\text{0.012}}$ & $\text{0.335}_{\text{0.013}}$ & \bfseries $\text{0.0769}_{\text{0.0071}}$ \\
AT1 & $\text{0.278}_{\text{0.010}}$ & \bfseries $\text{0.0533}_{\text{0.0051}}$ & $\text{0.297}_{\text{0.011}}$ & $\text{0.100}_{\text{0.0080}}$ & \bfseries $\text{0.0565}_{\text{0.0063}}$ \\
JUR & $\text{0.0785}_{\text{0.0070}}$ & \bfseries $\text{0.0494}_{\text{0.0064}}$ & $\text{0.0819}_{\text{0.0078}}$ & \bfseries $\text{0.0518}_{\text{0.0062}}$ & \bfseries $\text{0.0501}_{\text{0.0068}}$ \\
OE1 & $\text{0.410}_{\text{0.013}}$ & \bfseries $\text{0.0733}_{\text{0.0057}}$ & $\text{0.410}_{\text{0.013}}$ & $\text{0.395}_{\text{0.013}}$ & \bfseries $\text{0.0749}_{\text{0.0054}}$ \\
ENB & $\text{0.0458}_{\text{0.0030}}$ & \bfseries $\text{0.0326}_{\text{0.0028}}$ & $\text{0.0854}_{\text{0.0099}}$ & \bfseries $\text{0.0367}_{\text{0.0049}}$ & $\text{0.114}_{\text{0.010}}$ \\
WQ & $\text{0.112}_{\text{0.0060}}$ & \bfseries $\text{0.0405}_{\text{0.0054}}$ & $\text{0.114}_{\text{0.0060}}$ & \bfseries $\text{0.0467}_{\text{0.0051}}$ & \bfseries $\text{0.0391}_{\text{0.0053}}$ \\
SF2 & $\text{0.386}_{\text{0.0057}}$ & \bfseries $\text{0.0869}_{\text{0.0087}}$ & $\text{0.387}_{\text{0.0057}}$ & \bfseries $\text{0.150}_{\text{0.0090}}$ & \bfseries $\text{0.139}_{\text{0.011}}$ \\
SCP & $\text{0.361}_{\text{0.013}}$ & \bfseries $\text{0.0808}_{\text{0.0082}}$ & $\text{0.369}_{\text{0.013}}$ & \bfseries $\text{0.0517}_{\text{0.0063}}$ & $\text{0.115}_{\text{0.014}}$ \\
ANS & $\text{0.0428}_{\text{0.0043}}$ & \bfseries $\text{0.0172}_{\text{0.0012}}$ & $\text{0.0524}_{\text{0.0047}}$ & \bfseries $\text{0.0162}_{\text{0.0011}}$ & \bfseries $\text{0.0183}_{\text{0.0019}}$ \\
HO2 & $\text{0.125}_{\text{0.0038}}$ & \bfseries $\text{0.0155}_{\text{0.0019}}$ & $\text{0.130}_{\text{0.0035}}$ & \bfseries $\text{0.0155}_{\text{0.0016}}$ & \bfseries $\text{0.0156}_{\text{0.0020}}$ \\
SC2 & $\text{0.150}_{\text{0.0030}}$ & \bfseries $\text{0.0123}_{\text{0.00093}}$ & $\text{0.153}_{\text{0.0032}}$ & $\text{0.0491}_{\text{0.00096}}$ & \bfseries $\text{0.0124}_{\text{0.00095}}$ \\
RF1 & $\text{0.173}_{\text{0.0076}}$ & \bfseries $\text{0.0172}_{\text{0.0024}}$ & $\text{0.186}_{\text{0.0081}}$ & $\text{0.0422}_{\text{0.0030}}$ & \bfseries $\text{0.0180}_{\text{0.0025}}$ \\
SC1 & $\text{0.164}_{\text{0.0033}}$ & \bfseries $\text{0.00888}_{\text{0.0010}}$ & $\text{0.170}_{\text{0.0040}}$ & $\text{0.0541}_{\text{0.0016}}$ & \bfseries $\text{0.0103}_{\text{0.0014}}$ \\
BI2 & $\text{0.286}_{\text{0.0034}}$ & \bfseries $\text{0.0133}_{\text{0.0014}}$ & $\text{0.287}_{\text{0.0034}}$ & $\text{0.0616}_{\text{0.0019}}$ & \bfseries $\text{0.0146}_{\text{0.0013}}$ \\
BI1 & $\text{0.0961}_{\text{0.0039}}$ & \bfseries $\text{0.00962}_{\text{0.0014}}$ & $\text{0.100}_{\text{0.0040}}$ & $\text{0.0126}_{\text{0.0016}}$ & \bfseries $\text{0.0102}_{\text{0.00095}}$ \\
AIR & $\text{0.178}_{\text{0.0034}}$ & \bfseries $\text{0.0146}_{\text{0.0016}}$ & $\text{0.181}_{\text{0.0033}}$ & \bfseries $\text{0.0151}_{\text{0.0013}}$ & \bfseries $\text{0.0140}_{\text{0.0016}}$ \\
WAG & $\text{0.116}_{\text{0.0011}}$ & \bfseries $\text{0.0127}_{\text{0.0021}}$ & $\text{0.113}_{\text{0.0011}}$ & \bfseries $\text{0.0158}_{\text{0.0019}}$ & \bfseries $\text{0.0166}_{\text{0.0033}}$ \\
ME3 & $\text{0.287}_{\text{0.0047}}$ & \bfseries $\text{0.0105}_{\text{0.0016}}$ & $\text{0.294}_{\text{0.0040}}$ & $\text{0.0304}_{\text{0.0025}}$ & \bfseries $\text{0.0140}_{\text{0.0022}}$ \\
ME1 & $\text{0.287}_{\text{0.0029}}$ & \bfseries $\text{0.0110}_{\text{0.0011}}$ & $\text{0.295}_{\text{0.0027}}$ & $\text{0.0237}_{\text{0.0031}}$ & \bfseries $\text{0.0149}_{\text{0.0020}}$ \\
ME2 & $\text{0.282}_{\text{0.0031}}$ & \bfseries $\text{0.00893}_{\text{0.00095}}$ & $\text{0.290}_{\text{0.0030}}$ & $\text{0.0220}_{\text{0.0021}}$ & \bfseries $\text{0.0144}_{\text{0.0017}}$ \\
HO1 & $\text{0.0801}_{\text{0.0046}}$ & \bfseries $\text{0.0102}_{\text{0.0012}}$ & $\text{0.0880}_{\text{0.0041}}$ & \bfseries $\text{0.00989}_{\text{0.00082}}$ & \bfseries $\text{0.0120}_{\text{0.0024}}$ \\
BIO & $\text{0.200}_{\text{0.0086}}$ & \bfseries $\text{0.00695}_{\text{0.00065}}$ & $\text{0.213}_{\text{0.0080}}$ & \bfseries $\text{0.00985}_{\text{0.00048}}$ & \bfseries $\text{0.00916}_{\text{0.0013}}$ \\
CAL & $\text{0.0711}_{\text{0.013}}$ & \bfseries $\text{0.00549}_{\text{0.00046}}$ & $\text{0.0883}_{\text{0.013}}$ & \bfseries $\text{0.00681}_{\text{0.00041}}$ & $\text{0.0169}_{\text{0.0024}}$ \\
BLO & $\text{0.0972}_{\text{0.0039}}$ & \bfseries $\text{0.00657}_{\text{0.00092}}$ & $\text{0.117}_{\text{0.0025}}$ & \bfseries $\text{0.00843}_{\text{0.00092}}$ & $\text{0.0241}_{\text{0.0033}}$ \\
TAX & $\text{0.0444}_{\text{0.0032}}$ & \bfseries $\text{0.00490}_{\text{0.00060}}$ & $\text{0.0565}_{\text{0.0020}}$ & \bfseries $\text{0.00804}_{\text{0.00053}}$ & \bfseries $\text{0.00918}_{\text{0.0015}}$ \\
\bottomrule
\end{tabular}